\documentclass{article}

\PassOptionsToPackage{numbers, compress}{natbib}

\usepackage[main, final]{neurips_2026}

\usepackage[utf8]{inputenc} %
\usepackage[T1]{fontenc}    %
\usepackage{hyperref}       %
\usepackage{url}            %
\usepackage{booktabs}       %
\usepackage{amsfonts}       %
\usepackage{nicefrac}       %
\usepackage{microtype}      %
\usepackage{xcolor}         %

\usepackage{etoolbox}
\usepackage{minitoc}
\usepackage{multirow}
\usepackage{tabularx}

\usepackage{graphicx}
\usepackage{subcaption}
\usepackage{pifont}%

\usepackage{amsmath,amsfonts}
\usepackage{amssymb}
\usepackage{mathtools}
\usepackage{amsthm}

\usepackage{array}  

\usepackage{algorithm}
\usepackage{algorithmic}

\usepackage{enumitem}

\usepackage{amsmath}
\usepackage{wrapfig}
\usepackage{siunitx}

\newif\ifshowchanges
\showchangestrue

\theoremstyle{plain}

\newtheorem{thm}{Theorem}[section]
\newtheorem{lem}{Lemma}[section]

\newtheorem{prop}{Proposition}[section]
\newtheorem{asmp}{Assumption}[section]

\newtheorem{remark}{Remark}[section]

\title{Conformal Prediction with Paraphrase-Aware Scoring for LLM Uncertainty Quantification}

\author{%
  Jiayi Xin\thanks{Corresponding authors: Jiayi Xin
    (\texttt{jiayixin@engineering.upenn.edu}), Xiang Li (\texttt{lx10077@upenn.edu}), Weijie J. Su
    (\texttt{suw@wharton.upenn.edu}), and Qi Long
    (\texttt{qlong@upenn.edu}).} \\
  University of Pennsylvania
  \And
  Evan Qiang \\
  University of Pennsylvania
  \And
  Zihan Zhu \\
  University of Pennsylvania
  \AND
  Xiang Li$^{*}$ \\
  University of Pennsylvania
  \And
  Weijie J. Su$^{*}$  \\
  University of Pennsylvania
  \And
  Qi Long$^{*}$  \\
  University of Pennsylvania
}

\begin{document}

\maketitle

\doparttoc

\begin{abstract}

Uncertainty quantification (UQ) for large language models (LLMs) aims to provide reliable measures of predictive confidence, yet current methods are often unstable under meaning-preserving perturbations. Semantically equivalent paraphrases can induce substantial variability in predictive confidence, even for methods with formal guarantees, such as conformal prediction. %
{To address this issue, we propose a paraphrase-aware UQ framework robust to semantic rewordings. Our approach trains a lightweight proxy model on LLM hidden states and aggregates its predictions across paraphrases to construct label-wise nonconformity scores. Under score exchangeability, conformal calibration retains marginal coverage. This guarantee can also hold under test-only rewording, provided that the paraphrase pipeline satisfies an additional distributional alignment condition. We evaluate three settings (normal, fully reworded, and semi-reworded) which apply rewording to neither dataset, both calibration and test datasets, or only the test dataset, respectively. Across seven multiple-choice QA benchmarks and multiple model families, our method produces compact prediction sets with empirical coverage generally near the nominal target, even in the semi-reworded setting. Ablation studies show that the learned proxy accounts for most of the reduction in set size, while paraphrase-augmented training and inference-time aggregation improve stability under rewording. Code is available at {\small\url{https://github.com/Raina-Xin/PA_Score}}.}
\end{abstract}

\doparttoc %
\faketableofcontents %

\section{Introduction}
\label{sec:intro}

Large language models (LLMs) are increasingly deployed in high-stakes domains such as education and medicine \citep{bouchard2025uncertainty,lopez2025uncertainty}. In these settings, reliable uncertainty quantification (UQ) is indispensable, as decision-makers must assess when to trust model predictions. However, UQ in LLMs is often unstable: predictive confidence can be poorly calibrated and highly sensitive to superficial variations in input phrasing \citep{shorinwa2024survey}. Such instability undermines the reliability of uncertainty estimates and limits their practical usefulness.

Another fundamental limitation is that many existing UQ methods lack formal guarantees. Most approaches rely on heuristic confidence scores \citep{shorinwa2024survey,nado2022uncertainty,ulmer2022exploring,band2024linguistic,huang2024calibrating}, making it difficult to determine whether their outputs can be trusted, especially in high-stakes scenarios. This motivates the use of methods with theoretical guarantees. Conformal prediction (CP) provides such a framework, constructing prediction sets with finite-sample coverage under mild assumptions \citep{vovk2005algorithmic,shafer2008tutorial}. In CP, a nonconformity score measures how atypical a candidate output is for a given input and determines whether it is included in the prediction set.

While CP ensures statistical validity, its practical reliability depends critically on the stability of the underlying uncertainty scores. Ideally, semantically equivalent inputs should yield comparable uncertainty estimates. For example, the questions ``What is the capital of France?'' and ``Which city serves as France’s capital?'' should lead to similar prediction sets. In practice, however, even such paraphrases can induce substantial variability in uncertainty estimates, leading to unstable or inflated prediction sets. This simple example reveals a gap between theoretical guarantees and semantic robustness. This challenge is particularly pronounced in the LLM setting, where variability arises from semantic-preserving perturbations such as paraphrasing, rather than classical distributional shifts. Addressing such instability requires methods that can achieve both statistical validity and semantic robustness.

\begin{figure}[t!]
   \vspace{-3mm}
  \centering
  \includegraphics[width=0.99\linewidth]{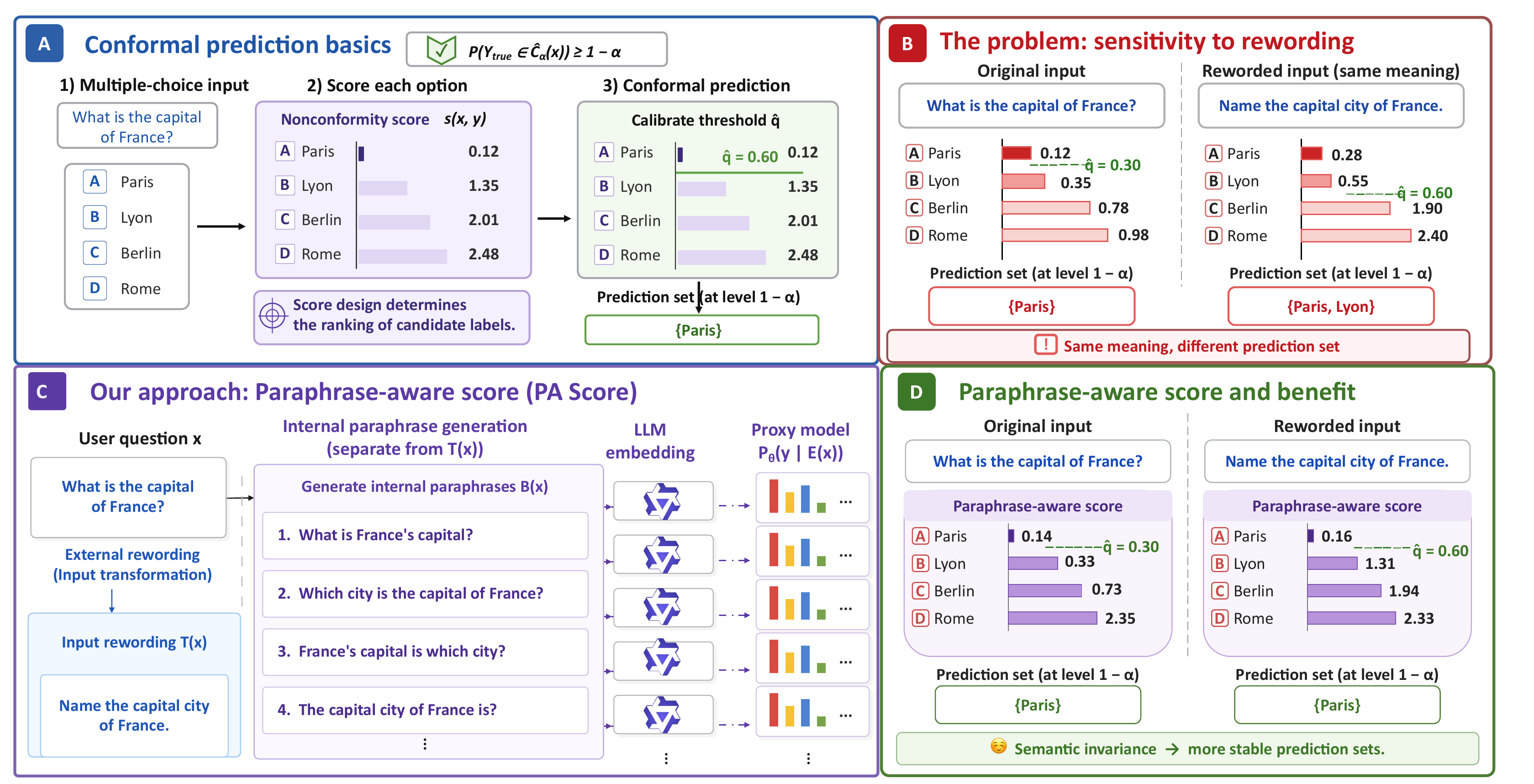}
  \vspace{-2mm}
  \caption{Prior scores are sensitive to rewording or paraphrasing of input question (denoted as $\mathcal{T}$), producing larger prediction sets (Panel A and B). %
{Our PA framework aggregates multiple paraphrases with a proxy model to reduce score and prediction-set sensitivity to rewording (Panel C and D).}}
  \label{fig:fig1-comparison}
  \vspace{-10pt}
\end{figure}

\vspace{-10pt}
\paragraph{Contributions.}
To address this challenge, we introduce a paraphrase-aware scoring mechanism within the conformal prediction framework that %
{reduces sensitivity to semantic rewordings}. Rather than relying on a single input, our approach constructs uncertainty estimates by aggregating predictions across a set of paraphrases, thereby capturing variability induced by rewording. To enable consistent aggregation, we employ a lightweight proxy model to produce calibrated and comparable outputs across paraphrases. %
{The resulting nonconformity scores can be integrated with standard CP, retaining its coverage guarantees when the required exchangeability conditions hold.}

Focusing on the nonconformity score provides a convenient mechanism for incorporating problem-specific structure.
Importantly, modifying the score leaves the underlying conformal guarantees unchanged, making it a simple and modular approach to improving robustness. %
{We demonstrate that our method applies to both split CP and quasi-conditional CP (QCCP) \citep{gibbs2023conformal}. The learned proxy accounts for most of the end-to-end set-size reduction, while paraphrase-augmented training and test-time aggregation contribute to wording stability.}

Theoretically, we show that our scores preserve score-level exchangeability under paraphrased inputs, and more generally under transformations satisfying a distributional closedness condition (see Asm.~\ref{asmp:main}). Empirically, we conduct large-scale experiments on five general QA benchmarks and two medical QA benchmarks, %
{showing end-to-end prediction-set reductions of up to $2$--$4\times$ against raw-probability baselines, with empirical coverage generally near the nominal target. We also conducted controlled ablations to separate the proxy, training-view, and inference-aggregation contributions.}

\section{{Related Work}}
\label{sec:main-related}

{Conformal prediction provides finite-sample marginal coverage when
calibration and test examples are exchangeable
\citep{vovk2005algorithmic,shafer2008tutorial}. Its applications to
language models include conformal language modeling, multiple-choice
question answering, and API-based inference
\citep{quach2024conformal,kumar2023conformal,su2024api}.
Learned nonconformity scores and quasi-conditional methods further
adapt prediction sets to task structure
\citep{einbinder2022training,gibbs2023conformal}.
Our focus is the additional question of how label-wise scores and
prediction sets change when an input is reworded without changing
its intended meaning.
Perturb-and-aggregate uncertainty estimation is established in
SPUQ~\citep{gao2024spuq}, while other work studies semantic uncertainty
and prompt sensitivity \citep{kossen2024semantic,qiu2025semantic}.
SPUQ's original instance-level uncertainty value is distinct from
the per-option score required to form conformal prediction sets.
Prior work establishes perturbation-based uncertainty estimation and conformal coverage, but leaves open how to design label-wise nonconformity scores that are stable under meaning-preserving rewording. This matters in practice because semantically equivalent queries can yield different prediction sets even when marginal coverage is maintained. Our work addresses this nonconformity score design problem through a unified treatment of methodology, evaluation, and validity analysis, extending the focus from instance-level uncertainty to the stability of conformal decisions. For detailed related work, please refer to Appendix \ref{sec:related}.}

\section{Preliminaries}
\label{sec:preliminaries}

\vspace{-8pt}
\paragraph{Problem Setup.}  
{We study conformal prediction for classification when inputs may be reworded without changing their meaning or correct labels. Our goal is to construct small prediction sets that contain the correct label at a prescribed coverage level. We investigate how the choice of nonconformity score affects this coverage-efficiency tradeoff under paraphrasing, considering both matched calibration and test distributions and test-only rewording shifts. Coverage under the latter setting requires additional assumptions, which we formalize in Section \ref{sec:theoretical-analysis}.}

Formally, we study supervised prediction with input $x \in \mathcal{X}$ (e.g., a natural language question) and a finite label space $\mathcal{Y}$ (e.g., multiple-choice answers). Given a calibration dataset $\{(X_i, Y_i)\}_{i=1}^n$ with $Y_i \in \mathcal{Y}$, the goal is to construct a prediction set $\widehat{C}(X_{n+1}) \subseteq \mathcal{Y}$ for a new data pair $(X_{n+1}, Y_{n+1})$ such that $\mathbb{P}\big(Y_{n+1} \in \widehat{C}(X_{n+1})\big) \ge 1 - \alpha$ for a significance level $\alpha$ (e.g., $\alpha=0.01$), while remaining robust to paraphrase perturbations of the input $X_{n+1}$.
To ensure rigorous coverage guarantees, we work within the conformal prediction framework. Within this framework, two widely used approaches are split conformal prediction (split CP) and quasi-conditional conformal prediction (QCCP).

\vspace{-8pt}
\paragraph{Split CP.}  
Given an i.i.d. calibration set $\{(X_i,Y_i)\}_{i=1}^n$, split CP \citep{vovk2005algorithmic,shafer2008tutorial} constructs prediction sets based on a \emph{nonconformity score} $S:\mathcal{X}\times\mathcal{Y}\to\mathbb{R}$, which measures how atypical a candidate label $y$ is for an input $x$. 
The prediction set for a new input $X_{n+1}$ is defined as
\begin{equation}
\label{eq:CI-CP}
\widehat{C}(X_{n+1}) = \{y : S(X_{n+1},y) \leq \widehat{\gamma}_\alpha \},
\end{equation}
where $\widehat{\gamma}_\alpha$ is the empirical $(1-\alpha)$ quantile of $\{S(X_i,Y_i)\}_{i=1}^n \cup \{\infty\}$.  
Split CP is simple and distribution-free, but it relies on a single global threshold $\widehat{\gamma}_\alpha$ for all inputs, which can be overly conservative in heterogeneous input spaces.

\vspace{-8pt}
\paragraph{Quasi-Conditional CP.}
Quasi-conditional conformal prediction (QCCP) \citep{gibbs2023conformal} extends split CP by allowing thresholds to depend on the input. Given a precollected i.i.d. calibration set $\{(X_i, Y_i)\}_{i=1}^{n}$, it first computes nonconformity scores $S_i = S(X_i,Y_i)$ for $i=1,\ldots,n$. These scores are then used to fit the following quantile regressor, which for any given value of $S$, is defined as,
\begin{equation}
\label{eq:augqr}
\widehat{g}_S = \arg\min_{g\in\mathcal{F}}\left[
\sum_{i=1}^{n}\ell_\alpha\big(g(X_i),S_i\big)
+ \ell_\alpha\big(g(X_{n+1}),S\big)
\right],
\end{equation}
where $\mathcal{F}=\{\phi(\cdot)^\top\beta:\beta\in\mathbb{R}^d\}$ denotes a linear function class over a feature mapping $\phi$ (to be specified later), and $\ell_\alpha$ is the pinball loss defined by $\ell_{\alpha}(\theta, S) = (S - \theta)\,(\mathbf{1}_{\{S \geq \theta\}} - \alpha)$. At inference time, for a new input $X_{n+1}$, we construct the prediction set\footnote{Here, we use the notation $\widehat g_{S(X_{n+1}, y)}$ to emphasize that the learned threshold function depends on the score value $S(X_{n+1}, y)$ which is evaluated on the candidate input--label pair $(X_{n+1}, y)$.}
\begin{equation}
\label{eq:predset}
\widehat{C}(X_{n+1}) =
\left\{
y \,:\, S(X_{n+1},y)\le %
\widehat g_{S(X_{n+1}, y)}{(X_{n+1})}
\right\}.
\end{equation} 
Compared to split CP (which uses a global threshold $\widehat{\gamma}_\alpha$), QCCP adapts thresholds via $\widehat{g}_S$ based on the feature mapping $\phi$; see \citet{gibbs2023conformal} and Appendix~\ref{append:QCCP} for more details.

\section{Paraphrase-Aware Nonconformity Scores for Conformal Prediction}
\label{sec:scores}

 \begin{figure*}[t!]
    \centering
    \includegraphics[width=\linewidth]{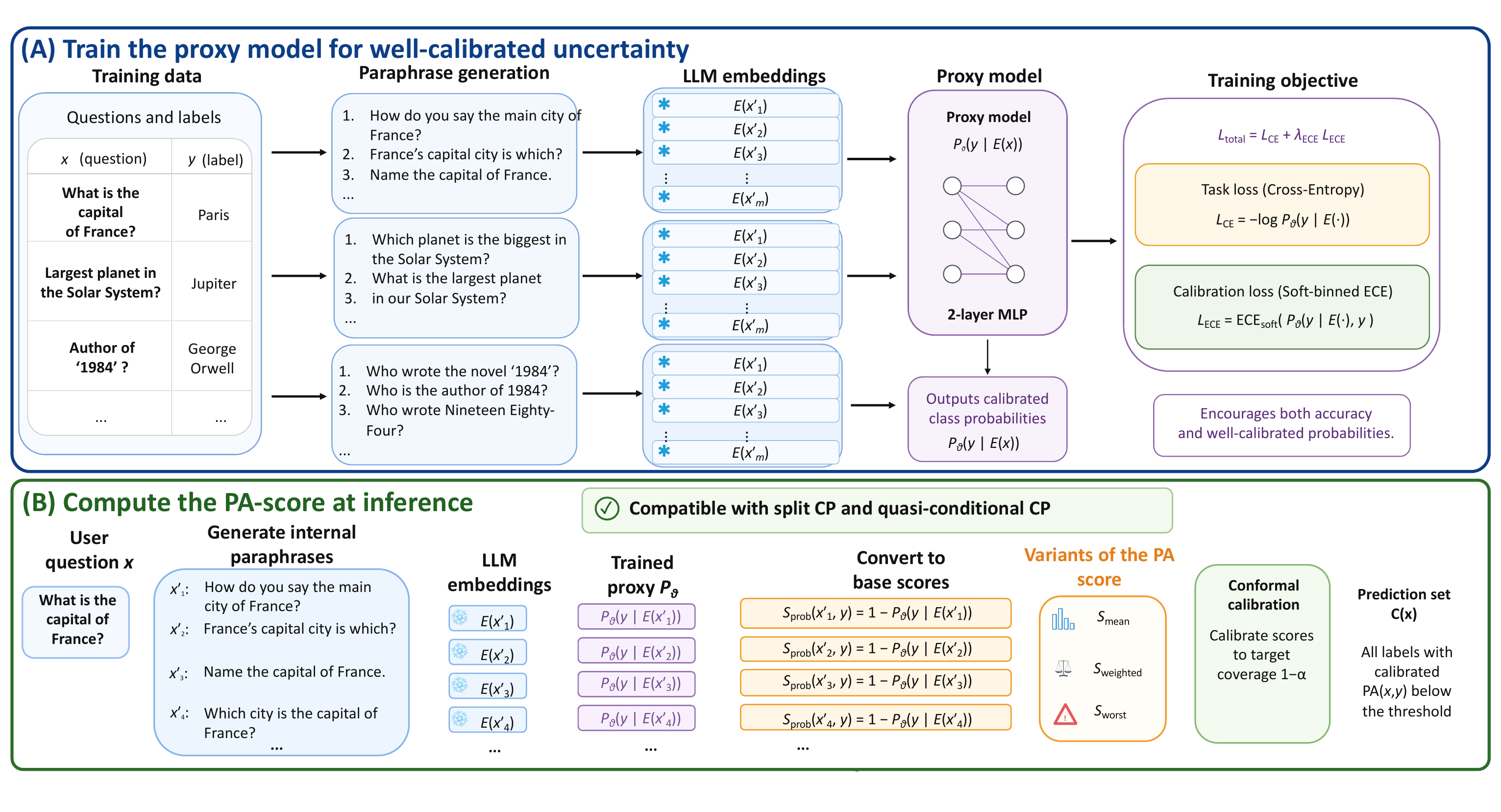}
    \vspace{-1.5em}
\caption{{Schematic illustration of our proposed PA score integrated with CP.}}
    \label{fig:fig1-motivation}
    \vspace{-1em}
\end{figure*}

A key limitation of existing CP methods is that the nonconformity scores $S$ are computed from a single prompt, and thus fail to capture variability induced by semantically equivalent rewordings. As a result, even minor paraphrases can lead to unstable uncertainty estimates. 
To address this issue, we introduce \emph{paraphrase-aware} (PA) nonconformity scores, which explicitly incorporate semantic variability by aggregating predictions across paraphrases. These scores can be integrated with either split CP or QCCP, yielding robust prediction sets to reworded inputs while retaining statistical coverage guarantees.

\vspace{-3pt}
\paragraph{Overview of Score Construction.}  
At a high level (Figure~\ref{fig:fig1-motivation}), our approach constructs uncertainty scores by aggregating predictions across paraphrases, thereby capturing variability induced by semantic rewordings. 
It proceeds as follows. For each input question $x$, we generate a set of paraphrases $\mathcal{B}(x)$. For each paraphrased input $x' \in \mathcal{B}(x)$, we extract its semantic embedding $E(x')$ (e.g., the corresponding LLM hidden state), resulting in a collection of embeddings $\{E(x') : x' \in \mathcal{B}(x)\}$.
We then train a lightweight \emph{proxy classifier} to map these embeddings to calibrated predictive probabilities. Based on these probabilities, we construct paraphrase-aware nonconformity scores, denoted by $S(x, y)$, which aggregate information across the paraphrase set $\mathcal{B}(x)$. The resulting scores are subsequently calibrated via split CP or QCCP to produce the final prediction set.
Due to space constraints, the full procedure is summarized in Algorithm~\ref{alg:full} in the appendix. Additional details on paraphrase quality and runtime are provided in Appendices~\ref{appendix:paraphrase-quality} and~\ref{appendix:runtime-analysis}.

\vspace{-8pt}
\paragraph{Motivation for a Proxy Classifier.}
Raw LLM probabilities are often overconfident and poorly calibrated,\footnote{When well calibrated, the proxy outputs are also interpretable. For example, if $P_\theta(\texttt{Paris} \mid E(x)) = 0.8$, this indicates that approximately $80\%$ of semantically similar inputs have ``Paris'' as the correct answer.} as they are trained for next-token prediction rather than uncertainty quantification \citep{nado2022uncertainty,ulmer2022exploring}. 
To address this limitation, we introduce a lightweight proxy classifier $P_\theta(y \mid E(x))$ that produces more reliable predictive probabilities for downstream decision making. The proxy takes the last-layer hidden state of the LLM as input and outputs a probability distribution over the label space $\mathcal{Y}$.
In our implementation, the proxy is a shallow two-layer MLP trained on LLM embeddings $\{E(X_i)\}$ with labels $\{Y_i\}$, using a combination of task loss (cross-entropy) and a calibration-aware loss (soft-binned ECE; \citealp{karandikar2021soft}). %
This design improves calibration while remaining computationally efficient. {The main variant requires hidden-state access and labeled proxy-training data. The blackbox API extension and full online generation cost are reported in Appendices~\ref{appendix:api-only} and~\ref{appendix:runtime-analysis}.}

{\begin{remark}
Since the PA score depends on embeddings $E(\cdot)$, one may be concerned about sensitivity to representation choices. In Appendix~\ref{sec:pooling-sensitivity}, we show that the PA score remains robust across a range of embedding variants, including different ways of summarizing token-level representations (e.g., pooling) and selecting model layers.
\end{remark}}

\vspace{-8pt}
\paragraph{Paraphrase-Aware Nonconformity Scores.}
We now introduce paraphrase-aware nonconformity scores, the core component of our method. 
A standard choice in conformal prediction is the probability-based score \citep{sadinle2019least}, 
$S_{\mathrm{prob}}(x,y) = 1 - P_\theta(y \mid E(x))$, which measures the model's lack of confidence for a single prompt. 
However, this score fails to capture instability under paraphrasing.

Given a paraphrase set $\mathcal{B}(x) = \{x'_1,\ldots,x'_m\}$ for each input $x$, we construct paraphrase-aware scores by aggregating $S_{\mathrm{prob}}(x',y)$ across paraphrases. 
As summarized in Table~\ref{tab:pa-scores}, we consider three variants.
The mean score $S_{\mathrm{mean}}$ averages base scores across paraphrases, capturing overall semantic stability. 
The weighted score $S_{\mathrm{weighted}}$ emphasizes semantically similar paraphrases by assigning higher weights to closer variants in the embedding space. 
The worst-case score $S_{\mathrm{worst}}$ takes the maximum across paraphrases, making it sensitive to the most difficult rewording.

\begin{table}[t]
\centering
\caption{Paraphrase-aware nonconformity scores. The weight is defined as $w(x,x')=1/(1+\|E(x)-E(x')\|_2)$, an inverse-distance weighting in the embedding space, so paraphrases closer to the original input receive larger weights. We use $\mathcal{T}(\cdot)$ to denote the paraphrase operator.}
\label{tab:pa-scores}

\resizebox{\textwidth}{!}{
\begin{tabular}{p{0.12\textwidth} p{0.28\textwidth} p{0.24\textwidth} p{0.28\textwidth}}
\toprule
\textbf{Score} & \textbf{Definition} & \textbf{Interpretation} & \textbf{Properties} \\
\midrule

Mean 
& $\frac{1}{m}\sum_{x'\in\mathcal{B}(x)} S_{\mathrm{prob}}(x',y)$ 
& Captures average semantic stability 
& Most efficient; may miss rare paraphrases \\

Weighted 
& {\tiny $\frac{\sum_{x'\in\mathcal{B}(x)} w(\mathcal{T}(x),x')\,S_{\mathrm{prob}}(x',y)}{\sum_{x'\in\mathcal{B}(x)} w(\mathcal{T}(x),x')}$}
& Emphasizes semantically similar paraphrases 
& Balances compactness and robustness \\

Worst-\allowbreak case 
& $\max_{x'\in\mathcal{B}(x)} S_{\mathrm{prob}}(x',y)$ 
& Sensitive to hardest paraphrase 
& Robust but conservative; sensitive to low-quality paraphrases \\

\bottomrule
\end{tabular}}
\vspace{-15pt}
\end{table}

\vspace{-8pt}
\paragraph{Choice of the Class Function $\phi$.} 
We now describe how we construct the class function $\phi$ (introduced in \eqref{eq:augqr}), which captures semantic similarity among input questions.
At a high level, $\phi$ groups semantically similar inputs so that downstream procedures operate on coherent groups rather than a single global structure. Since most benchmarks lack explicit type annotations (e.g., commonsense vs.\ factual), we construct these groups automatically.
Specifically, we embed each question–context pair using a pretrained SBERT encoder (\texttt{all-MiniLM-L6-v2}) \citep{reimers-2019-sentence-bert}, fit $K$-means on the calibration embeddings only, and assign each test input to its nearest centroid; the resulting cluster assignment defines $\phi(x)$.
Formally, this partitions $\mathcal{X}$ into $N$ clusters $\{\mathcal{G}_j\}_{j=1}^N$, and we set $\phi(x) = \sum_{j=1}^N \mathbf{1}\{x \in \mathcal{G}_j\} \cdot \phi_j$, where $\phi_j$ denotes the representative vector of cluster $j$. The number of clusters $K$ can be fixed or selected via the silhouette score on calibration embeddings.
This construction yields semantically coherent groupings and enables conditioning on latent question structure without requiring manual annotations. For more implementation details, see Appendix~\ref{appendix:detailed-implement}.

\section{Theoretical Analysis}
\label{sec:theoretical-analysis}
When integrated into standard CP frameworks, our PA scores inherit the original theoretical guarantees under exchangeability, since these guarantees are agnostic to the specific choice of nonconformity score. 
However, such guarantees rely on the exchangeability of the resulting scores, which may be violated under distributional shifts.
We show that, in certain structured settings, our PA construction can restore exchangeability at the score level, even when it does not hold for the original calibration–test pairs. We formalize this intuition below.
\begin{asmp}[Distributional closedness]
\label{asmp:main}
Let $\mathcal{T}$ be a Markov kernel on $\mathcal{X}$, where 
$\mathcal{T}(\cdot \mid x)$ denotes the distribution of a random paraphrase of $x$.
We say $\mathcal{T}$ is distributionally closed if, for any $x \in \mathcal{X}$, $\mathcal{T}(\cdot \mid x') = \mathcal{T}(\cdot \mid x)$ for $\mathcal{T}(\cdot \mid x)\text{-almost every } x' .$
\end{asmp}

\begin{prop}[Data exchangeability under distribution shift]
\label{prop:new-exchangability}
Let Asm.~\ref{asmp:main} hold, and $\{(X_i, Y_i)\}_{i=1}^{n+1}$ be exchangeable (or i.i.d.). 
We denote by $\mathcal{T}(X_{n+1})$ a random paraphrase of $X_{n+1}$ drawn from the kernel $\mathcal{T}(\cdot \mid X_{n+1})$.
Then $(\mathcal{T}(X_{n+1}), Y_{n+1})$ may NOT be exchangeable with $\{(X_i, Y_i)\}_{i=1}^{n}$ due to the shift induced by $\mathcal{T}$. 
However, applying $\mathcal{T}$ to all inputs restores exchangeability, i.e., $\bigl(\mathcal{T} \circ \mathcal{T}(X_{n+1}), Y_{n+1}\bigr)$ is exchangeable with $\{(\mathcal{T}(X_i), Y_i)\}_{i=1}^n$.
\end{prop}

\vspace{-8pt}
Prop.~\ref{prop:new-exchangability} shows that although the transformation $\mathcal{T}$ breaks exchangeability between calibration and test samples, this issue can be resolved by applying $\mathcal{T}$ across all data.
The key mechanism is that, under Asm.~\ref{asmp:main}, the paraphrase distribution is stable under repeated rewording $\mathcal{T}$.
While $(\mathcal{T}(X_{n+1}), Y_{n+1})$ may not align with the original calibration distribution, $\mathcal{T}(X_{n+1})$ (and its multiple collections) becomes compatible with those from the calibration data.

\begin{lem}[Score exchangeability]
\label{lem:score}
Let $\{(X_i, Y_i)\}_{i=1}^{n+1}$ be i.i.d. For our proposed scores $S \in \{ S_{\mathrm{mean}}, S_{\mathrm{weighted}}, S_{\mathrm{worst}}\}$, the exchangeability properties summarized in Table~\ref{tab:score-exchangeability} hold, i.e., the test score is exchangeable with the corresponding calibration scores.
\end{lem}

\begin{table}[t]
\caption{Summary of score exchangeability in Lemma \ref{lem:score} where $\{(X_i, Y_i)\}_{i=1}^{n+1}$ is i.i.d. Here, (*) requires Asm.~\ref{asmp:main}.}
\label{tab:score-exchangeability}
\centering
\small
\vspace{-2pt}
\resizebox{0.57\textwidth}{!}{
\begin{tabular}{l l l l}
\toprule
\textbf{Setting.} & \textbf{Calibration Data} & \textbf{Test Sample} & \textbf{Exchangeable} \\
\midrule
Normal 
& $\{(X_i,Y_i)\}_{i=1}^n$ 
& $(X_{n+1},Y_{n+1})$ 
& \checkmark \\
Full 
& $\{(\mathcal{T}(X_i),Y_i)\}_{i=1}^n$ 
& $(\mathcal{T}(X_{n+1}),Y_{n+1})$ 
& \checkmark \\
Semi 
& $\{(X_i,Y_i)\}_{i=1}^n$ 
& $(\mathcal{T}(X_{n+1}),Y_{n+1})$ 
& \checkmark\ (*) \\
\bottomrule
\end{tabular}}
\end{table}

\vspace{-10pt}

\paragraph{Implications for Paraphrasing.}
{Meaning-preserving paraphrases express the same underlying content through different wordings. Distributional closedness formalizes the stronger requirement that this rewording leaves the distribution of generated paraphrases unchanged.}
Under this view, our PA scores leverage multiple independent samples drawn from $\mathcal{T}(\cdot \mid x)$ to construct a robust statistic for each input.
Under Asm.~\ref{asmp:main}, this invariance extends to the joint distribution of these samples.
{Lemma~\ref{lem:score} then establishes exchangeability of the resulting scores under the stated score conditions. Consequently, our scores retain the corresponding marginal coverage guarantees within CP and quasi-conditional guarantees within QCCP under its joint feature--score assumptions, including when test inputs undergo rewording.}
A formal statement together with proofs is given in Appendix~\ref{append:coverage}.

\begin{remark}
{The analysis also quantifies departures from exact distributional closedness. Under the assumptions in Appendix~\ref{appen:approximate-main}, a total-variation bound controls the change in coverage induced by rewording. Appendix~\ref{appendix:theory-impl-gap} complements this result with an empirical diagnostic of score-distribution shift and its effect on coverage.}
\end{remark}
\vspace{-6pt}
\section{Empirical Performance}
\vspace{-4pt}

\paragraph{Datasets.}
We evaluate our framework on seven benchmark datasets. Five are general multiple-choice question answering (MCQA) datasets from the LLM-Uncertainty-Benchmark \citep{ye2024benchmarking}: {MMLU}, {CosmosQA}, {HellaSwag}, {HaluDial}, and {HaluSum}, covering question answering (\textbf{QA}), reading comprehension (\textbf{RC}), commonsense inference (\textbf{CI}), dialogue response selection (\textbf{DRS}), and document summarization (\textbf{DS}). 
We further include two medical MCQA datasets, \textbf{MedMCQA} \citep{pal2022medmcqa} and \textbf{MedQA} \citep{jin2021disease}, to evaluate robustness in specialized domains. For each dataset, we sample 10,000 questions.
Following the LLM-Uncertainty-Benchmark protocol \citep{ye2024benchmarking}, we augment the response space $|\mathcal{Y}|$ of all datasets to six options by adding \texttt{I don't know}'' and \texttt{None of the above}.''
Together, these datasets span diverse reasoning tasks and domains, enabling a comprehensive evaluation of both generalization and robustness.

\begin{table*}[t!]
\centering
\vspace{-2mm}
\caption{Performance under split CP and QCCP with \texttt{Qwen3-8B} on normal and fully reworded (Full) test sets across five benchmarks. 
For coverage rate (CR, \%), bold indicates values closest to the nominal $90\%$ target. For set size (SS, $\downarrow$), bold indicates the smallest value.}
\vspace{-1mm}
\label{tab:qwen-paraphrase-perdataset}
\small
\sisetup{round-mode=places,round-precision=2,table-number-alignment=center}
\resizebox{\textwidth}{!}{
\begin{tabular}{l|
  *{2}{S} | *{2}{S} | *{2}{S} | *{2}{S} | *{2}{S}}
\toprule
& \multicolumn{2}{c|}{\textbf{QA}} & \multicolumn{2}{c|}{\textbf{RC}} &
  \multicolumn{2}{c|}{\textbf{CI}} & \multicolumn{2}{c|}{\textbf{DRS}} &
  \multicolumn{2}{c}{\textbf{DS}} \\
\cmidrule(lr){2-11}
\textbf{Method} & {Normal} & {{Full}} & {Normal} & {{Full}} & {Normal} & {{Full}} & {Normal} & {{Full}} & {Normal} & {{Full}} \\
\midrule
\multicolumn{11}{l}{\textit{Coverage Rate (CR, \%) Better if closer to $90\%$ --- Split CP}} \\
LAC        
& 89.60 & 98.83 & \textbf{90.53} & 90.47 & 89.73 & \textbf{90.03} & 99.00 & 99.10 & 89.27 & 89.10 \\
APS        
& 99.73 & 98.90 & 99.07 & 99.20 & 98.90 & 99.00 & 100.00 & 99.97 & 100.00 & 98.50 \\
{SPUQ-adapt.}       
& 90.37 & 89.77 & 91.27 & \textbf{90.20} & \textbf{89.87} & 89.77 & \textbf{90.90} & 91.13 & \textbf{89.80} & \textbf{89.73} \\
{TRON-adapt.}       
& \textbf{89.77} & 90.40 & \textbf{90.53} & 90.67 & 89.80 & 90.20 & 90.93 & 91.53 & 89.53 & 89.70 \\
PA   
& 90.93 & \textbf{89.90} & 91.30 & 91.47 & 90.47 & 90.20 & 92.70 & \textbf{90.40} & 95.13 & 90.47 \\
\addlinespace[3pt]
\multicolumn{11}{l}{\textit{Coverage Rate (CR, \%) Better if closer to $90\%$ --- QCCP}}\\
LAC        
& 94.47 & 95.13 & 92.23 & 93.67 & 91.87 & 91.87 & 97.80 & 96.40 & 94.47 & 94.23 \\
APS        
& 97.10 & 96.90 & 98.63 & 97.50 & 98.60 & 99.17 & 93.10 & 94.33 & 97.10 & 95.73 \\
{SPUQ-adapt.}       
& 91.23 & \textbf{90.43} & 91.67 & \textbf{90.57} & 90.67 & \textbf{90.23} & 91.67 & 91.23 & \textbf{90.00} & \textbf{90.37} \\
{TRON-adapt.}       
& 91.03 & 90.93 & \textbf{90.70} & 91.10 & \textbf{90.07} & 90.27 & \textbf{91.07} & 91.47 & 90.03 & 90.47 \\
PA   
& \textbf{91.00} & 90.80 & 91.60 & 91.83 & 92.00 & 91.83 & 92.90 & \textbf{90.40} & 95.63 & 91.17 \\
\midrule
\multicolumn{11}{l}{\textit{Set Size (SS) $\downarrow$ --- Split CP}}\\
LAC        
& 2.18 & 5.62 & 1.56 & 1.80 & 1.66 & 1.76 & 5.72 & 5.72 & 2.52 & 2.73 \\
APS        
& 5.72 & 5.64 & 5.57 & 5.61 & 5.60 & 5.61 & 6.00 & 5.97 & 5.97 & 5.65 \\
{SPUQ-adapt.}       
& 2.48 & 2.65 & 1.65 & 1.77 & 1.84 & 1.79 & 3.29 & 3.32 & 2.68 & 2.73 \\
{TRON-adapt.}       
& 2.17 & 2.58 & 1.57 & 1.83 & 1.67 & 1.77 & 3.38 & 3.37 & 2.57 & 2.82 \\
PA   
& \textbf{2.05} & \textbf{2.19} & \textbf{1.25} & \textbf{1.25} & \textbf{1.44} & \textbf{1.37} & \textbf{1.55} & \textbf{1.36} & \textbf{1.00} & \textbf{1.08} \\
\addlinespace[3pt]
\multicolumn{11}{l}{\textit{Set Size (SS) $\downarrow$ --- QCCP}}\\
LAC        
& 3.58 & 4.28 & 2.47 & 3.11 & 2.61 & 2.68 & 5.03 & 4.74 & 4.27 & 4.30 \\
APS        
& 4.59 & 4.78 & 4.79 & 4.73 & 4.89 & 5.00 & 4.45 & 4.68 & 5.03 & 4.83 \\
{SPUQ-adapt.}       
& 2.66 & 2.83 & 1.88 & 1.96 & 2.18 & 2.15 & 3.37 & 3.32 & 2.85 & 2.92 \\
{TRON-adapt.}       
& 2.45 & 2.84 & 1.77 & 2.05 & 1.99 & 2.08 & 3.37 & 3.35 & 2.79 & 3.01 \\
PA   
& \textbf{2.29} & \textbf{2.49} & \textbf{1.44} & \textbf{1.48} & \textbf{1.81} & \textbf{1.76} & \textbf{1.59} & \textbf{1.41} & \textbf{1.24} & \textbf{1.32} \\
\bottomrule
\end{tabular}
}
\vspace{-20pt}
\end{table*}

\vspace{-10pt}
\paragraph{Baselines and Evaluation Metrics.}
{We compare LAC~\citep{sadinle2019least} and APS~\citep{romano2020classification} with label-wise adaptations inspired by TRON~\citep{wangsample} and SPUQ~\citep{gao2024spuq}, denoted TRON-adapt.\ and SPUQ-adapt.\ Let $p_y(x)$ denote the raw LLM probability of option $y$, and let $\bar p_y(x)=|\mathcal B(x)|^{-1}\sum_{\tilde x\in\mathcal B(x)}p_y(\tilde x)$ denote its average across paraphrases. The implementations use the scores $-\log(p_y(x)+10^{-12})$ and $-\log(\bar p_y(x)+10^{-12})$, respectively. Figure labels TRON and SPUQ refer to the corresponding adaptations. We separately examine variants $1-p_y(x)$ and $1-\bar p_y(x)$ in Appendix~\ref{appendix:baseline_choice}. } Data are split into training, calibration, and test sets with a 40/30/30 ratio. We report {Coverage Rate (CR)}, the fraction of test points whose true label lies in the prediction set, and {Set Size (SS)}, the average size of the prediction set. The target coverage is $1-\alpha = 0.90$.

\vspace{-10pt}
\paragraph{Paraphrase-Based Distribution Shift.}
To simulate controlled distribution shifts, we generate reworded inputs by prompting a local LLM to paraphrase each question while preserving its semantics. Diversity is introduced via stochastic decoding (temperature sampling), followed by filtering to ensure validity and well-formedness. When generation fails, we retry with a stricter template or apply a rule-based rewrite. 
We denote the above overall (paraphrase) transformation by $\mathcal{T}$ for ease of reference.
Further implementation details are provided in Appendix~\ref{app:adv-paraphrase}.\footnote{We verify robustness to the choice of paraphrase generator and evaluate paraphrase quality in Appendix~\ref{appendix:different-llm-paraphrase-generator} and Appendix~\ref{appendix:paraphrase-quality} respectively.}

\vspace{-10pt}
\paragraph{Evaluation Settings.}
% We instantiate the three settings in Table \ref{tab:score-exchangeability} as follows.
{We instantiate the three settings in Table \ref{tab:score-exchangeability}  to distinguish the effect of rewording from a mismatch between calibration and deployment.} In the \textbf{normal} setting, all training, calibration, and test inputs are original human-written questions. {It provides a reference for coverage and prediction-set size for the two following scenarios.} In the \textbf{fully reworded} setting, we apply the same paraphrasing operator $\mathcal{T}$ (introduced above) to all inputs across training, calibration, and test sets, and compute all scores on reworded questions. 
{This tests how the LLM paraphrased input can influence the stability of different nonconformity scores.}
All CP scores are computed on these LLM-reworded questions. In the \textbf{semi-reworded} setting, training and calibration use original inputs, while test inputs are reworded by an LLM. {This represents deployment in which users or an upstream system rephrase inputs after calibration, creating a structured distribution shift. Normal and Fully reworded retain matched calibration and test distributions, so standard conformal validity applies under the corresponding exchangeability assumptions. Semi-reworded tests whether coverage and efficiency persist across the rewording shift, with theoretical validity governed by the additional conditions stated in our analysis.} Additional details are provided in Appendix~\ref{appendix:semi-reworded-setting}.

\subsection{Stable Coverage and Compact Sets under Rewording}
\label{sec:exp-paraphrase-robust}

\paragraph{Normal and Fully Reworded Settings.}
We first evaluate whether the proposed PA score remains robust under prompt rewording, where each test question is rephrased using the procedure described above. Results on \texttt{Qwen-3-8B} across five MCQA benchmarks are summarized in Table~\ref{tab:qwen-paraphrase-perdataset}.
Across both normal and fully reworded settings, PA consistently achieves the smallest prediction sets among all methods, while maintaining coverage close to the nominal $90\%$ target. In particular, PA's coverage is typically slightly conservative (ranging from $89.90\%$ to $95.63\%$), but avoids the severe over-coverage exhibited by APS ($93.10\%$--$100.00\%$). Compared to methods such as TRON and SPUQ, which can be closer to the nominal target, PA achieves substantially smaller set sizes, often reducing them by a factor of $2\times$ or more.
Overall, PA achieves a more favorable coverage--efficiency tradeoff, delivering compact prediction sets without sacrificing reliability too much.

\vspace{-10pt}
\paragraph{Semi-reworded Setting.}
To further assess robustness under distribution shifts induced by paraphrasing, we evaluate all methods in the semi-reworded (Semi.) setting (Table~\ref{tab:semi-reworded}). This setting is more challenging, as the calibration and test distributions no longer match.
Under this shift, PA continues to produce the most compact prediction sets in most cases, while maintaining competitive coverage. Although TRON and SPUQ are occasionally closer to the nominal $90\%$ target, they consistently require larger sets. Meanwhile, APS remains overly conservative, and LAC exhibits instability in several benchmarks.
{These fixed-$\alpha$ comparisons describe empirical robustness and compactness, so achieved coverage may differ across methods. Table~\ref{tab:matched-coverage-attribution} in the appendix gives a separate held-out, approximately coverage-matched comparison.}

\begin{table*}[t!]
\centering
\vspace{-1mm}
\caption{{Performance under the \textbf{semi-reworded (Semi.)} setting. Bold numbers indicate the best.}}
\vspace{-1mm}
\label{tab:semi-reworded}
\small
\resizebox{\textwidth}{!}{
\begin{tabular}{c|l|
  *{5}{c}|*{5}{c}}
\toprule
& \textbf{Method} 
& \textbf{QA} & \textbf{RC} & \textbf{CI} & \textbf{DRS} & \textbf{DS} 
& \textbf{QA} & \textbf{RC} & \textbf{CI} & \textbf{DRS} & \textbf{DS} \\
\cmidrule(lr){3-7}\cmidrule(lr){8-12}
& Semi. & \multicolumn{5}{c|}{Coverage Rate (CR, \%) Better if closer to $90\%$} & \multicolumn{5}{c}{Set Size (SS) $\downarrow$} \\
\midrule
\multirow{5}{*}{{Split CP}} 
& LAC 
& 89.83 & \textbf{88.60} & 89.20 & 99.10 & 88.23 
& 2.54 & 1.62 & 1.69 & 5.72 & 2.61 \\
& APS 
& 99.47 & 98.90 & 99.00 & 100.00 & 100.00
& 5.74 & 5.57 & 5.61 & 6.00 & 5.97 \\
& {SPUQ-adapt.} 
& 84.83 & 87.93 & 87.30 & 86.23 & 88.20
& \textbf{2.21} & 1.53 & 1.60 & 2.72 & 2.49 \\
& {TRON-adapt.} 
& \textbf{90.10} & \textbf{88.60} & \textbf{89.23} & \textbf{91.40} & \textbf{88.70}
& 2.55 & 1.62 & 1.69 & 3.36 & 2.66 \\
& PA   
& 92.07 & 92.20 & 93.23 & 95.90 & 94.23
& 2.26 & \textbf{1.31} & \textbf{1.56} & \textbf{1.70} & \textbf{1.00} \\
\midrule
\multirow{5}{*}{{QCCP}}
& LAC
& 92.13 & \textbf{90.40} & {89.43} & 96.97 & 91.60
& 3.43 & 2.43 & 2.28 & 5.07 & 3.70 \\
& APS
& 95.00 & 96.97 & 97.60 & 93.57 & 94.53
& 4.19 & 4.15 & 4.61 & 4.67 & 4.47 \\
& {SPUQ-adapt.}
& {87.10} & {88.27} & {87.90} & {88.10} & {88.07}
& 2.40 & 1.81 & 1.81 & 2.99 & 2.57 \\
& {TRON-adapt.}
& {{88.77}} & {89.17} & {88.53} & 91.57 & \textbf{{88.80}}
& 2.55 & 1.90 & 1.90 & 3.49 & 2.75 \\
& PA
& \textbf{89.50} & 90.67 & \textbf{89.53} & \textbf{91.10} & 94.20
& \textbf{2.01} & \textbf{1.39} & \textbf{1.36} & \textbf{1.70} & \textbf{1.11} \\
\bottomrule
\end{tabular}
}
\vspace{-10pt}
\end{table*}

\vspace{-10pt}
\paragraph{Generalization Across LLMs.}
Finally, we evaluate cross-model generalization using two models, \texttt{Llama-3.1-8B-Instruct} \citep{dubey2024llama} and \texttt{Phi-3-small-8k-Instruct} \citep{abdin2024phi3technicalreporthighly}, on MMLU under the normal, semi-reworded, and fully {reworded} settings. The results mirror the patterns previously: APS tends to inflate coverage, often overshooting the nominal 90\% level, while LAC, SPUQ, and TRON can be closer to nominal coverage in some cases but generally require larger prediction sets. In contrast, our score (PA) consistently produces the most compact prediction sets and maintains competitive coverage close to the nominal target.

\vspace{-10pt}
\paragraph{Extensions to Other Settings.}
Due to space constraints, many additional experimental results are deferred to Appendix~\ref{append:additional-experiments}. We provide a brief summary here.
First, we observe similar patterns on the remaining two medical MCQA datasets, where our method consistently produces smaller prediction sets while achieving coverage closer to the target level than competing methods.
Second, our method is versatile and extends naturally to other settings. Extensions to API-only scenarios, other tasks (e.g., factuality verification), and image modalities are discussed in Appendix~\ref{append:applicability} in greater detail.

\begin{table*}[t!]
\centering
\vspace{-1mm}
\caption{{Cross-LLM generalization on MMLU.}
Coverage Rate (CR, \%) closer to $90\%$ is better; Set Size (SS) $\downarrow$ is better.
Results are shown for normal, semi-reworded (Semi.), and fully reworded (Full).}
\vspace{-1mm}
\label{tab:generalize-llms}
\small
\resizebox{\textwidth}{!}{
\begin{tabular}{c|l|
  *{3}{c}|*{3}{c}|*{3}{c}|*{3}{c}}
\toprule
& & \multicolumn{6}{c|}{\texttt{Llama-3.1-8B}} 
  & \multicolumn{6}{c}{\texttt{Phi-3-small}} \\
\cmidrule(lr){3-8}\cmidrule(lr){9-14}
& & \multicolumn{3}{c|}{CR (\%)} & \multicolumn{3}{c|}{SS} 
  & \multicolumn{3}{c|}{CR (\%)} & \multicolumn{3}{c}{SS} \\
\cmidrule(lr){3-5}\cmidrule(lr){6-8}
\cmidrule(lr){9-11}\cmidrule(lr){12-14}
& \textbf{Method} 
& Normal & Semi. & Full & Normal & Semi. & Full 
& Normal & Semi. & Full & Normal & Semi. & Full \\
\midrule
\multirow{5}{*}{{Split CP}} 
& LAC  
& 88.97 & 86.50 & 88.87 & 2.63 & 2.75 & 2.95 
& 91.03 & \textbf{88.67} & 89.03 & 5.65 & 5.55 & 5.57 \\
& APS  
& 95.87 & 94.90 & 94.93 & 3.81 & 3.97 & 3.98 
& 99.53 & 99.67 & 99.67 & 5.98 & 5.98 & 5.98 \\
& {SPUQ-adapt.} 
& 88.50 & 85.40 & 89.63 & 2.68 & \textbf{2.69} & 3.04 
& 90.33 & 84.47 & 90.13 & 5.61 & 5.38 & 5.61 \\
& {TRON-adapt.} 
& 89.00 & 86.53 & 88.93 & 2.63 & 2.75 & 2.95 
& \textbf{90.00} & 87.17 & 89.37 & 5.60 & 5.49 & 5.58 \\
& PA    
& \textbf{89.50} & \textbf{90.03} & \textbf{89.87} & \textbf{2.52} & 2.70 & \textbf{2.75} 
& 90.67 & 91.50 & \textbf{90.07} & \textbf{3.59} & \textbf{3.63} & \textbf{3.62} \\
\midrule
\multirow{5}{*}{{QCCP}}
& LAC
& 90.03 & {87.77} & 89.17 & 2.85 & 2.90 & 3.16
& 91.83 & {87.37} & 91.73 & 5.67 & 5.51 & 5.67 \\
& APS
& 95.47 & 94.40 & 94.43 & 3.91 & 3.99 & 4.08
& 94.77 & 97.17 & 93.73 & 5.80 & 5.88 & 5.77 \\
& {SPUQ-adapt.}
& 89.43 & {87.20} & \textbf{89.90} & 2.92 & 2.85 & 3.23
& 91.17 & {84.73} & 91.07 & 5.65 & 5.40 & 5.64 \\
& {TRON-adapt.}
& \textbf{89.97} & {87.77} & 89.13 & 2.84 & 2.90 & 3.16
& 91.07 & {86.50} & \textbf{90.60} & 5.64 & 5.49 & 5.62 \\
& PA
& 89.57 & \textbf{89.63} & 90.40 & \textbf{2.49} & \textbf{2.60} & \textbf{3.02}
& \textbf{90.33} & \textbf{90.67} & 90.73 & \textbf{3.57} & \textbf{3.60} & \textbf{3.65} \\
\bottomrule
\end{tabular}
}
\vspace{-10pt}
\end{table*}

\section{In-Depth Analysis and Ablation Studies}
\vspace{-5pt}

\paragraph{Proxy Model Improves Accuracy and Calibration.} 

In our work, the proxy model $P_{\theta}$ is a lightweight two-layer MLP trained on top of frozen hidden states from the LLM. We find that $P_{\theta}$ produces more accurate and better-calibrated distributions than raw LLM logits. As shown in Figure~\ref{fig:an1-mlp-calibration}, the proxy consistently outperforms logits across benchmarks on all calibration metrics. For example, accuracy improves significantly (\textbf{DS} $0.597 \!\to\! 0.944$, \textbf{DRS} $0.510 \!\to\! 0.742$), while both Brier score and NLL decrease (\textbf{RC} Brier $0.075 \!\to\! 0.039$, NLL $2.513 \!\to\! 0.487$).
Here, Brier measures the mean squared error between predicted probabilities and true one-hot labels, while NLL penalizes assigning low probability to the correct answer; both capture calibration quality. This improvement arises because raw logits are optimized for next-token prediction rather than calibrated posteriors, often leading to miscalibration and poor class separation. In contrast, the proxy leverages hidden states, which encode richer task-relevant signals, and is trained with a calibration-aware loss. As a result, it achieves both higher accuracy and better calibration.

\begin{figure}[!h]
\vspace{-3mm}
\centering
\includegraphics[width=\linewidth]{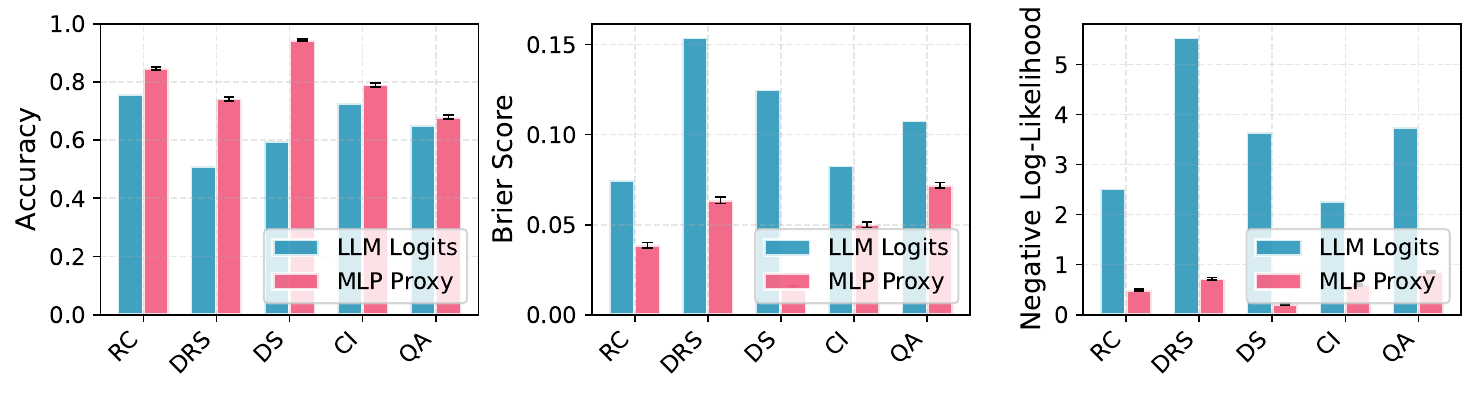}
\vspace{-5mm}
\caption{{The proxy improves both task performance and calibration consistently}. Evaluation of prediction accuracy and calibration metric using proxy model vs.\ raw LLM logits across datasets. Left: accuracy (higher is better). Middle: Brier (lower is better). Right: NLL (lower is better). }
\label{fig:an1-mlp-calibration}
\vspace{-2mm}
\end{figure}

\vspace{-10pt}
{\paragraph{Where Do the Gains Come From?}
\label{sec:ablation-no-proxy}
We ask how much of PA's improvement comes from learning a proxy and what paraphrases add during training and inference. Table~\ref{tab:factorized-attribution} first compares raw probabilities with proxies trained on either probability vectors or hidden states. \emph{Orig-only} uses one copy of each original training question; \emph{augmented} adds up to six meaning-preserving paraphrases per question, retaining its label. We then compare single-view and six-view inference using the same augmented proxy checkpoint, isolating the effect of aggregation.
The hidden-state proxy accounts for most of the compactness gain, reducing Semi SS from 2.8929 to 1.3725 without paraphrase training. Augmentation improves rewritten-input accuracy from 79.80\% to 80.38\% and reduces true-label score drift from 0.0596 to 0.0574. With this checkpoint fixed, six-view aggregation reduces set change by 15.8\% and Semi SS from 1.3587 to 1.3331, with coverage changing from 89.96\% to 89.81\%. Raw-probability averaging instead increases set change from 0.3561 to 0.5785. These comparisons distinguish proxy-driven compactness from the additional effects of paraphrase training and aggregation. Appendix~\ref{appendix:proxy-vs-pa} further controls for augmentation's larger training budget by repeating original questions to match training rows and optimizer updates.
}

\begin{table*}[t]{
\centering
\caption{{Factorized attribution on \texttt{Qwen3-8B} (five datasets,
three seeds; equal-weight mean over 15 dataset--seed results).
Rewrite Acc.\ is label accuracy on rewritten inputs. Semi CR and SS are
empirical coverage (\%) and average set size under split CP in the semi-reworded setting
at $\alpha=0.1$. Set change is the mean number of labels whose membership
differs between the prediction sets for an original input and its rewrite,
$|\widehat C(x)\triangle\widehat C(\mathcal T(x))|$.
The Full PA and augmented-proxy single-view rows use the same trained
proxy checkpoint; only inference-time aggregation differs.}}
\label{tab:factorized-attribution}
\small
\setlength{\tabcolsep}{4.5pt}
\begin{tabular}{@{}llrrrr@{}}
\toprule
Score and proxy training & Inference &
Rewrite Acc.\ $\uparrow$ & Semi CR & Semi SS $\downarrow$ &
Set change $\downarrow$ \\
\midrule
Raw probabilities; no proxy       & Single view       & 65.74 & 91.36 & 2.8929 & 0.3561 \\
Raw-probability proxy; Orig-only  & Single view       & 70.16 & 89.74 & 2.2032 & 0.2369 \\
Hidden-state proxy; Orig-only    & Single view       & 79.80 & 89.58 & 1.3725 & 0.2009 \\
Hidden-state proxy; augmented    & Single view       & 80.38 & 89.96 & 1.3587 & 0.2029 \\
Full PA; same augmented proxy     & Six-view $S_{\rm mean}$ & 80.34 & 89.81 & 1.3331 & 0.1708 \\
Raw probabilities; no proxy       & Six-view mean     & 65.53 & 86.11 & 2.0394 & 0.5785 \\
\bottomrule
\end{tabular}
\vspace{-1.5em}
}\end{table*}

\paragraph{Better Class-conditional Coverage.}
We group questions into semantic classes (or clusters) and evaluate class-conditional coverage by visualizing empirical coverage rates on different classes. As shown in Figure~\ref{fig:an2-classwise-condconf}, PA consistently tracks the $90\%$ target across different classes, while LAC and APS overshoot and the TRON / SPUQ baselines exhibit larger per-cluster swings.

\begin{figure}[!ht]
    \centering
    \includegraphics[width=0.4\textwidth]{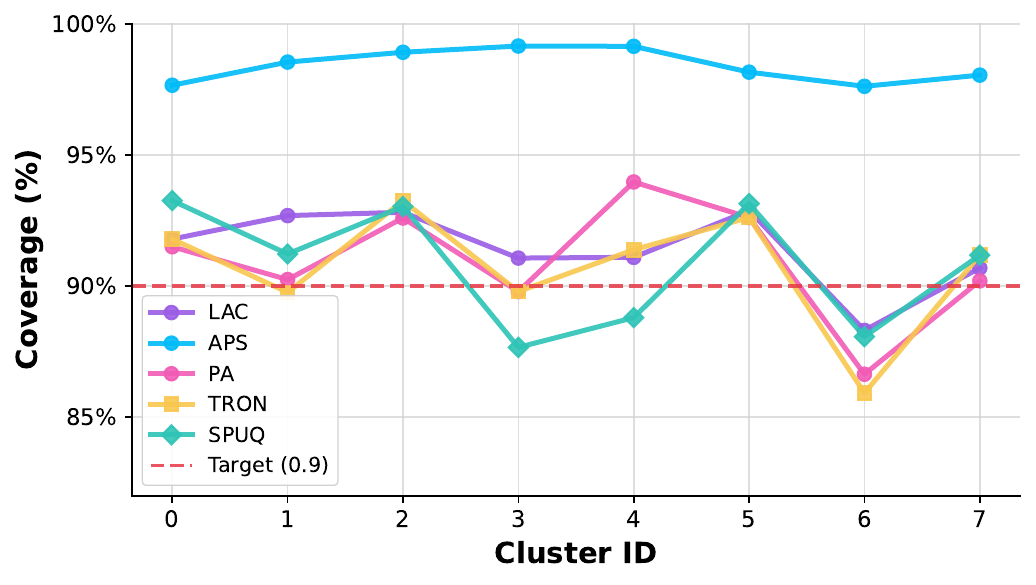}
    \includegraphics[width=0.4\textwidth]{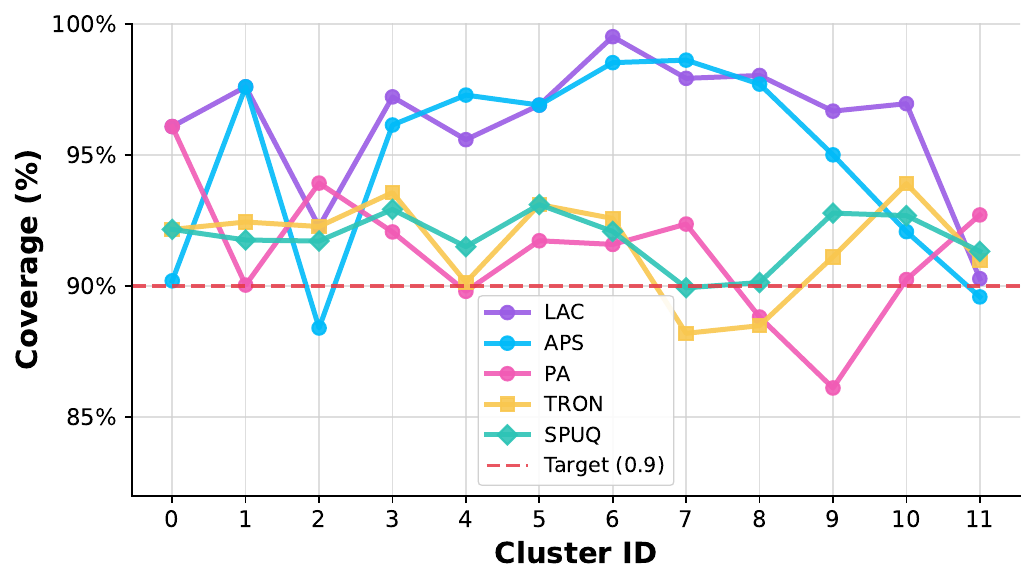}
    \vspace{-2mm}
    \caption{Evaluation of class-conditional coverage with QCCP on the CosmosQA dataset (left) and HaluDial dataset (right). Results on additional datasets are provided in Figure \ref{fig-appendix:an2-classwise-condconf-others}.}
    \label{fig:an2-classwise-condconf}
\end{figure}

\vspace{-5pt}
\textbf{Effect of Coverage Levels.}
To examine the effect of the level $\alpha$, we vary $\alpha \in \{0.2, 0.05, 0.01\}$ and evaluate both coverage and prediction set size again. As shown in Figure~\ref{fig:risk-level-results}, our PA score consistently tracks the nominal $(1-\alpha)$ target while maintaining the most compact sets. For example, at $\alpha=0.05$, it achieves $\approx 95\%$ coverage with an average set size of $2.9$, compared to LAC’s inflated $4.8$ and APS’s $4.7$. APS dramatically overshoots at $\alpha=0.2$ (95\% coverage with set size $4.5$ vs.\ an $80\%$ target), while LAC, TRON, and SPUQ track the target reasonably well at small $\alpha$ but produce strictly larger sets than PA at every level. Overall, PA consistently outperforms across different levels.

\begin{figure*}[t]
\centering
\vspace{-10pt}
\begin{minipage}{0.65\textwidth}
    \centering
    \begin{subfigure}[b]{0.49\linewidth}
        \includegraphics[width=\linewidth]{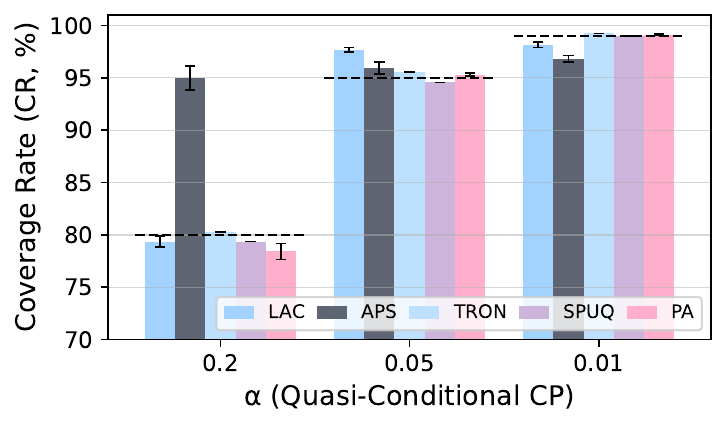}
    \end{subfigure}
    \hfill
    \begin{subfigure}[b]{0.49\linewidth}
        \includegraphics[width=\linewidth]{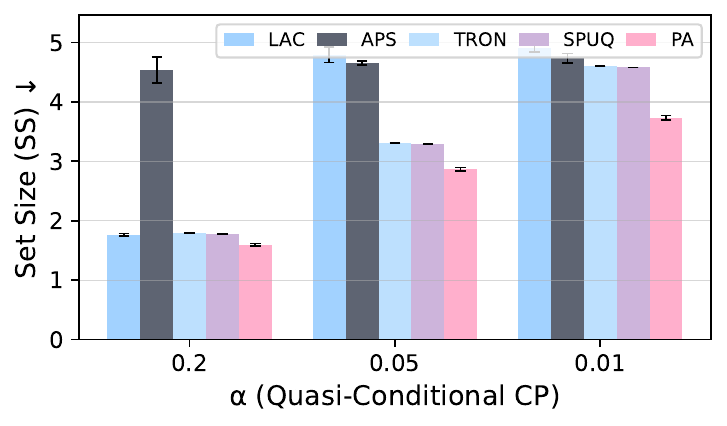}
    \end{subfigure}
    
    \vspace{-2mm}
    \caption{Coverage rate (left) and prediction set size (right) vs.\ different values of the level $\alpha$ on MMLU under QCCP.}
    \label{fig:risk-level-results}
\end{minipage}
\hfill
\begin{minipage}{0.34\textwidth}
    \centering
    \vspace{-10pt}
    \captionof{table}{Effect of paraphrase budget $m$ under split CP and QCCP.}
    \vspace{-2mm}
    \label{tab:ablation-paraphrase-budget}
    \sisetup{table-number-alignment=center,round-mode=places,round-precision=2}
    \resizebox{\textwidth}{!}{
    \begin{tabular}{lccc}
    \toprule
    & \textbf{$m{=}2$} & \textbf{$m{=}4$} & \textbf{$m{=}6$} \\
    \midrule
    \multicolumn{4}{l}{\textit{Coverage Rate (CR, \%)}}\\
    PA (marginal)    & 89.77 & 89.67 & \textbf{90.53} \\
    PA (conditional) & 90.40 & 90.40 & \textbf{90.93} \\
    \addlinespace[2pt]
    \multicolumn{4}{l}{\textit{Set Size (SS)}}\\
    PA (marginal)    & 1.92  & \textbf{1.86}  & 1.88 \\
    PA (conditional) & 2.19  & 2.16  & \textbf{2.15} \\
    \bottomrule
    \end{tabular}}
    \vspace{-2mm}
\end{minipage}
\vspace{-15pt}
\end{figure*}

\vspace{-10pt}
\paragraph{Larger Paraphrase Budgets Yield Smaller Prediction Sets.} 
We now study the effect of the paraphrase budget $m$ for \texttt{Qwen3-8B} \citep{yang2025qwen3} on the MMLU dataset. As shown in Table~\ref{tab:ablation-paraphrase-budget}, increasing $m \in \{2,4,6\}$ leads our PA score to achieve higher coverage and smaller set sizes. This is intuitive: more paraphrases provide richer semantic views of each question, which reduces variance in the aggregated score and enables more precise calibration.

\begin{figure}[ht]
    \centering
        \begin{subfigure}[t]{0.34\textwidth}
        \centering
        \includegraphics[width=\linewidth]{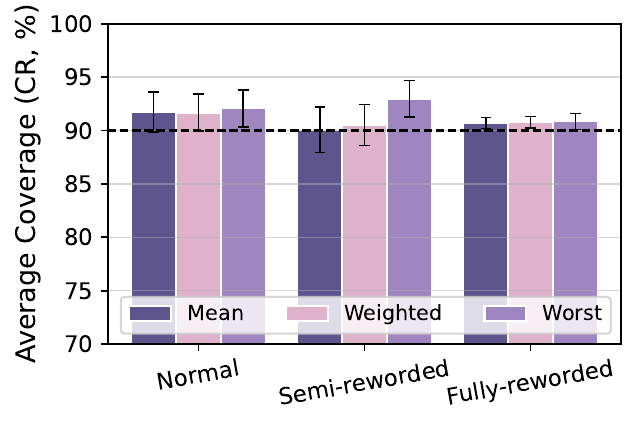}
        \label{fig:score-variant-cr}
    \end{subfigure}
    \hspace{-2mm}
    \begin{subfigure}[t]{0.34\textwidth}
        \centering
        \includegraphics[width=\linewidth]{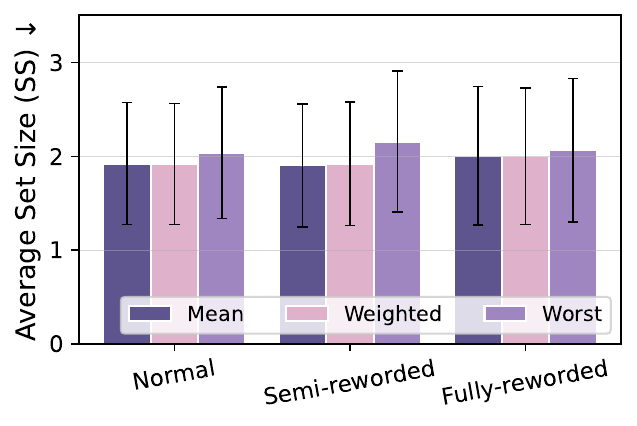}
        \label{fig:score-variant-ss}
    \end{subfigure}
    \vspace{-2mm}
        \caption{Comparison of different scores under QCCP.}
    \label{fig:score-variant-two-up}
    \vspace{-2mm}
    \end{figure}

\vspace{-10pt}
\paragraph{Comparison of Different Scores.}
Finally, we compare the three PA scores, namely Mean, Weighted, and Worst, introduced in Section~\ref{sec:scores}. We evaluate them under both split CP and QCCP, averaging across the seven datasets at each rewording setting (Normal, Semi-reworded, Fully-reworded). As shown in Figure~\ref{fig:score-variant-two-up}, all three variants track the $90\%$ coverage target across the three settings: under QCCP, Mean averages $91.7\% / 90.1\% / 90.7\%$ and Weighted $91.7\% / 90.6\% / 90.8\%$ on Normal / Semi-reworded / Fully-reworded prompts respectively, while Worst is mildly conservative at $92.1\% / 93.0\% / 90.9\%$. Mean and Weighted are nearly indistinguishable in practice (within $\approx 0.5$\,pp on coverage and $\approx 0.02$ on set size) and yield the most compact prediction sets (avg SS $\approx 1.9$–$2.0$); Worst trades a small coverage margin for $\approx 5$–$15\%$ larger sets (avg SS $\approx 2.0$–$2.2$). Split CP gives near-identical averages here, preserving the same relative ranking. Overall, the three variants span a smooth efficiency--conservatism tradeoff: Mean and Weighted are the tightest while still respecting the nominal target, and Worst supplies an additional coverage margin when extra robustness is desired.

% {Additional analyses of coverage levels, paraphrase budgets, and score variants are provided in Appendix~\ref{appendix:moved-main-analysis}.}

\section{Conclusion}

In this work, we introduced a framework for paraphrase-robust conformal prediction through paraphrase-aware nonconformity scores applied to both split CP and quasi-conditional CP. %
{Under the stated exchangeability conditions, our method inherits conformal coverage guarantees while producing compact prediction sets. The learned proxy drives most of the compactness gain, while paraphrase training and inference-time aggregation improve wording stability.} Experiments on five general QA and two medical QA benchmarks demonstrate robustness under rewording and strong generalization across model families. %
{These findings motivate further study of robust nonconformity score design across tasks and modalities, with preliminary claim-verification results and a vision--language extension discussed in Appendices~\ref{appendix:extension-to-other-task} and~\ref{appendix:extension-to-image-modality}.}

\begin{ack}
{This work was supported in part by grant R01-MH143267 and R01EB036016. The content is solely the responsibility of the authors and does not necessarily represent the official views of the NIH. We thank the anonymous reviewers for their constructive feedback. We also thank Jiancong Xiao and various people at the student paper swap event (Weiqiu You, Helen Jin, Seewon Choi, and others) for their helpful comments and discussions.}
\end{ack}

\bibliography{references}
\bibliographystyle{plainnat} 

\newpage
\appendix

\begingroup
\makeatletter
\let\oldpartname\partname
\let\oldthepart\thepart
\renewcommand{\partname}{}
\renewcommand{\thepart}{}
\part{Appendix}
\let\partname\oldpartname
\let\thepart\oldthepart
\makeatother
\endgroup

\parttoc

\section{{Detailed Related Work}}
\label{sec:related}

\paragraph{Heuristic and Calibration-Based Uncertainty for LLMs.}
LLM uncertainty is often estimated from token-level signals such as entropy, logits, or ranks, followed by calibration. Early work studies calibration in deep models and NLP tasks \citep{nado2022uncertainty,ulmer2022exploring,si2022reexamining}, and extends to prompt- and generation-level schemes for long-form outputs \citep{band2024linguistic,huang2024calibrating}. More recent approaches replace raw probabilities with representation-based surrogates, including semantic entropy probes \citep{kossen2024semantic}, relevance-aware confidence \citep{duan2024shifting}, perturbation-based measures \citep{gao2024spuq}, multi-agent diversity signals \citep{wang2024sample}, and semantic-density metrics \citep{qiu2025semantic}. 
Complementary directions include abstention mechanisms \citep{madhusudhan2024do}, multicalibration of confidence scores \citep{detommaso2024multicalibration}, post-hoc calibration from generated text \citep{ulmer2024calibrating}, and parameter-efficient Bayesian or ensemble-style methods for fine-tuned LLMs \citep{balabanov2025uncertainty,wang2025blob}. %
{While perturbation-based methods such as SPUQ explicitly aggregate across input variations, these methods address instance-level uncertainty. We study label-wise scores and prediction-set stability within CP, whose coverage guarantee additionally requires exchangeability.}

\paragraph{Conformal Prediction for LLMs.}
Conformal prediction (CP) provides distribution-free prediction sets with finite-sample coverage under exchangeability \citep{vovk2005algorithmic,shafer2008tutorial}. Recent work adapts CP to language modeling in several directions, including conformal language modeling \citep{quach2024conformal,quach2024conformal}, API-only inference without logit access \citep{su2024api}, and methods tailored to multiple-choice settings \citep{kumar2023conformal,yang2025conformalSetsMCQ}. CP has also been used for output alignment and certification \citep{gui2024conformal}, as well as benchmarking LLMs using uncertainty-aware metrics \citep{ye2024benchmarking}.
Extensions to generation tasks include \texttt{ConU}, which applies split CP to sets of sampled responses \citep{wang2024conu}, and \texttt{SConU}, which studies exchangeability violations and approximates conditional guarantees \citep{wang2025sconu}. Other approaches combine CP with re-asking strategies to improve accuracy and efficiency \citep{vishwakarma2025prunenpredict}, while selective answering with risk control has been explored via conformal abstention \citep{tayebati2025learning,wang2025coin}. 
In contrast to these works, our approach focuses on \emph{semantic invariance}: we develop paraphrase-aware scores that capture stability under meaning-preserving transformations, and integrate them with both split conformal prediction and quasi-conditional calibration on semantic embeddings, going beyond standard split CP.

\paragraph{Toward Conditional Guarantees and Paraphrase Robustness.}
Exact conditional coverage is unattainable in finite samples without distributional assumptions \citep{vovk2012conditional,foygel2021limits}, motivating relaxations such as \emph{quasi-conditional} guarantees via augmented quantile regression \citep{gibbs2023conformal}. Recent work extends CP to leverage feature-conditional structure in LLM settings \citep{cherian2024large}. 
In parallel, another line of work studies \emph{semantic uncertainty}, introducing embedding-based metrics and perturbation procedures \citep{gao2024spuq,kossen2024semantic,huang2024rankcalibration}. Related approaches examine prompt sensitivity and meaning-preserving perturbations \citep{qiu2025semantic,cox2025mapping}, but do not provide conformal guarantees. 
Our method bridges these directions by (i) quantifying predictive stability across paraphrases using semantic embeddings and a proxy model, and (ii) calibrating these scores within the conformal prediction framework, yielding prediction sets that are both semantically robust and statistically valid.

\section{Brief Introduction to the Conformal Framework}
\label{append:QCCP}

\paragraph{Theoretical guarantee of Split CP.}
We briefly review the standard coverage guarantee of split conformal prediction (CP), which serves as the foundation for our method. This result establishes finite-sample, distribution-free coverage under exchangeability.

\begin{lem}[\citep{vovk2005algorithmic}]
\label{lem:CP}
Assume ${(X_i,Y_i)}_{i=1}^{n+1}$ are i.i.d. and that $S(X,Y)$ has a continuous distribution.
The prediction set from split CP is
\[
\widehat{C}_{\mathrm{CP}}(X_{n+1}) := \{y: S(X_{n+1}, y) \le  Q_{\alpha}(\{S(X_{i}, Y_{i})\}_{i=1}^n )\}
\]
where $Q_{\alpha}(\{S(X_{i}, Y_{i})\}_{i=1}^n)$ is the empirical $(1-\alpha)$ quantile of $\{S(X_i,Y_i)\}_{i=1}^n \cup \{\infty\}$.  
Then, it achieves marginal coverage:  
\begin{equation}
\label{eq:theory}
1-\alpha \;\le\; {\mathbb P}(Y_{n+1} \in \widehat{C}_{\mathrm{CP}}(X_{n+1})) \;\le\; 1-\alpha + \frac{1}{n+1}.
\end{equation} 
\end{lem}

\paragraph{Theoretical Guarantee of Quasi-Conditional CP.}
Lemma~\ref{lem:QCCP} establishes the quasi-conditional coverage guarantee for QCCP. This guarantee is strictly stronger than the marginal coverage of split CP (which corresponds to the case where $\mathcal{F}$ contains only constant functions) and strictly weaker than full conditional coverage (which would require $\mathcal{F}$ to be the class of all measurable functions), thereby providing a principled middle ground.

In QCCP, the prediction set is constructed using an input-dependent threshold function $\widehat{g}_S$ learned from calibration scores via quantile regression (cf.\ \eqref{eq:augqr}).
We emphasize that $\widehat{g}_S$ depends on the choice of nonconformity score $S$, as it is trained on the induced calibration scores $\{S(X_i, Y_i)\}_{i=1}^n$. 
{For each candidate score value $s$, the fitted function $\widehat g_s:\mathcal X\to\mathbb R$ is evaluated at $X_{n+1}$; its dependence on the candidate label $y$ is through $s=S(X_{n+1},y)$.}
In practice, computing $\widehat{g}_S$ reduces to a convex optimization; see Section~4 of \citet{gibbs2023conformal} for details.

\begin{lem}[Theorem 2 in \citep{gibbs2023conformal}]
\label{lem:QCCP}
Let $\mathcal{F} = \{\phi(\cdot)^\top \beta: \beta \in \mathbb{R}^d\}$ denote the class of linear functions over the feature map $\phi: \mathcal{X} \to \mathbb{R}^d$.
Assume that $\{(X_i, Y_i)\}_{i=1}^{n+1}$ are i.i.d., and that $S(X,Y)\mid X$ has a continuous distribution. 
Define the prediction set as
\begin{equation*}
\tag{\ref{eq:predset}}
\widehat{C}_{\mathrm{QCCP}}(X_{n+1}) := \left\{ y : S(X_{n+1}, y) \le \widehat{g}_{S(X_{n+1}, y)}{(X_{n+1})} \right\},
\end{equation*}
where $\widehat{g}_{S(X_{n+1}, y)}$ is the fitted quantile regressor obtained by solving \eqref{eq:augqr} for a given value of $S(X_{n+1}, y)$ and $\ell_\alpha(\theta, S) = (S - \theta)\,(\mathbf{1}_{\{S \ge \theta\}} - \alpha)$ denoting the pinball loss.
Then, for any $f \in \mathcal{F}$,
\begin{align*}
\left|\mathbb{E}\!\left[ f(X_{n+1}) \left(\mathbf{1}\{Y_{n+1} \in \widehat{C}_{\mathrm{QCCP}}(X_{n+1})\} - (1-\alpha)\right) \right]\right| \le \frac{d}{n+1} \cdot \mathbb{E}\!\left[\max_{1 \le i \le n+1} |f(X_i)|\right].
\end{align*}
\end{lem}

\begin{algorithm}[t]
\caption{Paraphrase-Robust Quasi-Conditional CP}
\label{alg:full}
\begin{algorithmic}[1]
\STATE \textbf{Input:} Proper training set $\mathcal D_{\mathrm{tr}}=\{(X_i^{\mathrm{tr}},Y_i^{\mathrm{tr}})\}_{i=1}^{n_{\mathrm{tr}}}$; calibration set $\mathcal D_{\mathrm{cal}}=\{(X_i^{\mathrm{cal}},Y_i^{\mathrm{cal}})\}_{i=1}^{n_{\mathrm{cal}}}$; paraphrase generator; embedding map $E$; proxy model $P_\theta$; class function $\phi$; confidence level $\alpha$; PA score rule $S$.
\STATE \textbf{Output:} Prediction set $\widehat C(x)$ for a test question $x$.

\STATE \textbf{Train proxy:} Fit $P_\theta$ only on $\mathcal D_{\mathrm{tr}}$ using embeddings $\{(E(X_i^{\mathrm{tr}}),Y_i^{\mathrm{tr}})\}_{i=1}^{n_{\mathrm{tr}}}$.

\STATE \textbf{Compute calibration scores:} For each $(X_i^{\mathrm{cal}},Y_i^{\mathrm{cal}})\in\mathcal D_{\mathrm{cal}}$, generate paraphrases, compute the PA score
$S_i \gets S(X_i^{\mathrm{cal}},Y_i^{\mathrm{cal}})$ using the fixed proxy $P_\theta$, and record class $\phi(X_i^{\mathrm{cal}})$.

\STATE \textbf{Fit conformal threshold:} Use $\{(S_i,\phi(X_i^{\mathrm{cal}}))\}_{i=1}^{n_{\mathrm{cal}}}$ to fit the QCCP quantile function $\widehat g_S$; for split CP, compute the global conformal quantile.

\FOR{each test question $x$}
  \STATE $\widehat C(x) \gets \emptyset$
  \STATE Generate a paraphrase set $\mathcal{B}(x)$ and compute embeddings $\{E(x')\}_{x' \in \mathcal{B}(x)}$
  \STATE $z \gets \phi(x)$
  \FOR{each $y \in \mathcal{Y}$}
    \STATE Compute the PA score $S(x,y)$ using the fixed proxy $P_\theta$
    \STATE Query $\widehat g_z$ for QCCP, or the global quantile for split CP
    \IF{$S(x,y) \le \widehat g_z(x)$}
      \STATE $\widehat C(x) \gets \widehat C(x) \cup \{y\}$
    \ENDIF
  \ENDFOR
  \STATE \textbf{return} $\widehat C(x)$
\ENDFOR
\end{algorithmic}
\end{algorithm}

\section{Coverage Guarantees: Statements and Proofs}
\label{append:coverage}

\subsection{Proofs for Score Exchangeability}

\begin{prop}[Restatement of Proposition \ref{prop:new-exchangability}]
\label{prop:append-new-exchangability}
Suppose Asm.~\ref{asmp:main} holds, and let $\{(X_i, Y_i)\}_{i=1}^{n+1}$ be an exchangeable (or i.i.d.) dataset.
$(\mathcal{T}(X_{n+1}), Y_{n+1})$ might not be exchangeable (or i.i.d.)  with ${(X_i, Y_i)}_{i=1}^{n}$ due to the semantic shift induced by $\mathcal{T}$.
However, applying $\mathcal{T}$ to all inputs restores exchangeability (or i.i.d.), i.e., $\bigl(\mathcal{T} \circ \mathcal{T}(X_{n+1}), Y_{n+1}\bigr)$ is exchangeable (or i.i.d.) with 
$\bigl\{(\mathcal{T}(X_i), Y_i)\bigr\}_{i=1}^n$.
\end{prop}

\begin{proof}[Proof of Proposition \ref{prop:append-new-exchangability}]
Let $Z_i=(X_i,Y_i)$. Exchangeability means that $(Z_{\pi(1)},\ldots,Z_{\pi(n+1)})\overset{d}{=}(Z_1,\ldots,Z_{n+1})
$ for any permutation $\pi$.
{All draws from Markov kernels in this proof are taken to be conditionally independent across coordinates given the original data.}
Applying the Markov kernel $\mathcal{T}$ independently to each input coordinate preserves exchangeability. Hence, $\{(\mathcal{T}(X_i),Y_i)\}_{i=1}^{n+1}$ is exchangeable.

It remains to relate the last transformed sample to the semi-reworded test sample.
{Let $U_{n+1}\sim \mathcal{T}(\cdot\mid X_{n+1})$ denote the one-step transformed last coordinate, and let $U'_{n+1}\sim \mathcal{T}(\cdot\mid X_{n+1})$, $V_{n+1}\sim \mathcal{T}(\cdot\mid U'_{n+1})$ denote the two-step transformed last coordinate, using fresh randomness.}
By Asm.~\ref{asmp:main}, for any $x$,
\[
\int \mathcal{T}(A\mid x')\mathcal{T}(dx'\mid x)
=
\mathcal{T}(A\mid x)
\]
for every measurable set $A\subseteq\mathcal{X}$. Thus the two-step kernel $\mathcal{T}\circ\mathcal{T}$ equals $\mathcal{T}$ in distribution:
\[
(\mathcal{T}\circ\mathcal{T})(\cdot\mid x)=\mathcal{T}(\cdot\mid x).
\]
Consequently, conditional on $X_{n+1}$, $\mathcal{T}\circ\mathcal{T}(X_{n+1}) \overset{d}{=} \mathcal{T}(X_{n+1}).$
{Moreover, this replacement does not change the joint law because $V_{n+1}$ has the same conditional distribution as $U_{n+1}$ given $X_{n+1}$ and is generated independently of the remaining transformed coordinates conditional on the original data.}
Therefore, $\bigl\{(\mathcal{T}(X_i),Y_i)\bigr\}_{i=1}^n,
\bigl(\mathcal{T}\circ\mathcal{T}(X_{n+1}),Y_{n+1}\bigr)$ has the same joint law as $\bigl\{(\mathcal{T}(X_i),Y_i)\bigr\}_{i=1}^{n+1}$, which is exchangeable. Hence, $\bigl(\mathcal{T}\circ\mathcal{T}(X_{n+1}),Y_{n+1}\bigr)$ is exchangeable with
$\{(\mathcal{T}(X_i),Y_i)\}_{i=1}^n$.

The i.i.d. case follows as a special case of exchangeability.
\end{proof}

\paragraph{Notations.}

Fix any data point $(X,Y)$ from either the calibration set or the test set.  
Our PA score first generates $m$ paraphrased versions of $X$ using a language model.  
Let $\mathcal{T}(\cdot \mid X)$ denote the distribution of a single paraphrase of $X$, and let
\[
Z^{(1)},\ldots,Z^{(m)} \stackrel{\mathrm{i.i.d.}}{\sim} \mathcal{T}(\cdot \mid X).
\]
{
More precisely, for any measurable set $A\subseteq \mathcal{X}^m$, let $\mathcal{T}^{\otimes m}(A\mid x)$ denote the product law of $m$ conditionally independent draws from $\mathcal{T}(\cdot\mid x)$. All paraphrase draws are conditionally independent across data points and across paraphrase indices, given the original data.
}
After paraphrasing, the pair $(X,Y)$ is augmented into $(Z^{(1)},\ldots,Z^{(m)},Y)$.  
Since the calibrated predictor $P_\theta$ is trained on an unrelated dataset, it is treated as fixed throughout.
Define the augmented feature
\[
\widetilde{X} := (Z^{(1)},\ldots,Z^{(m)}),
\qquad 
\widetilde{\mathcal{T}}(X) := \widetilde{X}.
\]
Therefore, the PA score can be written as
\begin{align}
\label{eq:score-on-aug}
S_{\mathrm{PA}}(X,Y) 
&\;=\; S(Z^{(1)},\ldots,Z^{(m)},Y)\;=\;  S(\widetilde{\mathcal{T}}(X),Y) = S(\widetilde{X} ,Y) 
\end{align} 
for some deterministic computation rule $S$.
This notation highlights that the PA score simply operates on the augmented data $(\widetilde{X},Y)$.
{
The exact semi-reworded guarantee below applies to PA scores whose final value is a measurable function only of the internally generated paraphrase collection and the label. This includes the mean and worst-case PA scores. It also includes a weighted PA score if its weights are functions only of the generated paraphrase collection; if the weights additionally depend on the original input as an anchor, then an additional joint-invariance assumption is required.
}

\begin{lem}[Exchangeability for augmented data]
\label{lem:append-score}
Let $\{(X_i, Y_i)\}_{i=1}^{n+1}$ be an i.i.d. dataset and $\{(\widetilde{\mathcal{T}}(X_i), Y_i)\}_{i=1}^{n+1}$ be the augmented data defined above.
It follows that
\begin{itemize}[leftmargin=*]
\item \textbf{Normal setting:}  
$(\widetilde{\mathcal{T}}(X_{n+1}), Y_{n+1})$ is i.i.d. with  
$\{(\widetilde{\mathcal{T}}(X_i), Y_i)\}_{i=1}^n$.
\item \textbf{Fully reworded setting:}  
$(\widetilde{\mathcal{T}}(\mathcal{T}(X_{n+1})), Y_{n+1})$ is i.i.d. with $\{(\widetilde{\mathcal{T}}(\mathcal{T}(X_i)), Y_i)\}_{i=1}^n$.
\item \textbf{Semi-reworded setting:}  
Under Asm.~\ref{asmp:main},  
$(\widetilde{\mathcal{T}}(\mathcal{T}(X_{n+1})), Y_{n+1})$ is i.i.d. with  
$\{(\widetilde{\mathcal{T}}(X_i), Y_i)\}_{i=1}^n$.
\end{itemize}
\end{lem}

With Lemma~\ref{lem:append-score}, we can prove Lemma \ref{lem:score}.
\begin{proof}[Proof of Lemma \ref{lem:score}]
{
By Lemma~\ref{lem:append-score}, the corresponding augmented calibration variables and augmented test variable are i.i.d.\ in each of the three settings. Applying the same fixed measurable score rule $S$ in \eqref{eq:score-on-aug} to each augmented input--label pair preserves exchangeability. Therefore, the induced calibration and test scores are exchangeable.
}
\end{proof}

Now, we provide the proof for Lemma \ref{lem:append-score} below.

\begin{proof}[Proof of Lemma \ref{lem:append-score}]
By construction, each entry of the random vector $\widetilde{\mathcal{T}}(X)$ is generated by i.i.d.\ sampling from $\mathcal{T}(\cdot \mid X)$, and the sampling is independent across data points.  
Therefore, the transformed dataset $\{(\widetilde{\mathcal{T}}(X_i),Y_i)\}_{i=1}^{n+1}$ is still i.i.d., which proves the first point.

For the second point, by a similar argument, $\{(\widetilde{\mathcal{T}}(\mathcal{T}(X_i)),Y_i)\}_{i=1}^{n+1}$ is i.i.d.\ because $\{(\mathcal{T}(X_i), Y_i)\}_{i=1}^{n+1}$ is i.i.d.

{
For the third point, fix $x\in\mathcal{X}$ and a measurable set $A\subseteq\mathcal{X}^m$. Conditional on $X=x$, the augmented calibration input has law
\[
\mathbb{P}\left(\widetilde{\mathcal{T}}(X)\in A\mid X=x\right)
=
\mathcal{T}^{\otimes m}(A\mid x).
\]
For a semi-reworded test input, first draw $X'\sim\mathcal{T}(\cdot\mid x)$ and then draw $m$ internal paraphrases conditionally independently from $\mathcal{T}(\cdot\mid X')$. Therefore,
\begin{align*}
\mathbb{P}\left(\widetilde{\mathcal{T}}(\mathcal{T}(X))\in A\mid X=x\right)
&=
\int \mathcal{T}^{\otimes m}(A\mid x')\mathcal{T}(dx'\mid x) \\
&=
\int \mathcal{T}^{\otimes m}(A\mid x)\mathcal{T}(dx'\mid x) \\
&=
\mathcal{T}^{\otimes m}(A\mid x),
\end{align*}
where the second equality uses Asm.~\ref{asmp:main}. Hence,
\[
\mathcal{L}\left(\widetilde{\mathcal{T}}(\mathcal{T}(X))\mid X=x\right)
=
\mathcal{L}\left(\widetilde{\mathcal{T}}(X)\mid X=x\right).
\]
Averaging over $(X,Y)$ gives
\[
\mathcal{L}\left(\widetilde{\mathcal{T}}(\mathcal{T}(X)),Y\right)
=
\mathcal{L}\left(\widetilde{\mathcal{T}}(X),Y\right).
\]
Since all paraphrase draws are conditionally independent across data points, $(\widetilde{\mathcal{T}}(\mathcal{T}(X_{n+1})),Y_{n+1})$ is i.i.d.\ with $\{(\widetilde{\mathcal{T}}(X_i),Y_i)\}_{i=1}^n$.
}
\end{proof}

\subsection{Formal Coverage Results}

Now, we are ready to state and prove the formal coverage results.
\begin{thm}[Formal coverage guarantee]
\label{thm:coverage}
Assume ${(X_i,Y_i)}_{i=1}^{n+1}$ are i.i.d. and that, {in the corresponding data setting, both $S_{\mathrm{PA}}(U_{n+1},Y_{n+1})\mid U_{n+1}$ and $S_{\mathrm{PA}}(U_{n+1},Y_{n+1})$ have continuous distributions.}
Let $\mathcal{T}$ be an operator satisfying Assumption \ref{asmp:main}.
We consider the following three data settings:
\begin{itemize}[leftmargin=*]
\item \textbf{Normal setting:}  
The calibration dataset is $\{(X_i, Y_i)\}_{i=1}^n$ and the test data is $(X_{n+1}, Y_{n+1})$.
\item \textbf{Fully reworded setting:}  
The calibration dataset is $\{(\mathcal{T}(X_i), Y_i)\}_{i=1}^n$ and the test data is $(\mathcal{T}(X_{n+1}), Y_{n+1})$.
\item \textbf{Semi-reworded setting:}  
The calibration dataset is $\{(X_i, Y_i)\}_{i=1}^n$ and the test data is $(\mathcal{T}(X_{n+1}), Y_{n+1})$.
\end{itemize}
{
For notational clarity, let $U_i$ denote the effective input on which the PA score is evaluated:
\[
U_i =
\begin{cases}
X_i, & \text{normal setting},\\
\mathcal{T}(X_i), & \text{fully reworded setting},\\
X_i, & \text{semi-reworded setting and } i\le n,\\
\mathcal{T}(X_{n+1}), & \text{semi-reworded setting and } i=n+1.
\end{cases}
\]
Thus the calibration scores are computed as $S_{\mathrm{PA}}(U_i,Y_i)$ for $i=1,\ldots,n$, and the test score is computed as $S_{\mathrm{PA}}(U_{n+1},Y_{n+1})$.
}
For the above three data settings, the following coverage guarantees hold.
\begin{itemize}[leftmargin=*]
\item \textbf{PA score for CP.}
Let $\widehat{C}_{\mathrm{CP}}(\cdot)$ denote the prediction set produced by CP using the PA score and defined in \eqref{eq:CI-CP}. It follows that
\begin{equation*}
1-\alpha \;\le\; {\mathbb P}({Y_{n+1} \in \widehat{C}_{\mathrm{CP}}(U_{n+1})}) \;\le\; 1-\alpha + \frac{1}{n+1}.
\end{equation*} 
\item \textbf{PA score for QCCP.}
Let $\mathcal{F} = \{\phi(\cdot)^\top \beta: \beta \in \mathbb{R}^{d}\}$ denote the class of linear functions over the feature vector ${\phi}$.
Let $\widehat{C}_{\mathrm{QCCP}}(\cdot)$ be the prediction set produced by QCCP using the PA score and defined in \eqref{eq:augqr}.
Then, for any $f \in \mathcal{F}$,
\begin{align*}
\left|  \mathbb{E}\left[ {f(W_{n+1})} \left(\mathbf{1}\{{Y_{n+1} \in \widehat{C}_{\mathrm{QCCP}}(U_{n+1})}\} - (1-\alpha)\right) \right]  \right|
\le \frac{d}{n+1} \cdot \mathbb{E}\left[ \max_{1 
\le i \le (n+1)} |{f(W_i)}| \right].
\end{align*}
\end{itemize}

\end{thm}

\begin{proof}[Proof of Theorem \ref{thm:coverage}]
By Lemma~\ref{lem:append-score}, the PA scores $\{S_{\mathrm{PA}}(U_i,Y_i)\}_{i=1}^{n+1}$ are exchangeable in each of the three settings. The split-CP statement therefore follows directly from Lemma~\ref{lem:CP}. 

For QCCP, let $W_i$ denote the feature input (i.e., the input of $\phi$) used by QCCP in the corresponding setting, and assume that the pairs.
In the normal and fully reworded settings, we have $W_i=U_i$, and the same feature construction is applied to all points. 
In the semi-reworded setting, it holds if the QCCP feature is essentially $\mathcal{T}(U_i)$.
In both cases, we have $\{(W_i,S_{\mathrm{PA}}(U_i,Y_i))\}_{i=1}^{n+1}$ is exchangeable (from Proposition \ref{prop:append-new-exchangability}), giving exactly the joint exchangeability condition required by Lemma~\ref{lem:QCCP}. Applying Lemma~\ref{lem:QCCP} yields the stated bound.
\end{proof}

\subsection{Approximate Distributional Closedness}
\label{appen:approximate-main}

Assumption~\ref{asmp:main} requires the paraphrase kernel $\mathcal{T}$ to be distributionally closed, i.e., $\mathcal{T}(\cdot \mid x') = \mathcal{T}(\cdot \mid x)$ for $\mathcal{T}(\cdot \mid x)$-almost every $x'$.
In practice, however, our paraphrase pipeline is input-conditioned and involves stochastic generation with retries and fallbacks, which may violate this condition.
To quantify this deviation, we introduce a relaxed version of distributional closedness.

\begin{asmp}[Approximate distributional closedness]
\label{ass:approx_closed}
Let $\mathcal{T}$ be a Markov kernel on $\mathcal{X}$.
We say $\mathcal{T}$ is $\varepsilon$-approximately distributionally closed if, for any $x \in \mathcal{X}$,
\[
\int d_{\mathrm{TV}}\!\left(
\mathcal{T}(\cdot \mid x'),\mathcal{T}(\cdot \mid x)
\right)\mathcal{T}(dx'\mid x)
\le \varepsilon .
\]
\end{asmp}

This assumption captures the extent to which paraphrase distributions vary within the support of $\mathcal{T}(\cdot \mid x)$.

\begin{prop}[Coverage under approximate closedness]
\label{prop:approx_cov}
Let $\widehat C(x)$ be the prediction set constructed via conformal prediction (i.e., via split CP or QCCP) using the PA score, where the score internally uses $m$ paraphrases generated from $x$.
Let the ideal coverage level be
\[
p_{\mathrm{ideal}}
:=
\mathbb{P}\!\left( Y \in \widehat C(X) \right).
\]
where the test sample is $(X, Y)$.
Under Assumption~\ref{ass:approx_closed}, for this test sample $(\mathcal{T}(X_{n+1}), Y_{n+1})$, we have
\[
\left| \mathbb{P}\!\left( Y \in \widehat C(\mathcal T(X)) \right) - p_{\mathrm{ideal}} \right| \le m\varepsilon .
\]
In particular, if $p_{\mathrm{ideal}}=1-\alpha$, then
\[
\left|
\mathbb{P}\!\left( Y \in \widehat C(\mathcal T(X)) \right)
-(1-\alpha)
\right|
\le m\varepsilon .
\]
\end{prop}

Proposition~\ref{prop:approx_cov} shows that violations of distributional closedness translate into coverage deviations that scale linearly with the number of paraphrases. In particular, if the paraphrase pipeline deviates from ideal closedness by at most $\varepsilon$ in total variation, then the resulting coverage gap is at most $m\varepsilon$ (up to finite-sample effects).

\begin{proof}[Proof of Proposition \ref{prop:approx_cov}]
Let $A(x,y)=\mathbf{1}\{y\in \widehat C(x)\}$. By total variation, for any distributions $P,Q$,
$|\mathbb{P}_P(A)-\mathbb{P}_Q(A)| \le d_{\mathrm{TV}}(P,Q)$.
It suffices to bound the distance between 
$\mathrm{Law}(\widetilde{\mathcal{T}}(\mathcal{T}(X)),Y)$ and $\mathrm{Law}(\widetilde{\mathcal{T}}(X),Y)$.
Condition on $X=x$. Then $\widetilde{\mathcal{T}}(X)$ consists of $m$ i.i.d.\ samples from $\mathcal{T}(\cdot \mid x)$, while $\widetilde{\mathcal{T}}(\mathcal{T}(X))$ consists of $m$ i.i.d.\ samples from $\mathcal{T}(\cdot \mid x')$, where $x' \sim \mathcal{T}(\cdot \mid x)$.

By the tensorization property of total variation,
\[
d_{\mathrm{TV}}\!\left(
\mathcal{T}(\cdot \mid x')^{\otimes m},
\mathcal{T}(\cdot \mid x)^{\otimes m}
\right)
\le m \, d_{\mathrm{TV}}\!\left(
\mathcal{T}(\cdot \mid x'),
\mathcal{T}(\cdot \mid x)
\right).
\]
Taking expectation over $x' \sim \mathcal{T}(\cdot \mid x)$ and using Assumption~\ref{ass:approx_closed}, we obtain
\[
d_{\mathrm{TV}}\!\left(
\mathrm{Law}(\widetilde{\mathcal{T}}(\mathcal{T}(X)) \mid X=x),
\mathrm{Law}(\widetilde{\mathcal{T}}(X) \mid X=x)
\right)
\le m\varepsilon.
\]
Averaging over $X$ yields
\[
d_{\mathrm{TV}}\!\left(
\mathrm{Law}(\widetilde{\mathcal{T}}(\mathcal{T}(X)),Y),
\mathrm{Law}(\widetilde{\mathcal{T}}(X),Y)
\right)
\le m\varepsilon.
\]

Therefore,
\[
\left|
\mathbb{P}\{Y \in \widehat C(\widetilde{\mathcal{T}}(\mathcal{T}(X)))\}
-
\mathbb{P}\{Y \in \widehat C(\widetilde{\mathcal{T}}(X))\}
\right|
\le m\varepsilon.
\]
Combining this with the ideal coverage guarantee under exact closedness yields the result.
\end{proof}

\section{Experimental Details}
\label{append:experiment-detials}

\subsection{Dataset Details}
\label{appendix:dataset-details}
We used five general and two medical MCQA datasets to show that our paraphrase-aware score can be broadly used and can adapt to high-stakes scenarios where calibration is important. The five general datasets test a wide range of LLM capabilities and have been used in previous CP benchmarking papers \citep{ye2024benchmarking, vishwakarma2025prunenpredict}. We use the five general datasets processed by \citet{ye2024benchmarking}. Each of the datasets contains 10,000 questions. 

\begin{enumerate}[leftmargin=*]
    \item \textbf{QA, MMLU} \citep{hendrycks2021measuring}: MMLU is a dataset designed to test the general knowledge and \emph{question-answering} abilities of LLMs. Question topics range from sociology and high school geography to electrical engineering and abstract algebra. 
    \item \textbf{RC, CosmosQA} \citep{huang-etal-2019-cosmos}: CosmosQA focuses on gauging an LLM's \emph{reading comprehension} abilities. The LLM is given a short paragraph and is then asked to answer a follow-up question based on commonsense reasoning.
    \item \textbf{CI, HellaSwag} \citep{zellers2019hellaswag}: HellaSwag evaluates if LLMs can use \emph{commonsense inference} to construct a realistic and meaningful continuation of a given scenario. HellaSwag is deliberately designed so that LLMs struggle with questions that humans could normally answer with high confidence.
    \item \textbf{DRS, HaluEval} \citep{li-etal-2023-halueval}: HaluEval consists of hallucinated LLM responses to user queries. A subset of HaluEval contains queries that relate to \emph{dialogue response selection}: the LLM must be able to choose a logical response for a conversation. We refer to this part of HaluEval as HaluDial.
    \item \textbf{DS, HaluEval} \citep{li-etal-2023-halueval}: The HaluEval dataset also has hallucinated \emph{document summaries}. In an MCQA setting, the LLM must determine which summary in the answer choices is most relevant to a provided document. We refer to this part of HaluEval as HaluSum.
    \item \textbf{MedMCQA} \citep{pal2022medmcqa}: MedMCQA is a large-scale multiple-choice medical QA dataset comprising over 194,000 entrance-exam–style questions spanning 2400 healthcare topics and 21 medical subjects. We select a subset of 10,000 single-answer questions from MedMCQA (i.e. where exactly one option is marked correct) for our experiments.
    \item \textbf{MedQA} \citep{jin2021disease}: MedQA is a medical exam QA dataset derived from professional medical board exams (e.g. USMLE), providing each question paired with candidate answer options and corresponding references. In our work, we use only the US-part of MedQA, and further restrict to 10,000 multiple-choice items that have exactly one correct answer.
\end{enumerate}

Following the LLM-Uncertainty-Benchmark protocol \citep{ye2024benchmarking}, every question across all seven datasets is presented with a uniform label space of \emph{six} options $|\mathcal{Y}|=6$: the original answer choices (typically four for MMLU, CosmosQA, HellaSwag, and MedMCQA, and four to five for MedQA) are augmented with two additional options, ``\texttt{I don't know}'' and ``\texttt{None of the above}'', and HaluDial/HaluSum are likewise cast into this six-way format. This is why all reported set sizes are bounded by 6 rather than by the number of native answer choices.

\subsection{Paraphrase Generation for Full/Semi-Paraphrase Settings}
\label{app:adv-paraphrase}

We only use this procedure to generate paraphrases for the input question used in the full- and semi- paraphrase evaluation settings. This set of reworded questions is not directly used for calculating the PA score, but is used as another set of input questions.

\textbf{Prompt Pool.}
We maintain 7 short templates (e.g., ``Rephrase this question using varied vocabulary and phrasing: \{question\}''). One template is sampled per input to induce lexical or syntactic variety without changing semantics.

\textbf{Batched Generation.}
Given a batch of $n$ questions, we form $n$ prompts and generate once with:
\texttt{temperature}=0.7,\;
\texttt{top\_p}=0.9,\;
\texttt{do\_sample}=True,\;
\texttt{eos/pad\_token\_id} aligned.
We decode with \texttt{skip\_special\_tokens=True}.

\textbf{Cleaning \& Validation.}
We remove the prompt prefix and extract the first question sentence via regex matching the earliest ``?''. We then strip boilerplate (``Rephrase:'', ``Paraphrase:'', ``Here's a …''), quotes/bullets/code blocks, and normalize whitespace. The candidate paraphrase must satisfy:
(i) non-empty,
(ii) case-insensitive $\neq q$,
(iii) ends with ``?''.
Duplicates within the batch are dropped.

\textbf{Retry \& Fallback.}
Failures are retried with a stricter template: single line, no preface, $\le$20 words, must end with ``?''. Remaining failures are rewritten by a deterministic rule set that preserves meaning (e.g., ``Which of the following''$\!\to$``Which option'', add trailing ``?'', etc.).

\begin{algorithm}[H]
\caption{Paraphrase Pipeline (per question $q$)}
\begin{algorithmic}[1]
\STATE Sample a paraphrase template; compose prompt $p(q)$.
\STATE Generate with ($T{=}0.7$, top-p=0.9, sampling).
\STATE Decode; strip prompt prefix; take first sentence ending with ``?''.
\STATE Clean and validate; if valid, return $\tilde q$.
\STATE Else: retry with strict one-line template; clean and validate.
\STATE Else: apply rule-based fallback; return $\tilde q$.
\end{algorithmic}
\end{algorithm}

{\subsection{Baseline Score Adaptations}
\label{appendix:baseline_choice}
SPUQ-adapt.\ and TRON-adapt.\ are per-option score controls rather than the original SPUQ and TRON algorithms. For raw option probabilities $p_y(x)$, let
$\bar p_y(x)=|\mathcal B(x)|^{-1}\sum_{\tilde x\in\mathcal B(x)}p_y(\tilde x)$.
The probability-based controls are $1-\bar p_y(x)$ and $1-p_y(x)$, respectively, and their negative-log counterparts are $-\log(\bar p_y(x)+10^{-12})$ and $-\log(p_y(x)+10^{-12})$.
The probability-based TRON control is exactly LAC. When computed using the same proxy probabilities and paraphrase views as Full PA, the mean-probability control is algebraically identical to $S_{\mathrm{mean}}$; this identity does not extend to the weighted, worst-case, or negative-log scores.}

\subsection{Experimental Implementation Details}
\label{appendix:detailed-implement}

\paragraph{Paraphrased Generation for Calculating PA Score.} Paraphrases were generated by prompting \texttt{Qwen3-8B} \citep{yang2025qwen3} with the following query: ``Rephrase the following question in your own words (preserving its meaning): Original question: \{question\} \textbackslash n Rephrased question:". Additional details, such as any context or the answer choices, were not included with the original question. A total of 6 paraphrases were generated per question. Paraphrases that were equivalent to the original question or were previously generated were not included in the final set. The temperature of \texttt{Qwen3-8B} \citep{yang2025qwen3}  was set to 1.1 to encourage diverse responses.

\paragraph{LLM Embeddings.} We next obtained the LLM embeddings for each question and paraphrase. These LLM embeddings are used to train the proxy model and serve as the LLM's representation of the query. We treated the paraphrases as new samples in the dataset and assigned them the same answer choices and correct label as their parent question. For each sample, the input prompt included the question and the list of answer choices without the correct label. The LLM embedding was extracted from the final hidden layer; this layer simultaneously encodes the LLM's understanding of the question and its predicted answer. The LLM logits were also retrieved by finding the raw score corresponding to each answer represented as a token (e.g. `A', `B', etc.). We then applied softmax to the logits to obtain a probability distribution over the answer choices. For each dataset except the DS dataset, the LLM was provided with 2 example questions and their correct labels. Due to the long context needed for the questions in the DS dataset, only 1 example was provided to the LLM in this case. We followed this procedure to get the embeddings and logits of the instruction-tuned versions of \texttt{Qwen3-8B} \citep{yang2025qwen3}, \texttt{Llama-3.1-8B} \citep{dubey2024llama}, and \texttt{Phi-3-small-8k} \citep{abdin2024phi3technicalreporthighly}.

\paragraph{Proxy Model.} The proxy model calibrates the LLM's  probability distribution of the answer choices across the paraphrases. We used a 2-layer MLP with ReLU activation as our proxy model. The input to the proxy model is an LLM's embedding for a question or paraphrase, and the output is a vector with dimension $|\mathcal{Y}|$. The loss function used to train the proxy model is $L_{total} = L_{CE} + \lambda_{ece}\times L_{ECE}$, where $L_{CE}$ is the standard cross-entropy loss for multiclass classification and $L_{ECE}$ is the soft-binned expected calibration error loss \citep{karandikar2021soft}. $\lambda_{ECE}$ is a hyperparameter that can be tuned to increase or decrease the relative importance of calibration in different contexts. Out of the 4,000 questions in the training set, 600 questions were used as the validation set. Since each question has 6 paraphrases, the effective training and validation set sizes are 23,800 and 4,200, respectively. We conducted grid search on the following hyperparameters: learning rate, weight decay, $\lambda_{ECE}$, hidden dimension, and the batch size. The hyperparameters used for each dataset and each model are provided in Appendix \ref{appendix:hyperparams}.

\paragraph{Score Calculation.} For our baselines, we only use logits from the \emph{original} question. The LAC score \citep{sadinle2019least} is defined as \(s_{\text{LAC}}(x,y)=1-f(x)_y\), where \(f(x)_y\) is the softmax probability of label \(y\) for a question $x$. The APS score \citep{romano2020classification} is \(s_{\text{APS}}(x,y)=\sum_{y'\in\mathcal{Y}:~f(x)_{y'}\ge f(x)_y} f(x)_{y'}\), i.e., the cumulative probability of labels ranked at least as high as \(y\). To find the score of a question $x$ using the PA method, the LLM embeddings of its paraphrases (represented by the set $\mathcal{B}(x)$) are inputted into the trained proxy model. Then, the score for answer choice $y$ of question $x$ is given by $S_{\text{mean}}(x, y)$, $S_{\text{weighted}}(x, y)$, or $S_{\text{worst}}(x, y)$ (see Section \ref{sec:scores}). 
Note that only the scores of the 6 paraphrases are used in the PA formulas. The LLM's embedding of the original question does not contribute to that question's final PA score. This is in contrast to LAC and APS, which only rely on the logits of the original question and do not factor in the paraphrases.

\paragraph{Split Conformal Prediction.} Split CP \citep{vovk2005algorithmic} uses a separate calibration dataset to calculate a global score threshold that determines the prediction sets for the test set. For each calibration example, we use the
nonconformity score from the LLM’s softmax probabilities or proxy model's probability distribution. These nonconformity scores are collected and the $(1-\alpha)$ quantile is estimated with a finite-sample correction $q_{\alpha} = \mathrm{Quantile}(\{s_i\}, \lceil (n+1)(1-\alpha)\rceil/n)$. At test time, each example’s prediction set is formed by including all labels whose score is below $q_{\alpha}$, with a fallback to the most probable label if the set is empty. This procedure ensures that the empirical coverage approaches the nominal $1-\alpha$ guarantee.

\paragraph{Quasi-Conditional Conformal Prediction.} In contrast to split CP, QCCP \citep{gibbs2023conformal} calculates class-specific thresholds for pre-defined classes. The feature function $\phi$ is defined to assign a class to a question based on its features. In the case that $\phi$ is a constant function, then QCCP is equivalent to split CP; we use this $\phi$ function ($\phi$ = intercept) to obtain all split CP results. In our QCCP analysis, we focus on a $\phi$ function that takes in the embedding of each question generated by \texttt{all-MiniLM-L6-V2} \citep{reimers-2019-sentence-bert}. $K$-means clustering is used to separate the questions into clusters, with each cluster acting as a class, and a different score threshold is selected for each cluster. For a given test question, the $\phi$ function is either a one-hot encoding of its assigned cluster or a vector of its embedding's distances to the three closest cluster centroids (dist3) (which can be used to calculate the threshold). The coverage rates reported using QCCP are calculated marginally with the exception of Figure~\ref{fig:an2-classwise-condconf} and Figure~\ref{fig-appendix:an2-classwise-condconf-others}, which displays class-wise coverages. We used the \texttt{conditionalconformal} package \citep{gibbs2023conformal} to implement QCCP.

\subsection{Hyperparameters for Reproducibility}
\label{appendix:hyperparams}

Tables~\ref{tab:general-hparams} and~\ref{tab:med-hparams} list the framework hyperparameters for the five general MCQA datasets and the two medical MCQA datasets, respectively. Tables~\ref{proxy-model-hyperparam-general} and~\ref{proxy-model-hyperparam-medical} report the corresponding hyperparameters for the proxy model on the general and medical datasets.

\begin{table*}[ht!]
\caption{Hyperparameters for QA, RC, CI, DRS, DS. We train an MLP proxy with ECE regularization and compute paraphrase-aware scores using 6 paraphrases per question. Evaluation includes QCP with SBERT-clustered $\Phi$ and plain CP with intercept $\Phi$, using the same manifest-based splits and $\alpha{=}0.1$.}
\centering
\begin{tabular}{llcc}
\toprule
\multicolumn{2}{l}{\textbf{Setting / Hyperparameter}} & \multicolumn{2}{c}{\textbf{QA, RC, CI, DRS, DS}} \\
\midrule
\multicolumn{4}{l}{\textbf{General}} \\
\cmidrule(lr){1-4}
& Base LLM (for reps/logits) & \multicolumn{2}{c}{Qwen3-8B} \\
& Random seed & \multicolumn{2}{c}{42} \\
& Samples per question ($n$) & \multicolumn{2}{c}{7 (1 base + 6 paraphrases)} \\
& Options per item ($|\mathcal{Y}|$) & \multicolumn{2}{c}{6 (A–F)} \\
\midrule
\multicolumn{4}{l}{\textbf{Proxy model (training)}} \\
\cmidrule(lr){1-4}
& Input dim & \multicolumn{2}{c}{4096} \\
& Architecture & \multicolumn{2}{c}{MLP: 4096 $\rightarrow$ $h$ (ReLU) $\rightarrow$ 6} \\
& Optimizer & \multicolumn{2}{c}{Adam} \\
& Max epochs / patience & \multicolumn{2}{c}{200 / 20} \\
& Loss & \multicolumn{2}{c}{CE + soft-binned ECE (15 bins, temp 0.1))} \\
\midrule
\multicolumn{4}{l}{\textbf{Paraphrase-aware score computation}} \\
\cmidrule(lr){1-4}
& Paraphrases per question & \multicolumn{2}{c}{6} \\
& Metric used & \multicolumn{2}{c}{\texttt{S\_mean}} \\
\midrule
\multicolumn{4}{l}{\textbf{Conformal prediction / QCP evaluation}} \\
\cmidrule(lr){1-4}
& Prompting / ICL & \multicolumn{2}{c}{\texttt{base} / \texttt{icl1}} \\
& Error level ($\alpha$) & \multicolumn{2}{c}{0.1} \\
& Split config & \multicolumn{2}{c}{manifest: \texttt{tr0.4 / cf0.3 / tf0.3}} \\
& QCP $\Phi$ mode & \multicolumn{2}{c}{\texttt{cluster\_sbert}} \\
& SBERT model & \multicolumn{2}{c}{MiniLM-L6-v2} \\
& Cluster selection & \multicolumn{2}{c}{auto} \\
& $\Phi$ representation & \multicolumn{2}{c}{dist3} \\
& Embedding norm / mini-batch k-means & \multicolumn{2}{c}{on / on} \\
& Also reported (plain CP) & \multicolumn{2}{c}{\texttt{intercept} $\Phi$ (same splits, $\alpha$)} \\
\bottomrule
\end{tabular}
\label{tab:general-hparams}
\vspace{-10pt}
\end{table*}

\begin{table*}[h!]
\caption{Hyperparameters for MedMCQA-10k and MedQA-10k. We train an MLP proxy with ECE regularization and compute paraphrase-aware scores using 6 paraphrases per question. Evaluation includes QCP with SBERT-clustered $\Phi$ and plain CP with intercept $\Phi$, using the same manifest-based splits and $\alpha{=}0.1$.}
\centering
\setlength{\tabcolsep}{6pt}
\begin{tabular}{llcc}
\toprule
\multicolumn{2}{l}{\textbf{Setting / Hyperparameter}} & \textbf{MedMCQA-10k} & \textbf{MedQA-10k} \\
\midrule
\multicolumn{4}{l}{\textbf{General}} \\
\cmidrule(lr){1-4}
& Base LLM (for reps/logits) & \multicolumn{2}{c}{Qwen3-8B} \\
& Random seed & \multicolumn{2}{c}{42} \\
& Samples per question ($n$) & \multicolumn{2}{c}{7 (1 base + 6 paraphrases)} \\
& Options per item ($K$) & \multicolumn{2}{c}{6 (A–F)} \\
\midrule
\multicolumn{4}{l}{\textbf{Proxy model (training)}} \\
\cmidrule(lr){1-4}
& Input dim & \multicolumn{2}{c}{4096} \\
& Architecture & \multicolumn{2}{c}{MLP: 4096 $\rightarrow$ $h$ (ReLU) $\rightarrow$ 6} \\
& Optimizer & \multicolumn{2}{c}{Adam} \\
& Max epochs / patience & \multicolumn{2}{c}{50 / 5} \\
& Loss & \multicolumn{2}{c}{CE + soft-binned ECE (15 bins, temp 0.1)} \\
\midrule
\multicolumn{4}{l}{\textbf{Paraphrase-aware score computation}} \\
\cmidrule(lr){1-4}
& Paraphrases per question & \multicolumn{2}{c}{6} \\
& Metric used & \multicolumn{2}{c}{\texttt{S\_mean}} \\
\midrule
\multicolumn{4}{l}{\textbf{Conformal prediction / QCP evaluation}} \\
\cmidrule(lr){1-4}
& Prompting / ICL & \multicolumn{2}{c}{\texttt{task} / \texttt{icl1}} \\
& Error level ($\alpha$) & \multicolumn{2}{c}{0.1} \\
& Split config & \multicolumn{2}{c}{manifest: \texttt{tr0.4 / cf0.3 / tf0.3}} \\
& QCP $\Phi$ mode & \multicolumn{2}{c}{\texttt{cluster\_sbert}} \\
& SBERT model & \multicolumn{2}{c}{MiniLM-L6-v2} \\
& Cluster selection / $K$ & \multicolumn{2}{c}{fixed / 20} \\
& $\Phi$ representation & \multicolumn{2}{c}{one-hot (cluster ID)} \\
& Embedding norm / mini-batch k-means & \multicolumn{2}{c}{off / off} \\
& Also reported (plain CP) & \multicolumn{2}{c}{\texttt{intercept} $\Phi$ (same splits, $\alpha$)} \\
\bottomrule
\end{tabular}
\label{tab:med-hparams}
\vspace{-10pt}
\end{table*}

\begin{table*}[!ht]
\caption{Proxy model specific hyperparameters for the 5 general datasets}
\label{proxy-model-hyperparam-general}
\begin{center}
\begin{tabular}{cccccc}
\toprule 
\multicolumn{1}{c}{\bf Dataset}  &\multicolumn{1}{c}{\bf Hidden dimension $h$} &\multicolumn{1}{c}{\bf Batch size} &\multicolumn{1}{c}{\bf Learning rate} &\multicolumn{1}{c}{\bf Weight decay} &\multicolumn{1}{c}{\bf $\lambda_{\text{ECE}}$}
\\ \midrule 
MMLU & 256 & 64 & 1e-3 & 0.0 & 0.5 \\
CosmosQA & 512 & 128 & 1e-3 & 0.0001 & 0.5  \\
HellaSwag & 256 & 128 & 1e-3 & 0.0001 & 0.5 \\
HaluDial & 256 & 64 & 1e-3 & 0.0 & 0.5 \\
HaluSum & 256 & 128 & 1e-4 & 0.0 & 0.5 \\
\bottomrule \\
\end{tabular}
\end{center}
\vspace{-10pt}
\end{table*}

\begin{table*}[!ht]
\caption{Proxy model specific hyperparameters for MedMCQA and MedQA.}
\label{proxy-model-hyperparam-medical}
\begin{center}
\vspace{-5pt}
\begin{tabular}{cccccc}
\toprule 
\multicolumn{1}{c}{\bf Dataset}  &\multicolumn{1}{c}{\bf Hidden dimension $h$} &\multicolumn{1}{c}{\bf Batch size} &\multicolumn{1}{c}{\bf Learning rate} &\multicolumn{1}{c}{\bf Weight decay} &\multicolumn{1}{c}{\bf $\lambda_{\text{ECE}}$}
\\ \midrule 
MedMCQA & 256 & 64 & 1e-3 & 0.0 & 0.5 \\
MedQA & 256 & 64 & 1e-3 & 0.0 & 0.5 \\
\bottomrule \\
\end{tabular}
\end{center}
\vspace{-10pt}
\end{table*}

% \clearpage
\section{Additional Experimental Results}
\label{append:additional-experiments}

\subsection{Better Performance on Medical MCQA Datasets}
\label{sec:exp-medical}

Due to space constraints, we report results on medical MCQA datasets here.
We evaluate whether paraphrase-aware conformal prediction generalizes to domain-specific settings using two medical benchmarks, {MedMCQA} \citep{pal2022medmcqa} and {MedQA} \citep{jin2021disease}, with \texttt{Qwen3-8B} under normal, semi-reworded, and fully reworded conditions.
As shown in Table~\ref{tab:medical-qa}, the results are consistent with those in the general domain (shown in Tables \ref{tab:qwen-paraphrase-perdataset} and \ref{tab:semi-reworded}). APS tends to overshoot the $90\%$ target, with coverage ranging from $92.93\%$ to $100.00\%$. LAC is also often conservative and typically produces larger prediction sets. While SPUQ and TRON can achieve coverage closer to the nominal level in some cases, they generally require larger sets than PA.
In contrast, our PA score maintains coverage close to the target, ranging from $89.57\%$ to $91.77\%$, while yielding compact prediction sets of size $2.02$ to $3.61$. Notably, PA achieves the smallest set size across all $12$ medical-domain settings. These results demonstrate that PA preserves a strong coverage--efficiency trade-off in the medical domain, highlighting its ability to generalize beyond general-purpose benchmarks.

\begin{table*}[!ht]
\centering
\caption{{Results on two medical QA benchmarks with }\texttt{Qwen3-8B}. 
Coverage rate (CR, \%) closer to the nominal $90\%$ is better, while smaller set size (SS, $\downarrow$) denotes better.
Results are shown for normal, semi-reworded (Semi.), and fully reworded (Full) inputs.}
\vspace{-1mm}
\label{tab:medical-qa}
\resizebox{\textwidth}{!}{
\begin{tabular}{c|l|
  *{3}{c}|*{3}{c}||*{3}{c}|*{3}{c}}
\toprule
& & \multicolumn{6}{c||}{\textbf{MedMCQA}} & \multicolumn{6}{c}{\textbf{MedQA}} \\
\cmidrule(lr){3-8}\cmidrule(lr){9-14}
& & \multicolumn{3}{c|}{CR (\%)} & \multicolumn{3}{c||}{SS} 
   & \multicolumn{3}{c|}{CR (\%)} & \multicolumn{3}{c}{SS} \\
\cmidrule(lr){3-5}\cmidrule(lr){6-8}
\cmidrule(lr){9-11}\cmidrule(lr){12-14}
& \textbf{Method} 
& Normal & Semi. & Full & Normal & Semi. & Full
& Normal & Semi. & Full & Normal & Semi. & Full \\
\midrule
\multirow{5}{*}{\textit{Split CP}} 
& LAC        
& 99.03 & 98.90 & 90.20 & 5.64 & 5.63 & 2.74
& 98.97 & 98.73 & 99.40 & 5.81 & 5.82 & 5.90 \\
& APS        
& 98.97 & 98.93 & 98.90 & 5.64 & 5.65 & 5.65
& 99.97 & 100.00 & 100.00 & 6.00 & 6.00 & 6.00 \\
& {SPUQ-adapt.}       
& 90.10 & 86.37 & 91.27 & 2.59 & 2.34 & 2.86
& \textbf{89.97} & 81.83 & 88.80 & 4.05 & 3.72 & 4.43 \\
& {TRON-adapt.}       
& 90.17 & 87.73 & 91.03 & 2.46 & 2.50 & 2.81
& 90.47 & 84.97 & 89.00 & 3.97 & 3.99 & 4.35 \\
& PA 
& \textbf{90.00} & \textbf{89.63} & \textbf{89.87} & \textbf{2.23} & \textbf{2.02} & \textbf{2.17}
& 91.73 & \textbf{90.80} & \textbf{90.50} & \textbf{3.34} & \textbf{3.52} & \textbf{3.52} \\
\midrule
\multirow{5}{*}{\textit{QCCP}}
& LAC
& 92.87 & 93.53 & 94.83 & 3.64 & 3.98 & 4.13
& 95.60 & 94.47 & 93.67 & 4.97 & 5.20 & 5.06 \\
& APS
& 96.93 & 97.27 & 96.87 & 4.43 & 4.85 & 4.45
& 94.27 & 92.93 & 93.73 & 5.02 & 5.10 & 5.16 \\
& {SPUQ-adapt.}
& 90.27 & 87.43 & 91.70 & 2.74 & 2.61 & 3.02
& \textbf{90.53} & 82.37 & 89.40 & 4.14 & 3.80 & 4.50 \\
& {TRON-adapt.}
& \textbf{90.10} & 89.10 & 91.40 & 2.61 & 2.77 & 2.99
& 90.60 & 85.37 & 88.87 & 4.05 & 4.02 & 4.37 \\
& PA
& 89.57 & \textbf{90.03} & \textbf{89.67} & \textbf{2.47} & \textbf{2.48} & \textbf{2.40}
& 91.77 & \textbf{90.23} & \textbf{90.30} & \textbf{3.43} & \textbf{3.55} & \textbf{3.61} \\
\bottomrule
\end{tabular}
}
\end{table*}

\subsection{Finite-Sample Behavior of Class-Conditional Coverage}
When applying our PA scores within QCCP, we obtain population-level guarantees of the form
\[
1-\alpha \le \Pr(Y \in \widehat C(X) \mid X \in G) \le 1-\alpha + \mathrm{slack},
\]
for each cluster $G$, where the slack term decreases with the effective sample size within the cluster (see Lemma~\ref{lem:QCCP}). 

It is important to note that this guarantee holds at the population level, rather than for finite-sample empirical estimates. In practice, when we estimate coverage within each cluster using a finite number of samples, the observed coverage rate will generally fluctuate around the target level $1-\alpha = 0.90$. These deviations arise from two sources: (i) sampling variability due to finite data, and (ii) the presence of a nonzero slack term, which depends on the size and heterogeneity of each cluster. As a result, empirical cluster-wise coverage is not expected to match $0.90$ exactly, even when the theoretical guarantee holds.
Therefore, moderate deviations from $0.90$ in our experiments should not be interpreted as a violation of coverage guarantees, but rather as a natural consequence of finite-sample estimation under the QCCP framework.

To illustrate this phenomenon, we conduct three additional analyses in the semi-reworded setting, varying the number of clusters $K \in \{5, 7, 10, 20\}$.

\paragraph{Wilson Interval Analysis.}
To assess whether deviations from the nominal coverage are due to finite-sample variability, we construct Wilson 95\% confidence intervals for the empirical coverage of each cluster. 
Pooling clusters across datasets and choices of $K \in \{5, 7, 10, 20\}$, we observe a clear pattern. Among clusters with empirical coverage in $[0.85, 0.90)$---which appear mildly under-covered in Figure~\ref{fig:an2-classwise-condconf}---a large majority ($85.9\%$, $67/78$) have confidence intervals that still contain the nominal $0.90$ target. This suggests that these moderate deviations are largely attributable to finite-sample fluctuations. 
In contrast, for clusters with empirical coverage below $0.85$, none ($0/13$) of the corresponding confidence intervals contain $0.90$, indicating that these larger deviations are unlikely to be explained by sampling variability alone.

\paragraph{Effect of Cluster Sample Size.}
Cluster size has a clear impact on empirical coverage. Larger clusters tend to track the nominal target more closely, as expected from the theoretical guarantees. 
We partition PA-QCCP clusters into terciles based on calibration sample size (thresholds $q_1 = 166.67$ and $q_2 = 318.67$, with $168$ clusters in total). The average coverage increases monotonically across these groups: $0.8963$ (smallest), $0.9025$ (middle), and $0.9059$ (largest). 
A similar trend appears in the tail behavior. The fraction of clusters with empirical coverage $\le 0.85$ drops sharply from approximately $10.7\%$ in the smaller two terciles to $1.8\%$ in the largest tercile. 
These results confirm that effective local sample size plays a key role in stabilizing coverage. At the same time, it is not the sole factor: some large clusters, particularly in semantically challenging regions, still exhibit noticeable deviations.

\begin{table}[t]
\centering
\caption{Cluster-size-weighted coverage deviation (WCD, $\downarrow$) under QCCP in the semi-reworded setting on \texttt{Qwen3-8B}. Lower values indicate more uniform alignment with the nominal $1-\alpha = 0.9$ target. Bold marks the best method per dataset.}
\label{tab:wcd-qwen3}
\small
\begin{tabular}{l|ccc}
\toprule
\textbf{Dataset} & \textbf{LAC} & \textbf{APS} & \textbf{PA} \\
\midrule
QA (MMLU) & 0.0555 & 0.0687 & \textbf{0.0461} \\
CI (HellaSwag) & 0.0402 & 0.0860 & \textbf{0.0293} \\
DRS (HaluDial) & 0.0787 & 0.0476 & \textbf{0.0262} \\
\bottomrule
\end{tabular}
\end{table}

\paragraph{Cluster-Size-Weighted Coverage Deviation.}
To summarize cluster-level alignment with a single scalar, we define the \emph{cluster-size-weighted coverage deviation} (WCD):
\begin{equation}
\mathrm{WCD} \;=\; \sum_{k=1}^{K} \frac{n_k}{N}\, \big| \mathrm{CR}_k - (1-\alpha) \big|,
\end{equation}
where $n_k$ is the size of cluster $k$, $N = \sum_k n_k$, and $\mathrm{CR}_k$ is the empirical coverage in cluster $k$. Lower values indicate better alignment with the nominal target. Since all methods share the same cluster assignment, WCD provides a fair and comparable summary across baselines.

As shown in Table~\ref{tab:wcd-qwen3}, PA achieves the lowest WCD on QA, CI, and DRS at $K = 20$ on \texttt{Qwen3-8B}, indicating more uniform alignment with the $0.9$ target across semantic clusters than LAC or APS. 

\subsection{Robustness under the Semi-Reworded Setting}
\label{appendix:semi-reworded-setting}
Recall that the two versions of each dataset in our evaluation are defined as $\mathcal{D}_{\text{orig}} = \{(X_i, Y_i)\}_{i=1}^N$ and $\mathcal{D}_{\text{reworded}} = \{(\mathcal{T}(X_i), Y_i)\}_{i=1}^N$, where $\mathcal{T}$ represents the rewording operation. In the semi-reworded setting, we test a realistic semantic shift scenario where the calibration set is drawn from $\mathcal{D}_{\text{orig}}$ (human-written inputs), but the test set is drawn from $\mathcal{D}_{\text{reworded}}$ (LLM-reworded inputs). Under this shift, the exchangeability assumption required for standard conformal prediction is violated for baselines like LAC and APS, as the distribution of their nonconformity scores (which depend on raw logits) changes between calibration and testing. 

However, we hypothesize that our PA score remains robust to this shift. Our PA score is calculated by aggregating predictions over the set of paraphrases. We denote this paraphrasing generation as $\mathcal{T}^*$ to distinguish between the rewording used to create a distinct dataset ($\mathcal{T}$) and the paraphrasing used to calculate the PA score ($\mathcal{T}^*)$. Consequently, for a calibration point $X$, the score $S(X, Y)$ depends on $\mathcal{T}^*(X)$. For a test point $\mathcal{T}(X)$, the score depends on the composition $\mathcal{T}^*(\mathcal{T}(X))$. Invoking Assumption \ref{asmp:main} (distributional closedness), we posit that $\mathcal{T}^*(\mathcal{T}(X)) \stackrel{d}{\approx} \mathcal{T}^*(X)$. Under the exact alignment and score conditions stated in our analysis, calibration and test scores retain the required exchangeability. Under approximate alignment, only the corresponding coverage-deviation bound applies. We do not establish these conditions for an arbitrary external rewriter and internal generator. Please refer to Table \ref{tab:theory-impl-diagnostics} which reports the observed clean-to-rewrite change at a fixed cutoff.

\paragraph{Cluster $\Phi$ training in the semi-reworded setting.} For QCCP, the conditioning feature $\Phi(X)$ must remain exchangeable between calibration and test paired with the score $s(X,Y)$. In the semi-reworded setting, calibration uses $X$ but test inputs are $\mathcal{T}(X)$. To preserve $(\Phi, s)$ exchangeability for QCCP under this shift, we therefore train and apply the SBERT cluster $\Phi$ on a fixed paraphrase $\mathcal{T}^\dagger(X)$ (a single LM rewording, identical for calibration and test pieces), rather than on the raw original or raw rewritten text. The PA / LAC / APS / TRON / SPUQ scores themselves are computed exactly as in the main text; only the input to the cluster encoder is changed. 

We further evaluate our hypothesis in the semi-reworded setting using \texttt{Llama-3.1-8B} as both the dataset-level rewriter ($\mathcal{T}$) and the internal paraphrase generator ($\mathcal{T}^*$). As shown in Table~\ref{tab:llama-semi}, LAC performs poorly under this shift, achieving only 87.17\% coverage on CosmosQA and 86.97\% on HellaSwag, both well below the nominal 90\% target. APS overshoots on coverage and does so with substantially inflated prediction sets (SS up to 5.85 on CosmosQA and 4.79 on HellaSwag), making its predictions considerably less actionable.

In contrast, PA delivers the best tradeoff between reliability and efficiency. Under split CP, PA achieves 89.97\% coverage on CosmosQA and 90.27\% on HellaSwag---both tightly aligned with the nominal 90\%---while producing the smallest prediction sets (SS $\approx$ 1.41–2.03). Under QCCP, PA again reaches coverage closest to 90\% and maintains compact sets, substantially outperforming both LAC and APS. Compared with Table \ref{tab:semi-reworded}, these results highlight that the quality of the paraphrase generator $\mathcal{T}^*$ is crucial: with a stronger model such as \texttt{Llama-3.1-8B}, {empirical coverage improves for the evaluated generator.}

\begin{table*}[!ht]
\centering
\vspace{-1mm}
\caption{Performance of nonconformity scores under the \textbf{semi-reworded (Semi.)} setting on \texttt{Llama-3.1-8B} as paraphrase generator. 
Coverage Rate (CR, \%) closer to $90\%$ is better; Set Size (SS) $\downarrow$ is better. 
Results are reported on RC (Cosmos) and CI (HellaSwag). Bold numbers indicate the best performance.}
\vspace{-1mm}
\label{tab:llama-semi}
\begin{tabular}{c|l|
  *{2}{c}|*{2}{c}}
\toprule
& \textbf{Dataset} 
& \textbf{RC } & \textbf{CI } 
& \textbf{RC } & \textbf{CI} \\
\cmidrule(lr){3-4}\cmidrule(lr){5-6}
& Semi. & \multicolumn{2}{c|}{Coverage Rate (CR)} & \multicolumn{2}{c}{Set Size (SS)} \\

\midrule
\multirow{3}{*}{\textit{Split CP}} 
& LAC 
& 87.17 & 86.97 
& 3.51 & 3.91 \\
& APS 
& 99.33 & 90.90 
& 5.85 & 4.79 \\
& PA 
& \textbf{89.97} & \textbf{90.27} 
& \textbf{1.41} & \textbf{2.03} \\

\midrule
\multirow{3}{*}{\textit{QCCP}} 
& LAC 
& 87.00 & 87.77 
& 3.51 & 3.98 \\
& APS 
& 97.40 & 95.10 
& 5.41 & 5.33 \\
& PA 
& \textbf{90.00} & \textbf{90.43} 
& \textbf{1.43} & \textbf{2.01} \\

\bottomrule
\end{tabular}
\end{table*}

{\subsection{Quantifying Coverage Deviations in the Semi-Reworded Setting}
\label{appendix:theory-impl-gap}
Our implementation uses an external rewriter $\mathcal T$ and a separate
internal paraphrase generator $\mathcal T^*$. Exact distributional
closedness is a sufficient idealization; neither the exact theorem nor
its approximate bound establishes that this implemented pipeline
satisfies the assumption. We therefore examine the realized score
mismatch at the conformal decision boundary.}

{For each dataset and seed, we freeze the clean-trained proxy and the
clean split-CP cutoff $\hat q$, then score the same test questions in
original and externally rewritten form. Let $F_{\mathrm{clean}}$ and
$F_{\mathrm{rewrite}}$ be the empirical CDFs of their true-label
nonconformity scores. We report the signed cutoff-local change
\begin{equation}
\Delta_{\hat q}=F_{\mathrm{rewrite}}(\hat q)-F_{\mathrm{clean}}(\hat q).
\end{equation}
A negative value means that rewriting moves true-label score mass
above the cutoff, creating undercoverage pressure. The global KS
distance is $\sup_t|F_{\mathrm{rewrite}}(t)-F_{\mathrm{clean}}(t)|$.} KS is the maximum absolute gap between the two CDFs over all thresholds and therefore measures global shift.

\begin{table}[t]{
\centering
\caption{{Paired fixed-clean-cutoff diagnostics on \texttt{Qwen3-8B}
with $S_{\mathrm{mean}}$ and split CP at $\alpha=0.1$. Each row is the
mean over seeds 7, 42, and 123. The proxy and clean calibration cutoff
are held fixed while the same test questions are evaluated before and
after rewriting. Coverage error is relative to the nominal $90\%$
target; $\Delta_{\hat q}$ is relative to clean empirical coverage.}}
\label{tab:theory-impl-diagnostics}
\small
\setlength{\tabcolsep}{4pt}
\begin{tabular}{@{}lrrrr@{}}
\toprule
Dataset & Semi CR (\%) & Error from 90\% (pp) & $\Delta_{\hat q}$ (pp) & KS \\
\midrule
MMLU              & 85.5 & $-4.5$ & $-4.5$ & 0.059 \\
CosmosQA          & 89.1 & $-0.9$ & $-1.1$ & 0.052 \\
HellaSwag         & 89.5 & $-0.5$ & $-0.5$ & 0.084 \\
HaluDialogue      & 87.6 & $-2.4$ & $-2.7$ & 0.135 \\
HaluSummarization & 93.1 & $+3.1$ & $+2.9$ & 0.085 \\
\bottomrule
\end{tabular}
}\end{table}

{The two middle columns need not be numerically identical: the coverage error is measured relative to the nominal 90\% target, whereas $\Delta_{\hat q}$ is measured relative to clean empirical coverage. For example, MMLU loses approximately 4.5 percentage points of true-label mass below the operative cutoff and undercovers by 4.5 points. In contrast, HaluDialogue has the largest global KS distance (0.135), but only a 2.7-point local shift and a 2.4-point coverage error, showing why global score shift alone does not determine coverage.

Across the 15 underlying dataset--seed cells, rather than the five displayed
means, cutoff-local shift and Semi coverage deviation have Pearson
$r=0.983$ (paired-question bootstrap 95\% CI $[0.935,0.985]$).
At a fixed cutoff, this local shift directly accounts for the observed clean-to-rewrite empirical coverage change.}

\subsection{Semantic Consistency of Generated Paraphrases}
\label{appendix:paraphrase-quality}
To verify that the generated paraphrases approximately preserve the semantics of the original questions, we conduct a quantitative paraphrase quality evaluation across CosmosQA (RC), HellaSwag (CI), and MedMCQA (medical MCQA). For each dataset, we use the same paraphrase generation pipelines as in the main experiments and evaluate all paraphrases that are actually used for conformal prediction.

\paragraph{Paraphrase sets and generators.}
Recall that for each original question $x$ we use three paraphrase-related sets:
(i) $\mathcal{T}^*(x)$, the set of paraphrases used for PA score computation in the normal setting;
(ii) $\mathcal{T}(x)$, the fully reworded questions used in the fully reworded setting; and
(iii) $\mathcal{T}^*(\mathcal{T}(x))$, the paraphrases of fully reworded questions used for PA score computation in the fully reworded setting.
We instantiate these sets using two independent LLM paraphrase generators:
\texttt{Llama-3.1-8B} and \texttt{Qwen3-8B}.
All quality metrics are reported separately for each generator and each scenario.

\paragraph{Metrics.}
Given an original question $x$ and a paraphrase $p$, we compute three standard similarity metrics:
\begin{enumerate}[leftmargin=*]
    \item \textbf{SBERT cosine similarity.}
    We use the \texttt{all-MiniLM-L6-v2} SentenceTransformer to embed both $x$ and $p$, normalize the embeddings, and compute
    \[
        \cos(x, p) \;=\; \langle \mathrm{SBERT}(x), \mathrm{SBERT}(p) \rangle.
    \]
    This serves as a semantic similarity measure that is robust to minor lexical changes.
    \item \textbf{BLEU.}
    We compute sentence-level BLEU scores using \texttt{sacrebleu} with effective order enabled, treating the original question as the reference and the paraphrase as the hypothesis.
    \item \textbf{ROUGE-L.}
    We compute ROUGE-L F-measure using the \texttt{rouge\_score} library with stemming enabled, again using the original question as the reference and the paraphrase as the candidate.
\end{enumerate}

\paragraph{Aggregation and reporting.}
For each question $x$ and each paraphrase set (e.g., $\mathcal{T}^*(x)$), we first compute the metric values for all paraphrases $p \in \mathcal{T}^*(x)$ and then average them to obtain a single per-question score.
We repeat this for SBERT cosine, BLEU, and ROUGE-L.
At the dataset level, we report the mean and standard deviation across all questions that have at least one paraphrase in the corresponding set.
This yields, for each dataset, scenario, and model, a tuple of the form
\[
    \text{metric} \;=\; \mu \pm \sigma,
\]
which we summarize in Table~\ref{tab:paraphrase-quality}.

\paragraph{Findings.}
Across CosmosQA and MedMCQA, we observe moderate to high SBERT cosine similarity (typically in the range $0.61$--$0.87$, with Qwen3-8B paraphrases consistently above $0.81$), indicating that the paraphrases remain close in semantic space to the original questions while allowing stylistic variation.
HellaSwag exhibits lower cosine scores ($0.37$--$0.59$), which is expected due to its longer, multi-sentence contexts and richer surface-level variability.
BLEU and ROUGE-L follow the same pattern across datasets and paraphrase scenarios.
Overall, these results suggest that the paraphrase sets $\mathcal{T}^*(x)$, $\mathcal{T}(x)$, and $\mathcal{T}^*(\mathcal{T}(x))$ largely preserve the semantics of the original questions and are suitable for use in our paraphrase-aware conformal prediction framework.

\begin{table*}[ht!]
\centering
\caption{\textbf{Paraphrase quality evaluation} using Llama-3.1-8B and Qwen3-8B as paraphrase generators.
We report SBERT cosine similarity, BLEU, and ROUGE-L (mean $\pm$ std) for three paraphrase scenarios:
$\mathcal{T}^*(x)$,
$\mathcal{T}(x)$,
and $\mathcal{T}^*(\mathcal{T}(x))$.}
\label{tab:paraphrase-quality}
\small
\vspace{-2mm}

\begin{tabular}{lccc}
\toprule
\multicolumn{4}{c}{\textbf{RC (CosmosQA)}} \\
\cmidrule(lr){1-4}
\textbf{Scenario / Model} & \textbf{Cosine} & \textbf{BLEU} & \textbf{ROUGE-L} \\
\midrule
\multicolumn{4}{l}{\textit{$\mathcal{T}^*(x)$ --- paraphrases for human-written inputs}} \\
Llama & 0.82 $\pm$ 0.10 & 16.68 $\pm$ 8.57 & 45.18 $\pm$ 12.02 \\
Qwen3 & 0.81 $\pm$ 0.10 & 16.06 $\pm$ 8.60 & 42.19 $\pm$ 11.64 \\
\addlinespace[3pt]

\multicolumn{4}{l}{\textit{$\mathcal{T}(x)$ --- LLM-reworded inputs}} \\
Llama & 0.70 $\pm$ 0.21 & 11.94 $\pm$ 13.24 & 35.84 $\pm$ 19.98 \\
Qwen3 & 0.84 $\pm$ 0.11 & 19.28 $\pm$ 14.52 & 49.17 $\pm$ 17.48 \\
\addlinespace[3pt]

\multicolumn{4}{l}{\textit{$\mathcal{T}^*(\mathcal{T}(x))$ --- paraphrases of LLM-reworded inputs}} \\
Llama & 0.70 $\pm$ 0.20 & 11.73 $\pm$ 9.83 & 35.62 $\pm$ 16.60 \\
Qwen3 & 0.82 $\pm$ 0.11 & 19.07 $\pm$ 11.22 & 48.20 $\pm$ 13.76 \\
\bottomrule
\end{tabular}

\vspace{4mm}

\begin{tabular}{lccc}
\toprule
\multicolumn{4}{c}{\textbf{CI (HellaSwag)}} \\
\cmidrule(lr){1-4}
\textbf{Scenario / Model} & \textbf{Cosine} & \textbf{BLEU} & \textbf{ROUGE-L} \\
\midrule
\multicolumn{4}{l}{\textit{$\mathcal{T}^*(x)$ --- paraphrases for human-written inputs}} \\
Llama & 0.49 $\pm$ 0.03 &  9.80 $\pm$ 2.08 & 34.39 $\pm$ 4.53 \\
Qwen3 & 0.54 $\pm$ 0.03 &  8.96 $\pm$ 2.85 & 31.08 $\pm$ 4.97 \\
\addlinespace[3pt]

\multicolumn{4}{l}{\textit{$\mathcal{T}(x)$ --- LLM-reworded inputs}} \\
Llama & 0.37 $\pm$ 0.15 &  6.71 $\pm$ 4.64 & 24.65 $\pm$ 11.55 \\
Qwen3 & 0.57 $\pm$ 0.02 &  9.11 $\pm$ 3.53 & 39.16 $\pm$ 6.78 \\
\addlinespace[3pt]

\multicolumn{4}{l}{\textit{$\mathcal{T}^*(\mathcal{T}(x))$ --- paraphrases of LLM-reworded inputs}} \\
Llama & 0.37 $\pm$ 0.13 &  6.66 $\pm$ 2.75 & 24.61 $\pm$ 8.13 \\
Qwen3 & 0.59 $\pm$ 0.06 & 20.91 $\pm$ 7.20 & 48.79 $\pm$ 7.13 \\
\bottomrule
\end{tabular}

\vspace{4mm}

\begin{tabular}{lccc}
\toprule
\multicolumn{4}{c}{\textbf{MedMCQA}} \\
\cmidrule(lr){1-4}
\textbf{Scenario / Model} & \textbf{Cosine} & \textbf{BLEU} & \textbf{ROUGE-L} \\
\midrule
\multicolumn{4}{l}{\textit{$\mathcal{T}^*(x)$ --- paraphrases for human-written inputs}} \\
Llama & 0.81 $\pm$ 0.10 & 10.41 $\pm$ 6.50 & 38.25 $\pm$ 11.39 \\
Qwen3 & 0.84 $\pm$ 0.08 & 12.32 $\pm$ 7.74 & 40.16 $\pm$ 11.59 \\
\addlinespace[3pt]

\multicolumn{4}{l}{\textit{$\mathcal{T}(x)$ --- LLM-reworded inputs}} \\
Llama & 0.61 $\pm$ 0.27 & 5.81 $\pm$ 5.69 & 24.51 $\pm$ 14.60 \\
Qwen3 & 0.87 $\pm$ 0.09 & 14.81 $\pm$ 11.51 & 46.09 $\pm$ 15.90 \\
\addlinespace[3pt]

\multicolumn{4}{l}{\textit{$\mathcal{T}^*(\mathcal{T}(x))$ --- paraphrases of LLM-reworded inputs}} \\
Llama & 0.61 $\pm$ 0.27 & 5.81 $\pm$ 5.69 & 24.51 $\pm$ 14.60 \\
Qwen3 & 0.84 $\pm$ 0.10 & 12.33 $\pm$ 8.88 & 40.71 $\pm$ 13.43 \\
\bottomrule
\end{tabular}

\end{table*}

{\subsection{End-to-End Runtime}
\label{appendix:runtime-analysis}
We measure the complete online path on one NVIDIA B200 with
\texttt{Qwen3-8B}. Online Full PA includes paraphrase generation,
backbone inference, proxy scoring, aggregation, and prediction-set
construction. The cached column assumes paraphrases and hidden
representations have already been stored. All entries are mean
milliseconds per question.}

\begin{table}[t]{
\centering
\caption{{End-to-end latency (ms/question) on one NVIDIA B200.
Batch denotes the inference batch size; $m$ is the number of
paraphrases. Cached Full PA excludes paraphrase generation and
backbone representation extraction.}}
\label{runtime-analysis}
\small
\setlength{\tabcolsep}{4pt}
\begin{tabular}{@{}llrrrrr@{}}
\toprule
Dataset & Batch & Raw LAC & Aug-Proxy Single &
Full PA $m{=}2$ & Full PA $m{=}6$ &
Cached $m{=}6$ \\
\midrule
MMLU     & 1  & 19.51 & 18.73 & 1599.76 & 1626.58 & 0.200 \\
CosmosQA & 1  & 24.14 & 23.52 & 1565.11 & 1611.55 & 0.195 \\
MMLU     & 16 &  9.63 &  9.57 &  187.54 &  357.37 & 0.056 \\
CosmosQA & 16 & 10.78 & 10.74 &  150.32 &  245.09 & 0.056 \\
\bottomrule
\end{tabular}
}\end{table}

{In these measurements, online Full PA with $m=6$ is
22.7--83.4 times slower than raw single-input inference; paraphrase
generation accounts for 67--97\% of the online time. Once the views
are cached, the remaining scoring, aggregation, and set construction
take at most 0.20 ms per question. The online generation cost is
therefore a practical limitation when latency matters.}

% \clearpage
\section{Analysis of the Proxy Model}
\label{appendix:proxy-vs-pa}

A natural question is whether the gain of our method comes from the proxy classifier alone---which improves the calibration of each individual prediction---or from the paraphrase-aware (PA) score that aggregates variability across rewordings. To disentangle the two, we compare four variants under the same \texttt{Qwen3-8B} backbone and the same CP/QCCP wrapper:
\begin{enumerate}[leftmargin=*]
\item \textbf{Full PA} (proxy $+$ paraphrase aggregation): the method proposed in the main paper.
\item \textbf{Paraphrase-only}: PA score computed directly from raw LLM logits, with no proxy classifier.
\item \textbf{{Aug-Proxy Single}}: {the paraphrase-augmented proxy evaluated on a single input view, without inference-time aggregation.}
\item \textbf{Neither} ($\equiv$ LAC): raw logits with no proxy and no paraphrase aggregation.
\end{enumerate}

\paragraph{Removing the Proxy Alone.} We first report proxy-removal ablation results under split CP in Figure~\ref{fig-appendix:nomlp-ablation-split}; the same trend holds under QCCP (Figure~\ref{fig:nomlp-ablation}). Coverage remains close to $90\%$, but set sizes inflate substantially without the proxy.

\begin{figure*}[!ht]
\vspace{-10pt}
  \centering
  \includegraphics[width=0.45\textwidth]{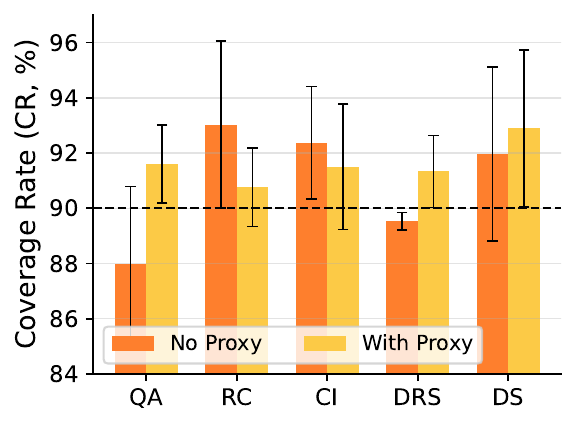}
  \includegraphics[width=0.45\textwidth]{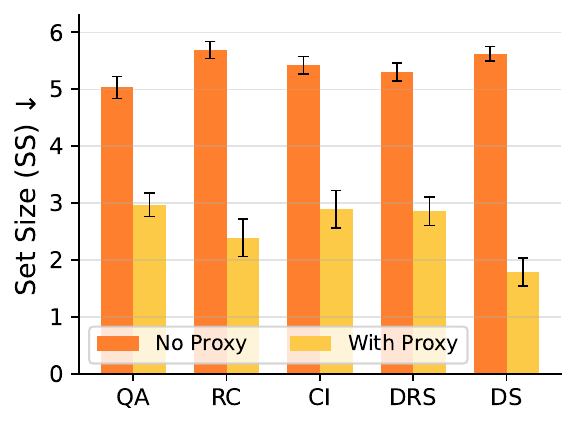}
\vspace{-5pt}
  \caption{Effect of removing the proxy model under the split CP setting on \texttt{Qwen3-8B}.
  Coverage remains close to $90\%$ (dashed line), but set size increases notably without the proxy model.}
  \label{fig-appendix:nomlp-ablation-split}
\end{figure*}

\begin{figure*}[!ht]
  \centering
  \vspace{-4mm}
  \includegraphics[width=0.36\textwidth]{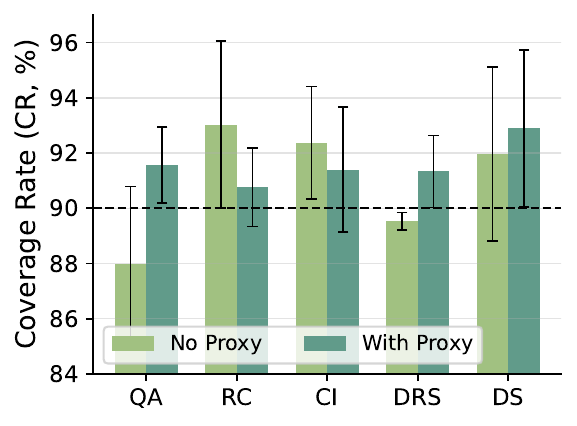}
  \includegraphics[width=0.36\textwidth]{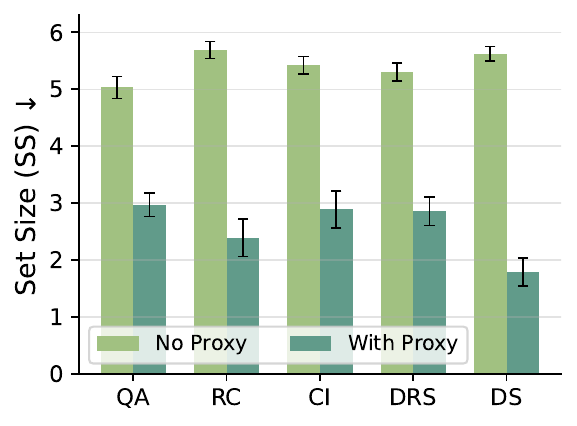}
  \vspace{-2mm}
  \caption{{Effect of removing the proxy model under the quasi-conditional CP setting.} 
  Coverage remains close to $90\%$ (dashed line), but set size increases notably without the proxy model.}
  \label{fig:nomlp-ablation}
  \vspace{-2mm}
\end{figure*}

\paragraph{{Single-Run Attribution.}} 
Table~\ref{tab:proxy-pa-attribution} reports CR/SS for the three PA variants in both the semi-reworded and fully reworded settings, under both CP and QCCP. The pattern is consistent across datasets:

\begin{itemize}[leftmargin=*]
    \item In the \textbf{semi-reworded} setting, \emph{paraphrase-only} becomes unstable without proxy calibration: under split CP it drops to $84.83\%$ on QA and $87.93\%$ on CosmosQA, while under QCCP it swings to $94.20\%/4.53$ on HaluDial. \emph{{Aug-Proxy Single}} is more stable but tends to be conservative, especially on DS where QCCP reaches $95.30\%$ coverage. The \emph{full method} is usually the most efficient variant, keeping the smallest or tied-smallest sets while staying close to nominal on RC, CI, and DRS.
    \item In the \textbf{fully reworded} setting, all three variants are close to nominal coverage, but the full method remains the most efficient in most columns: on CosmosQA under split CP its set size is $1.17$ versus $1.77$ for \emph{paraphrase-only}, and on HaluDial under QCCP it is $1.39$ versus $4.86$. Even when \emph{{augmented-proxy single-view}} is competitive, the full method still matches or slightly improves its set size on nearly every dataset.
\end{itemize}

{These results motivate the controlled decomposition in Table~\ref{tab:factorized-attribution}: the proxy supplies most of the compactness, and same-checkpoint aggregation supplies a smaller stability gain.}

\begin{table*}[th!]
\centering
\vspace{-1mm}
\caption{{Single-run attribution of the PA score} on \texttt{Qwen3-8B}. Each cell reports \textbf{CR / SS} (coverage rate, \%, closer to $90$ is better; set size $\downarrow$ is better). \textbf{Para+Proxy} denotes the full method, \textbf{Para only} removes the proxy, and \textbf{{Aug-Proxy Single}} {removes inference-time aggregation while retaining paraphrase-augmented training.} Bold entries highlight severe under- or over-coverage exhibited by single-component variants. {These single-run rows retain the paraphrase-trained proxy in the single-view arm and are kept separate from the three-seed controls in Table~\ref{tab:factorized-attribution} and Table~\ref{tab:separate-same-checkpoint-settings}; they are not pooled with those later three-seed controls.}}
\vspace{-1mm}
\label{tab:proxy-pa-attribution}
\small
\resizebox{\textwidth}{!}{
\begin{tabular}{l|l|ccccc}
\toprule
& \textbf{Variant} & \textbf{QA (MMLU)} & \textbf{RC (Cosmos)} & \textbf{CI (HellaSwag)} & \textbf{DRS (HaluDial)} & \textbf{DS (HaluSum)} \\
\midrule
\multicolumn{7}{l}{\textit{Semi-reworded, Split CP}}\\
& Para + Proxy (Full) & 87.10 / \textbf{1.85} & \textbf{90.10} / \textbf{1.12} & 89.13 / \textbf{1.28} & \textbf{89.67} / \textbf{1.36} & 94.13 / \textbf{1.00} \\
& Para only           & 84.83 / 2.21 & 87.93 / 1.53 & 87.30 / 1.60 & 86.23 / 2.72 & \textbf{88.20} / 2.49 \\
& {Aug-Proxy Single}          & \textbf{89.53} / 2.02 & 90.47 / 1.14 & \textbf{89.37} / 1.29 & 90.87 / 1.43 & 95.20 / \textbf{1.00} \\
\addlinespace[2pt]
\multicolumn{7}{l}{\textit{Semi-reworded, QCCP}}\\
& Para + Proxy (Full) & 89.50 / \textbf{2.01} & \textbf{90.67} / \textbf{1.39} & 89.53 / \textbf{1.36} & \textbf{91.10} / \textbf{1.70} & 94.20 / \textbf{1.11} \\
& Para only           & 87.73 / 2.54 & 88.70 / 1.94 & 88.63 / 2.06 & 94.20 / 4.53 & \textbf{89.90} / 3.21 \\
& {Aug-Proxy Single}          & \textbf{89.77} / 2.16 & 90.73 / 1.40 & \textbf{89.90} / 1.51 & 91.87 / 1.76 & 95.30 / \textbf{1.11} \\
\midrule
\multicolumn{7}{l}{\textit{Fully reworded, Split CP}}\\
& Para + Proxy (Full) & 90.37 / \textbf{2.13} & 90.97 / \textbf{1.17} & \textbf{90.03} / 1.34 & \textbf{90.57} / \textbf{1.35} & 90.50 / \textbf{1.07} \\
& Para only           & \textbf{89.77} / 2.65 & \textbf{90.20} / 1.77 & 89.77 / 1.79 & 91.13 / 3.32 & \textbf{89.73} / 2.73 \\
& {Aug-Proxy Single}          & 90.57 / 2.15 & 91.03 / 1.19 & 89.73 / \textbf{1.32} & 91.37 / 1.41 & 90.47 / \textbf{1.07} \\
\addlinespace[2pt]
\multicolumn{7}{l}{\textit{Fully reworded, QCCP}}\\
& Para + Proxy (Full) & \textbf{90.43} / \textbf{2.40} & \textbf{90.97} / \textbf{1.37} & 91.80 / \textbf{1.73} & \textbf{90.33} / \textbf{1.39} & \textbf{90.73} / \textbf{1.31} \\
& Para only           & 91.47 / 3.12 & \textbf{90.97} / 2.08 & \textbf{91.50} / 2.69 & 97.00 / 4.86 & 94.97 / 4.46 \\
& {Aug-Proxy Single}          & 90.50 / 2.42 & 91.80 / 1.41 & 91.80 / 1.74 & 90.77 / 1.44 & 91.27 / \textbf{1.31} \\
\bottomrule
\end{tabular}}
\end{table*}

{\paragraph{Training-View Control.} The main factorized comparison is in Table~\ref{tab:factorized-attribution}.
Here we distinguish standard Orig-only training, which uses each
original question once, from Repeat-Orig training, which repeats the
same originals to match the augmented run's number of training rows
and optimizer updates. Augmented training instead uses the originals
and up to six meaning-preserving paraphrases with their original labels.}

\begin{table}[t]{
\centering
\caption{{Training-view control for the hidden-state proxy on
\texttt{Qwen3-8B}, averaged equally over five datasets and three seeds.
All rows use single-view inference. Score drift is the mean absolute
clean/rewrite change in the true-label nonconformity score. Repeat-Orig
controls for repeated training updates; standard Orig-only is the
primary no-paraphrase training baseline.}}
\label{tab:proxy-training-views}
\small
\setlength{\tabcolsep}{5pt}
\begin{tabular}{@{}lrrrrr@{}}
\toprule
Proxy training & Rewrite Acc.\ $\uparrow$ & Semi CR &
Semi SS $\downarrow$ & Score drift $\downarrow$ &
Set change $\downarrow$ \\
\midrule
Standard Orig-only   & 79.80 & 89.58 & 1.3725 & 0.0596 & 0.2009 \\
Repeat-Orig          & 79.64 & 89.27 & 1.3717 & 0.0635 & 0.2188 \\
Paraphrase-augmented & 80.38 & 89.96 & 1.3587 & 0.0574 & 0.2029 \\
\bottomrule
\end{tabular}
}\end{table}

{Relative to standard Orig-only, paraphrase-augmented training improves
rewritten-input accuracy by 0.58 percentage points and reduces
true-label score drift by 3.7\%. Repeat-Orig performs slightly worse
than standard Orig-only on both measures, so the improvement cannot
be attributed simply to seeing more copies of the original questions.
The set-change measure for augmented single-view inference is
slightly higher than for Orig-only (0.2029 versus 0.2009); the
reduction to 0.1708 occurs when the same augmented checkpoint is
used with six-view inference (Table~\ref{tab:factorized-attribution}).}

{\paragraph{Set size at closely matched empirical coverage.}
The fixed-$\alpha$ comparison in Table~\ref{tab:factorized-attribution}
uses the same nominal coverage target for every method, but achieved
coverage can differ. As a separate descriptive analysis in the
semi-reworded setting, we use one half of the shifted examples to
choose an operating point near 90\% coverage for each method and
evaluate on the other half. Both methods use the same
paraphrase-augmented proxy checkpoint. The results are shown in Table \ref{tab:matched-coverage-attribution}. At approximately matched achieved coverage, Full PA has smaller
sets on four of the five datasets and is effectively tied on
HaluSummarization.}

\begin{table}[t]{
\centering
\caption{{Held-out, approximately coverage-matched comparison in the
Semi setting. The coverage columns report achieved coverage on the
held-out evaluation half. Negative Full PA minus Single set-size
differences favor Full PA. This is a separate operating-point
analysis, not the fixed-$\alpha$ experiment.}}
\label{tab:matched-coverage-attribution}
\small
\setlength{\tabcolsep}{5pt}
\begin{tabular}{@{}lrrr@{}}
\toprule
Dataset & Single CR (\%) & Full PA CR (\%) &
Full PA $-$ Single SS \\
\midrule
MMLU              & 89.81 & 89.43 & $-0.0022$ \\
CosmosQA          & 91.27 & 91.03 & $-0.0233$ \\
HellaSwag         & 89.90 & 89.43 & $-0.0404$ \\
HaluDialogue      & 90.12 & 90.18 & $-0.0178$ \\
HaluSummarization & 89.27 & 89.36 & $+0.0002$ \\
\bottomrule
\end{tabular}
}\end{table}

\begin{table}[t]{
\centering
\caption{{Paired prediction-set stability in the same separate study.
Set change is the mean symmetric-difference size
$|\widehat C(x)\triangle\widehat C(\mathcal T(x))|$, the number of labels
entering or leaving a set under rewriting. Lower is better. Values are five-dataset macro means $\pm$ sample standard deviation across seeds 7, 42, and 123.}}
\label{tab:separate-same-checkpoint-stability}
\small
\begin{tabular}{@{}lccc@{}}
\toprule
Proxy and cutoff fitted on & Aug-Proxy Single & Full PA & Relative change \\
\midrule
Original questions & $0.1994\pm0.0082$ & $0.1737\pm0.0089$ & $-12.9\%$ \\
Reworded questions & $0.2213\pm0.0004$ & $0.2098\pm0.0043$ & $-5.2\%$ \\
\bottomrule
\end{tabular}
}\end{table}

{\paragraph{Same-Checkpoint Inference Control.}
A separate five-dataset, three-seed study compares one versus six inference views using the same augmented proxy checkpoint, calibration/test IDs, and evaluator. Table~\ref{tab:separate-same-checkpoint-settings} uses different frozen splits and validation-selected checkpoints from Table~\ref{tab:factorized-attribution}. In that study, aggregation reduces mean clean/rewrite set change by 12.9\% with an original-input-fitted cutoff and 5.2\% with a reworded-input-fitted cutoff (Appendix Table~\ref{tab:separate-same-checkpoint-stability}). 
}\begin{table}[t]{
\centering
\caption{{Separate same-checkpoint study of inference-time aggregation
under split CP at $\alpha=0.1$. Both variants use the same
paraphrase-augmented proxy. CR is empirical coverage in percent and
SS is mean prediction-set size. Values are five-dataset macro means $\pm$ sample standard deviation across seeds 7, 42, and 123.}}
\label{tab:separate-same-checkpoint-settings}
\small
\begin{tabular}{@{}lcc@{}}
\toprule
Setting & Aug-Proxy Single: CR / SS & Full PA: CR / SS \\
\midrule
Normal & $90.58\pm0.16\;/\;1.293\pm0.007$ & $90.50\pm0.06\;/\;1.299\pm0.005$ \\
Semi & $90.40\pm0.18\;/\;1.329\pm0.008$ & $89.56\pm0.31\;/\;1.286\pm0.006$ \\
Fully reworded & $90.59\pm0.37\;/\;1.352\pm0.012$ & $90.17\pm0.31\;/\;1.338\pm0.010$ \\
\bottomrule
\end{tabular}
}\end{table}
\vspace{-10pt}

\subsection{Different Class-Conditional Levels for QCCP}

We report class-conditional coverage for QCCP on the remaining datasets in Figure~\ref{fig-appendix:an2-classwise-condconf-others}, complementing Figure~\ref{fig:an2-classwise-condconf} in the main text. The results are consistent with those in the main text: PA tracks the $90\%$ target across clusters, while LAC and APS overshoot and the TRON / SPUQ baselines exhibit larger per-cluster swings.
\begin{figure*}[th!]
  \centering
    \vspace{-3mm}
    \includegraphics[width=0.45\textwidth]{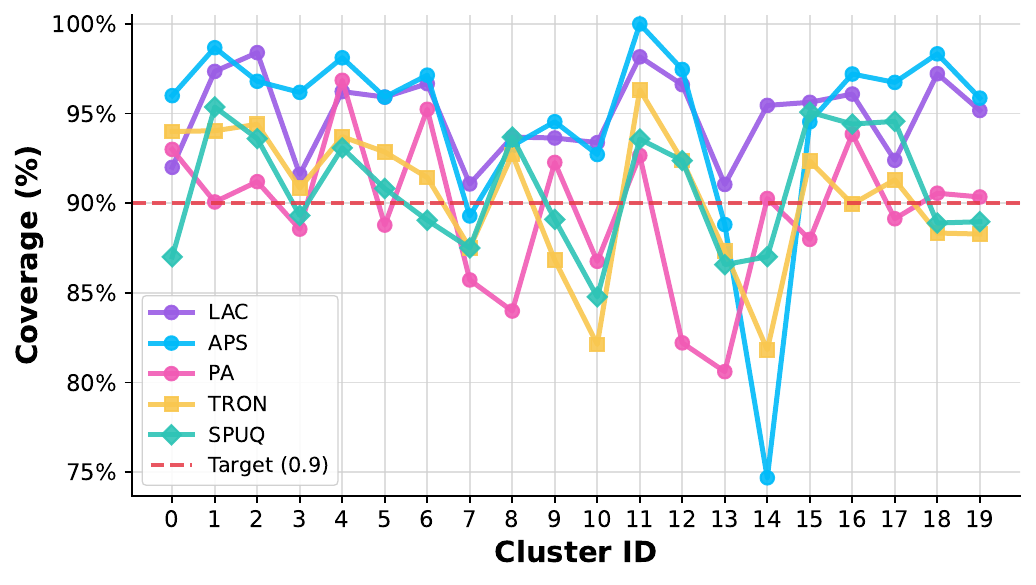}
    \includegraphics[width=0.45\textwidth]{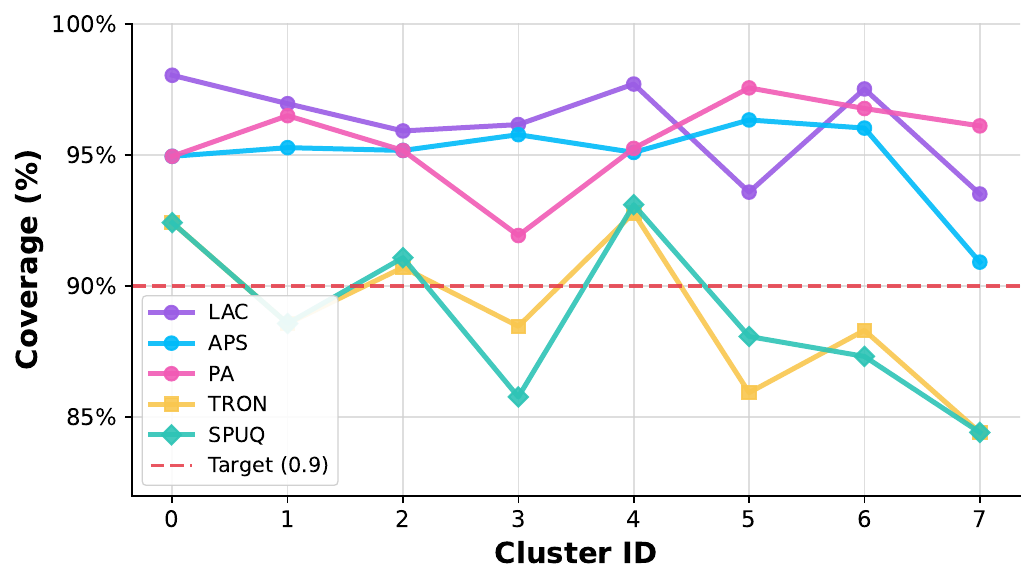}
    \vspace{-1mm}
  \caption{Evaluation of class-conditional coverage with QCCP on the MMLU (left) and HaluSum (right) datasets under \texttt{Qwen3-8B}.}
  \label{fig-appendix:an2-classwise-condconf-others}
\end{figure*}

\subsection{Different \texorpdfstring{$\alpha$}{Alpha}-Level for Split CP}
We report results at different user-specified risk levels under split CP in Figure~\ref{fig-appendix:risk-level-results-split-cp}, which complements Figure~\ref{fig:risk-level-results} for QCCP in the main text. The trends are consistent: PA tracks the nominal $(1-\alpha)$ target most tightly across all three levels (e.g., $79.7\%$ at $\alpha{=}0.2$, $95.1\%$ at $\alpha{=}0.05$, $99.2\%$ at $\alpha{=}0.01$) while keeping the most compact sets ($1.63$, $2.67$, and $3.48$ respectively); APS dramatically overshoots at $\alpha{=}0.2$ ($\approx 95\%$ coverage with set size $\approx 4.5$), and LAC produces strictly larger sets than PA at low $\alpha$ (set size $\approx 4.8$ at $\alpha{=}0.05$ and $\approx 4.9$ at $\alpha{=}0.01$). The TRON and SPUQ baselines track the target reasonably well at small $\alpha$ but produce larger sets than PA at every level.

\begin{figure*}[!ht]
    \centering
    \vspace{-3mm}
    \begin{subfigure}[b]{0.4\textwidth}
        \includegraphics[width=\linewidth]{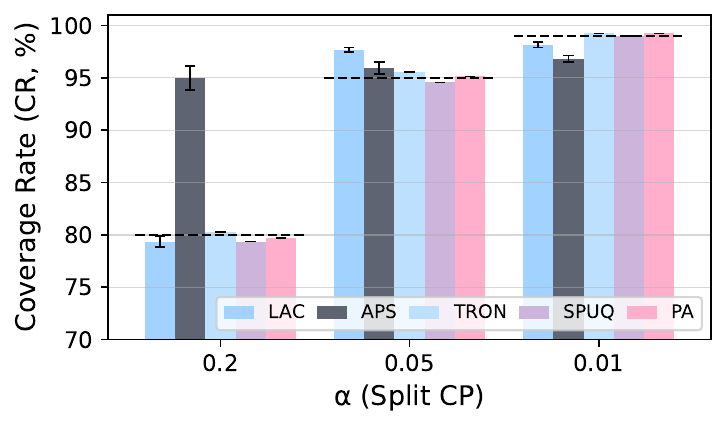}
    \end{subfigure}
    \begin{subfigure}[b]{0.4\textwidth}
        \includegraphics[width=\linewidth]{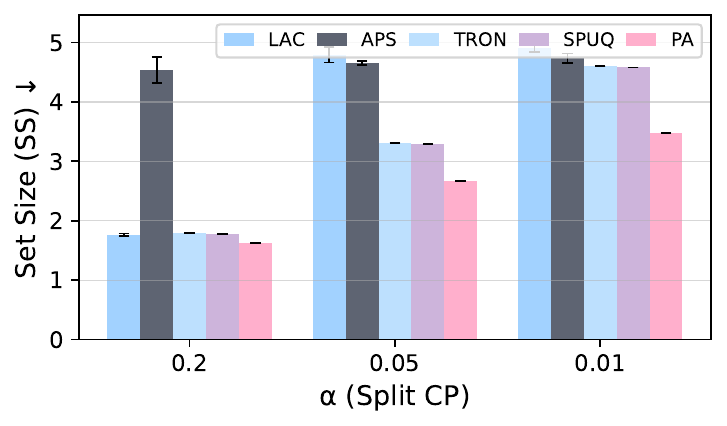}
    \end{subfigure}

    \vspace{-3mm}
    \caption{Coverage rate (left) and prediction set size (right) vs.\ $\alpha$ on MMLU under Split CP, with \texttt{Qwen3-8B} as the backbone.}
    \label{fig-appendix:risk-level-results-split-cp}
\end{figure*}

\subsection{Different Paraphrase Generators}
\label{appendix:different-llm-paraphrase-generator}

To assess sensitivity to the choice of paraphrase generator, we use another model, \texttt{Llama-3.1-8B}, as the paraphrase generator while keeping the rest of the pipeline fixed. In particular, we continue to use \texttt{Qwen3-8B} for embedding extraction. We evaluate this setup on three datasets (CosmosQA, HellaSwag, and MedMCQA) under both normal and fully reworded settings.
As shown in Table~\ref{tab:llama-paraphrase-qwen3}, PA remains robust to this change. Under split CP, it maintains coverage close to the $90\%$ target while achieving smaller set sizes across all datasets. Similar trends are observed under QCCP.

\begin{table*}[!ht]
\centering
\vspace{-1mm}
\caption{{{Performance of nonconformity scores under split CP and QCCP with \texttt{Llama-3.1-8B} as the paraphrase generator and \texttt{Qwen3-8B} as the embedding model.
Results are reported on normal and fully reworded test sets across three benchmarks. Bold numbers indicate the best performance.}}}
\vspace{-1mm}
\label{tab:llama-paraphrase-qwen3}
\begin{tabular}{l|
  *{2}{c} | *{2}{c} | *{2}{c}}
\toprule
& \multicolumn{2}{c|}{\textbf{RC (Cosmos)}} & \multicolumn{2}{c|}{\textbf{CI (HellaSwag)}} &
  \multicolumn{2}{c}{\textbf{MedMCQA}} \\
\cmidrule(lr){2-7}
\textbf{Method} & {Normal} & {Full} & {Normal} & {Full} & {Normal} & {Full} \\
\midrule
\multicolumn{7}{l}{\textit{Coverage Rate (CR, \%) Better if closer to $90\%$ --- Split CP}} \\
LAC        & 90.40 & 99.00 & 89.83 & \textbf{89.87} & \textbf{89.42} & 99.20 \\
APS        & 99.27 & 99.10 & 99.63 & 98.90 & 90.92 & 99.97 \\
PA         & \textbf{90.13} & \textbf{90.60} & \textbf{90.13} & 90.67 & 91.03 & \textbf{89.83} \\
\addlinespace[3pt]
\multicolumn{7}{l}{\textit{Coverage Rate (CR, \%) Better if closer to $90\%$ --- QCCP }}\\
LAC        & 91.90 & 93.70 & 91.93 & 91.50 & 92.31 & 95.00 \\
APS        & 98.60 & 97.70 & 99.00 & 98.70 & 96.69 & 96.80 \\
PA         & \textbf{90.63} & \textbf{90.57} & \textbf{90.63} & \textbf{90.80} & \textbf{91.88} & \textbf{90.13} \\
\midrule
\multicolumn{7}{l}{\textit{Set Size (SS) $\downarrow$ --- Split CP}}\\
LAC        & 1.55 & 5.56 & 1.67 & 1.75 & 2.45 & 5.67 \\
APS        & 5.65 & 5.59 & 5.71 & 5.59 & 2.68 & 5.87 \\
PA         & \textbf{1.12} & \textbf{1.27} & \textbf{1.35} & \textbf{1.65} & \textbf{2.27} & \textbf{2.55} \\
\addlinespace[3pt]
\multicolumn{7}{l}{\textit{Set Size (SS) $\downarrow$ --- QCCP }}\\
LAC        & 2.26 & 3.16 & 2.60 & 2.65 & 3.66 & 4.19 \\
APS        & 4.86 & 4.61 & 5.06 & 5.00 & 4.54 & 4.46 \\
PA         & \textbf{1.34} & \textbf{1.47} & \textbf{1.70} & \textbf{1.97} & \textbf{2.51} & \textbf{2.71} \\
\bottomrule
\end{tabular}
\end{table*}

\subsection{Different Embedding Choices}
\label{sec:pooling-sensitivity}

To evaluate whether the PA nonconformity score depends on a particular LLM representation, we conduct a controlled sensitivity analysis on the \texttt{Qwen3-8B} model using the \textbf{CosmosQA} benchmark. In the main paper, PA uses the last-layer hidden state at the final token position. Here, we consider three alternative representations that are commonly used in transformer analysis:

\begin{itemize}[leftmargin=*]
    \item \textbf{Mean pooling over tokens.}
    We average hidden states across all sequence positions (masking padding tokens). This yields a sequence-level representation that does not rely on a specific token location.

    \item \textbf{Intermediate-layer hidden states.}
    Instead of the last layer, we extract representations from a configurable layer index. In the experiment below, we uses the second-to-last layer). We keep the extraction location fixed (final token), so only the layer depth varies.

    \item \textbf{Attention-weighted pooling.}
    We use attention weights from the last layer to produce a weighted average over token embeddings. Specifically, we take the last token’s attention distribution, average over attention heads, and use this vector as weights over the hidden states of all tokens.
\end{itemize}

\paragraph{Results.}
Tables~\ref{tab:pooling-cr-qwen3} and~\ref{tab:pooling-ss-qwen3} report coverage rate (CR) and set size (SS) under split CP. Across all pooling and layer choices, PA-based nonconformity scores remain extremely stable. This demonstrates that PA does \emph{not} depend on a specific pooling strategy or layer depth, and that the semantic structure encoded throughout the transformer stack is sufficient for proxy training.

\begin{table*}[!ht]
\centering
\vspace{-2mm}
\caption{Coverage Rate (CR, \%) under different pooling and layer choices for \texttt{Qwen3-8B} on CosmosQA. CR is better if closer to 90\%. Results are reported for both the normal and fully reworded (Full) settings. Best results are in bold.}
\label{tab:pooling-cr-qwen3}
\small
\sisetup{round-mode=places,round-precision=2,table-number-alignment=center}
\resizebox{\textwidth}{!}{
\begin{tabular}{l l|
  *{2}{S} | *{2}{S} | *{2}{S} | *{2}{S}}
\toprule
& & \multicolumn{2}{c|}{\textbf{Last Hidden (Paper)}}
  & \multicolumn{2}{c|}{\textbf{Attention Weighted}}
  & \multicolumn{2}{c|}{\textbf{Intermediate Layer}}
  & \multicolumn{2}{c}{\textbf{Mean Pooling}} \\
\cmidrule(lr){3-10}
\textbf{CP Type} & \textbf{Method}
  & {Normal} & {{Full}}
  & {Normal} & {{Full}}
  & {Normal} & {{Full}}
  & {Normal} & {{Full}} \\
\midrule
\multirow{3}{*}{Split CP} & LAC
    & \textbf{90.53} & \textbf{90.47}
    & 90.60 & 98.87
    & \textbf{90.47} & 98.87
    & \textbf{90.47} & 98.87 \\
  & APS
    & 99.07 & 99.20
    & 99.40 & 99.50
    & 99.60 & 100.00
    & 99.60 & 100.00 \\
  & PA
    & 91.00 & 90.97
    & \textbf{90.00} & \textbf{90.10}
    & 89.27 & \textbf{89.43}
    & 88.83 & \textbf{90.30} \\
\midrule
\multirow{3}{*}{QCCP} & LAC
    & 92.23 & 93.67
    & 91.63 & 93.37
    & 91.73 & 93.13
    & 91.73 & 93.13 \\
  & APS
    & 98.63 & 97.50
    & 98.63 & 97.77
    & 98.40 & 97.73
    & 98.40 & 97.73 \\
  & PA
    & \textbf{91.37} & \textbf{90.97}
    & \textbf{90.43} & \textbf{90.10}
    & \textbf{90.17} & \textbf{90.37}
    & \textbf{89.33} & \textbf{90.77} \\
\bottomrule
\end{tabular}}
\end{table*}

\begin{table*}[!ht]
\centering
\vspace{-2mm}
\caption{Prediction set size (SS, $\downarrow$) under different pooling and layer choices for \texttt{Qwen3-8B} on CosmosQA. Lower SS indicates sharper uncertainty in Normal and fully reworded (Full) settings. Best results are in bold.}
\label{tab:pooling-ss-qwen3}
\small
\sisetup{round-mode=places,round-precision=2,table-number-alignment=center}
\resizebox{\textwidth}{!}{
\begin{tabular}{l l|
  *{2}{S} | *{2}{S} | *{2}{S} | *{2}{S}}
\toprule
& & \multicolumn{2}{c|}{\textbf{Last Hidden (Paper)}}
  & \multicolumn{2}{c|}{\textbf{Attention Weighted}}
  & \multicolumn{2}{c|}{\textbf{Intermediate Layer}}
  & \multicolumn{2}{c}{\textbf{Mean Pooling}} \\
\cmidrule(lr){3-10}
\textbf{CP Type} & \textbf{Method}
  & {Normal} & {{Full}}
  & {Normal} & {{Full}}
  & {Normal} & {{Full}}
  & {Normal} & {{Full}} \\
\midrule
\multirow{3}{*}{Split CP} & LAC
    & 1.56 & 1.80
    & 1.61 & 5.56
    & 1.57 & 5.56
    & \textbf{1.57} & 5.56 \\
  & APS
    & 5.57 & 5.61
    & 5.64 & 5.66
    & 5.77 & 6.00
    & 5.77 & 6.00 \\
  & PA
    & \textbf{1.13} & \textbf{1.17}
    & \textbf{1.15} & \textbf{1.12}
    & \textbf{1.16} & \textbf{1.12}
    & 2.58 & \textbf{2.52} \\
\midrule
\multirow{3}{*}{QCCP} & LAC
    & 2.47 & 3.11
    & 2.35 & 2.91
    & 2.39 & 2.87
    & \textbf{2.39} & 2.87 \\
  & APS
    & 4.79 & 4.73
    & 4.80 & 4.68
    & 4.61 & 4.59
    & 4.61 & 4.59 \\
  & PA
    & \textbf{1.34} & \textbf{1.37}
    & \textbf{1.36} & \textbf{1.34}
    & \textbf{1.40} & \textbf{1.35}
    & 2.73 & \textbf{2.69} \\
\bottomrule
\end{tabular}}
\end{table*}

\subsection{Different Proxy Training and Calibration Budgets}
\label{appendix:low-label}

A practical concern is whether PA requires large labeled datasets to fit the proxy classifier. To investigate this, we study how performance varies with two label budgets:
(i) the \emph{proxy-train fraction}, denoted by $\mathrm{tf}$, the fraction of training instances used to fit the proxy classifier; and
(ii) the \emph{calibration fraction}, denoted by $\mathrm{cf}$, the fraction of calibration instances used by CP/QCCP.
We vary $\mathrm{tf} \in \{0.01, 0.05, 0.10, 1.00\}$ and $\mathrm{cf} \in \{0.05, 1.00\}$ using \texttt{Qwen3-8B} in the normal setting, with three random seeds per configuration. Table~\ref{tab:low-label} reports CR/SS on CosmosQA and HellaSwag under PA-QCCP with the $S_{\mathrm{weighted}}$ score.

\begin{table*}[!ht]
\centering
\caption{Low-label sensitivity of PA on \texttt{Qwen3-8B} in the normal setting under PA-QCCP. \textbf{tf} is the proxy-training fraction; \textbf{cf} is the calibration fraction. Each cell is CR / SS, mean $\pm$ std over 3 random seeds. Bold rows mark the full-budget ($\mathrm{tf}=1$) reference.}
\label{tab:low-label}
\small
\setlength{\tabcolsep}{4pt}
\begin{tabular}{c|c|cc|cc}
\toprule
 &  & \multicolumn{2}{c|}{\textbf{RC (CosmosQA)}} & \multicolumn{2}{c}{\textbf{CI (HellaSwag)}} \\
\textbf{tf} &  & cf $=$ 0.05 & cf $=$ 1.00 & cf $=$ 0.05 & cf $=$ 1.00 \\
\midrule
0.01 & CR & $91.16 \pm 0.73$ & $91.17 \pm 0.46$ & $92.51 \pm 1.57$ & $90.31 \pm 0.10$ \\
 & SS & $1.859 \pm 0.147$ & $1.602 \pm 0.116$ & $2.309 \pm 0.030$ & $1.848 \pm 0.128$ \\
\addlinespace[2pt]
0.05 & CR & $91.00 \pm 2.37$ & $91.31 \pm 0.25$ & $92.41 \pm 0.94$ & $91.14 \pm 0.77$ \\
 & SS & $1.640 \pm 0.182$ & $1.408 \pm 0.050$ & $2.205 \pm 0.037$ & $1.802 \pm 0.097$ \\
\addlinespace[2pt]
0.10 & CR & $90.79 \pm 1.81$ & $91.01 \pm 0.16$ & $91.06 \pm 2.04$ & $90.18 \pm 0.95$ \\
 & SS & $1.641 \pm 0.349$ & $1.395 \pm 0.054$ & $2.126 \pm 0.131$ & $1.774 \pm 0.047$ \\
\addlinespace[2pt]
\textbf{1.00} & CR & $\mathbf{91.20 \pm 1.78}$ & $\mathbf{91.27 \pm 0.35}$ & $\mathbf{92.19 \pm 2.35}$ & $\mathbf{90.92 \pm 0.35}$ \\
 & SS & $\mathbf{1.470 \pm 0.129}$ & $\mathbf{1.338 \pm 0.044}$ & $\mathbf{1.992 \pm 0.082}$ & $\mathbf{1.674 \pm 0.077}$ \\
\bottomrule
\end{tabular}
\end{table*}

On the stronger \texttt{Qwen3-8B} backbone, the failure mode at very small $\mathrm{tf}$ is much milder:
\begin{itemize}[leftmargin=*]
    \item At $\mathrm{tf} = 0.01$, PA already reaches near-nominal coverage (CR $\approx 91\%$ on RC, $90$--$93\%$ on CI) and pays only a modest set-size penalty relative to the full-budget reference. Larger $\mathrm{tf}$ steadily tightens prediction sets: from $\mathrm{tf}=0.01$ to $\mathrm{tf}=1.00$ at $\mathrm{cf}=1.00$, set size shrinks by {$16.5\%$ on RC} ($1.602 \to 1.338$) and {$9.4\%$ on CI} ($1.848 \to 1.674$).
    \item Increasing $\mathrm{tf}$ to $0.05$--$0.10$ already recovers most of the efficiency gain, {with near-target coverage in the full-calibration columns.}
    \item {Reducing the calibration fraction from $1$ to $0.05$ retains broadly similar empirical coverage, but increases set size by 9.9--24.9\% across these cells.}
\end{itemize}
This confirms that PA does not require the full labeled training budget to remain effective: a modest $5$--$10\%$ proxy-training fraction is enough to nearly match the full-budget reference, {while reducing the calibration budget incurs a measurable set-size cost.} 

\subsection{Different Aggregation Scoring Rules}
\label{appendix:score-analysis}
Finally, we analyze the behavior of the Mean, Weighted, and Worst nonconformity scores across the five general QA benchmarks, computing statistics over the calibration and test splits. 

As shown in Table~\ref{tab:proxy-model-score-stats}, the Worst score consistently exhibits substantially higher mean and standard deviation than the Mean and Weighted variants. 
Since higher nonconformity scores correspond to lower predicted probabilities (i.e., lower confidence), this pattern highlights the brittleness of LLM predictions: models often struggle with specific “hard” paraphrases, assigning high nonconformity scores to the correct answer even when other paraphrases are handled correctly.
In contrast, the Mean and Weighted scores remain lower and more stable (with smaller standard deviation). This suggests that our aggregation strategy effectively smooths out the noise induced by brittle paraphrases, yielding uncertainty estimates that better reflect the model’s underlying semantic knowledge rather than its sensitivity to surface-level variations.

\begin{table*}[!ht]
\caption{Score statistics for Qwen3-8B displayed as mean $\pm$ std for the correct answer choice on each dataset. These statistics are calculated using scores from both the calibration and test sets. A lower score indicates higher confidence (higher predicted probability) from the proxy model.}
\small
\label{tab:proxy-model-score-stats}
\begin{center}
\begin{tabular}{cccc}
\toprule 
\multicolumn{1}{c}{\bf Dataset}  &\multicolumn{1}{c}{\bf Mean Score} &\multicolumn{1}{c}{\bf Weighted Score} &\multicolumn{1}{c}{\bf Worst Score} 
\\ \midrule 
MMLU & 0.385 $\pm$ 0.326 & 0.375 $\pm$ 0.332 & 0.534 $\pm$ 0.368 \\
CosmosQA & 0.190 $\pm$ 0.284 & 0.182 $\pm$ 0.287 & 0.286 $\pm$ 0.356 \\
HellaSwag & 0.258 $\pm$ 0.313 & 0.249 $\pm$ 0.317 & 0.371 $\pm$ 0.373 \\
HaluDial & 0.299 $\pm$ 0.310 & 0.288 $\pm$ 0.316 & 0.445 $\pm$ 0.363 \\
HaluSum & 0.068 $\pm$ 0.174 & 0.060 $\pm$ 0.182 & 0.151 $\pm$ 0.297 \\
\bottomrule \\
\end{tabular}
\end{center}
\end{table*}

\section{Scalability and Broader Applicability}
\label{append:applicability}

\subsection{Extension to Black-Box / API-Only Settings}
\label{appendix:api-only}

The white-box PA variant in the main paper relies on LLM hidden states to train a proxy classifier. To evaluate whether the paraphrase-aware idea extends to API-only deployments---where only token log-probabilities are available---{we instantiate an SPUQ-inspired per-option CP adaptation that uses neither hidden states nor a proxy model.}

Given a question $x$ with paraphrases $\tilde x_1, \dots, \tilde x_P$ and option $y$, suppose the API returns option probabilities $p_j(y)$ for each paraphrase $\tilde x_j$. We define
\begin{equation}
S^{\mathrm{API}}_{\mathrm{mean}}(x, y) \;=\; \frac{1}{P} \sum_{j=1}^{P} \big(1 - p_j(y)\big),
\qquad
S^{\mathrm{API}}_{\mathrm{worst}}(x, y) \;=\; \max_{j \in [P]} \big(1 - p_j(y)\big).
\end{equation}
These scores rely only on option logprobs or probabilities; they do not access hidden states and do not train a proxy. We then plug them into the same PA-CP / PA-QCCP pipeline as the main method.

 Table~\ref{tab:api-only} reports the API-only PA variant on \texttt{Qwen3-8B} in the semi-reworded QCCP setting (the most demanding setting in our paper). The goal of this experiment is to demonstrate \emph{applicability under black-box access}, not parity with the white-box pipeline.

\begin{table}[!ht]
\centering
\caption{Black-box / API-only PA variant on \texttt{Qwen3-8B} in the \textbf{semi-reworded QCCP} setting. Each cell is CR / SS. The API-only variant uses option logprobs only---no hidden states and no proxy classifier---and is provided as evidence that the paraphrase-aware idea is not restricted to white-box access.}
\label{tab:api-only}
\small
\begin{tabular}{l|ccccc}
\toprule
\textbf{Method} & \textbf{QA} & \textbf{RC} & \textbf{CI} & \textbf{DRS} & \textbf{DS} \\
\midrule
LAC & 93.17 / 3.72 & 90.47 / 2.50 & 91.07 / 2.63 & 97.87 / 5.02 & 94.23 / 4.30 \\
APS & 95.83 / 4.68 & 98.43 / 4.83 & 98.60 / 4.92 & 93.37 / 4.45 & 96.57 / 5.03 \\
PA (API-only) & 87.83 / \textbf{2.66} & 88.43 / \textbf{1.76} & 89.37 / \textbf{2.33} & 93.97 / \textbf{4.38} & 91.97 / \textbf{3.77} \\
\bottomrule
\end{tabular}
\end{table}

API-only PA reduces set size substantially relative to LAC and APS on all five datasets. {Coverage falls below 90\% on QA, RC, and CI (87.83\%, 88.43\%, and 89.37\%), while DRS and DS are conservative.} We therefore view the API-only variant as a meaningful but \emph{partial} extension: it shows that paraphrase-aware aggregation does not fundamentally require hidden-state access, while the absence of a proxy still leaves some robustness gap in the harder shifted cases.

\subsection{Extension to Larger Backbone LLMs}
\label{appendix:larger-backbones}

We evaluate whether the advantage of PA persists as the backbone LLM scales. Specifically, we consider \texttt{Qwen3-14B} and \texttt{Qwen3-32B}, and report results under standard split CP on \texttt{CosmosQA} and \texttt{HellaSwag} across the normal, semi-reworded, and fully reworded settings.

Table~\ref{tab:larger-backbones-qwen3} shows that the qualitative pattern remains consistent. Across both datasets and model sizes, PA yields substantially smaller prediction sets than LAC and APS, reducing set size from roughly $1.3$--$6.0$ to about $1.0$--$1.3$.
\begin{itemize}[leftmargin=*]
    \item On \texttt{CosmosQA}, PA maintains coverage close to the $90\%$ target in the normal and fully reworded settings while preserving this efficiency gain. In the semi-reworded setting, the LAC baseline shows the largest deviations (e.g., $86.13\%$ with the 32B model), confirming that this remains the most challenging transfer regime for that baseline.
    \item A similar pattern holds on \texttt{HellaSwag}. PA again achieves the smallest prediction sets, while maintaining near-nominal coverage across all three settings: in the semi-reworded setting it reaches $91.27\%$ coverage with the 14B model and $91.33\%$ with the 32B model.
\end{itemize}
Overall, these results show that the core benefit of PA---substantially tighter prediction sets at competitive coverage---persists at larger scales.

\begin{table*}[!ht]
\centering
\vspace{-1mm}
\caption{Performance of nonconformity scores with larger backbone LLMs under \textbf{standard split CP}. Results are reported on RC (\texttt{CosmosQA}) and CI (\texttt{HellaSwag}) using \texttt{Qwen3-14B} and \texttt{Qwen3-32B}. Each entry is \textbf{CR / SS} (coverage rate / set size), where coverage closer to $90\%$ is better and set size $\downarrow$ is better. Normal / Semi / Full denote \textbf{normal / semi-reworded / fully reworded}. Boldface highlights the smallest set size in each column.}
\vspace{-1mm}
\label{tab:larger-backbones-qwen3}
\resizebox{\columnwidth}{!}{%
\begin{tabular}{c|l|ccc|ccc}
\toprule
& \textbf{Method}
& \multicolumn{3}{c|}{\textbf{Qwen3-14B}}
& \multicolumn{3}{c}{\textbf{Qwen3-32B}} \\
\cmidrule(lr){3-5}\cmidrule(lr){6-8}
\textbf{Dataset} &
& {Normal} & {Semi} & {Full}
& {Normal} & {Semi} & {Full} \\
\midrule
\multirow{3}{*}{{RC}}
& LAC
& 89.87 / 1.71 & 88.40 / 1.78 & 90.90 / 1.99
& 88.90 / 1.49 & 86.13 / 1.57 & 90.20 / 1.86 \\
& APS
& 100.00 / 6.00 & 100.00 / 6.00 & 100.00 / 6.00
& 99.77 / 5.74 & 99.77 / 5.76 & 97.93 / 3.49 \\
& PA
& 89.60 / \textbf{1.01} & 88.90 / \textbf{1.01} & 88.80 / \textbf{1.07}
& 90.23 / \textbf{1.00} & 89.60 / \textbf{1.00} & 90.40 / \textbf{1.01} \\
\midrule
\multirow{3}{*}{{CI}}
& LAC
& 89.40 / \textbf{1.28} & 87.97 / 1.32 & 89.93 / 1.46
& 91.30 / 1.30 & 90.63 / 1.34 & 91.30 / 1.37 \\
& APS
& 99.33 / 5.56 & 99.40 / 5.57 & 99.83 / 5.72
& 99.87 / 5.69 & 99.90 / 5.72 & 99.87 / 5.69 \\
& PA
& 91.77 / 1.29 & 91.27 / \textbf{1.26} & 89.87 / \textbf{1.19}
& 91.83 / \textbf{1.10} & 91.33 / \textbf{1.09} & 90.83 / \textbf{1.11} \\
\bottomrule
\end{tabular}}
\end{table*}

\subsection{Extension to Factuality Verification Tasks}
\label{appendix:extension-to-other-task}

While our primary framework addresses tasks with finite label spaces (i.e., $|\mathcal{Y}| < \infty$), it naturally extends to critical sub-problems within open-ended generation, specifically \textit{factuality verification}. Although the space of generated text is theoretically infinite, modern evaluation protocols typically decompose long-form generations into discrete atomic claims, each evaluated independently against a knowledge source \citep{mohri2024language, cherian2024large}. 

This decomposition effectively reduces the verification problem to a binary classification task. Formally, let an open-ended generation be parsed into a set of atomic claims $\{c_1, c_2, \dots, c_k\}$. We treat each claim $c_i$ as an input $x \in \mathcal{X}$, where the label space is restricted to $\mathcal{Y} = \{\text{True}, \text{False}\}$.

Our paraphrase-robust framework is particularly well-suited for this setting due to the sensitivity of LLM-based evaluators to surface-level phrasing. To apply our method:
\begin{enumerate}[leftmargin=*]
    \item \textbf{Paraphrase Generation:} For a given claim $x$, we generate a set of semantic paraphrases $B(x)$.
    \item \textbf{Proxy Scoring:} We utilize our lightweight proxy model $P_\theta(y \mid E(x))$ to predict the truthfulness of the claim based on the embeddings of these paraphrases.
    \item \textbf{Robust Uncertainty:} We calculate PA scores across $B(x)$.
\end{enumerate}

By integrating these scores into the conformal prediction framework, we ensure that the resulting uncertainty estimates, as well as the subsequent decision to flag a claim as hallucinated, are robust to lexical diversity. This guarantees that a claim is not penalized merely due to the specific syntax chosen by the LLM, but rather evaluated based on its semantic validity.

\paragraph{Preliminary experiment on claim-verification.}
To provide concrete evidence beyond MCQA, we conducted a preliminary claim-verification experiment on FEVER \citep{thorne2018fever}, a standard benchmark for factuality verification. We evaluated the same three settings used in the main paper---\emph{clean}, \emph{semi-reworded}, and \emph{fully-reworded}---and report coverage and average set size at $\alpha=0.1$. As in the main text, coverage closer to $0.9$ is better, while smaller set size is preferred.

Table~\ref{tab:fever_cp} and Table~\ref{tab:fever_qccp} show that the proposed PA score transfers meaningfully to this open-ended verification setting under \texttt{Qwen3-8B}. On FEVER, the backbone is sharp enough that all non-APS methods return essentially singleton prediction sets, so the main distinction is whether coverage stays closest to the nominal target under shift. PA remains the most stable method across settings, especially in the fully reworded case, where it reaches $90.5\%$ under Split CP and $90.3\%$ under QCCP, while LAC/TRON/SPUQ dip slightly below $90\%$. APS remains overly conservative, with prediction sets close to the full label set. These results suggest that the benefit of paraphrase-aware scoring is not specific to MCQA, but can also improve robustness in claim-level factuality verification, which serves as a practically relevant building block for open-ended generation evaluation.

\begin{table}[!ht]
\centering
\caption{Preliminary extension to open-ended generation on \textsc{FEVER} under \textit{Split CP} at $\alpha=0.1$, with \textbf{Qwen3-8B} as the backbone. Coverage Rate (CR, \%) closer to $90\%$ is better; Set Size (SS) $\downarrow$ is better. Results are reported under the \textbf{clean}, \textbf{semi-reworded (Semi.)}, and \textbf{Fully-reworded (Full)} settings. {TRON-adapt.\ and SPUQ-adapt.\ denote the per-option probability-score controls here;} {the probability-score adaptation equals LAC for any label count, so these two rows match by construction}.}
\label{tab:fever_cp}
\setlength{\tabcolsep}{4pt}
\begin{tabular}{c|cc|cc|cc}
\toprule
\textbf{Method} 
& \multicolumn{2}{c|}{\textbf{Clean}} 
& \multicolumn{2}{c|}{\textbf{Semi.}} 
& \multicolumn{2}{c}{\textbf{Full}} \\
\cmidrule(lr){2-3}\cmidrule(lr){4-5}\cmidrule(lr){6-7}
& \textbf{CR} & \textbf{SS} 
& \textbf{CR} & \textbf{SS} 
& \textbf{CR} & \textbf{SS} \\
\midrule
LAC & 93.7 & 1.000 & 90.5 & 1.000 & 89.7 & 1.000 \\
APS & 100.0 & 1.910 & 100.0 & 1.965 & 100.0 & 1.993 \\
{TRON-adapt.} & 93.7 & 1.000 & 90.5 & 1.000 & 89.7 & 1.000 \\
{SPUQ-adapt.} & 93.7 & 1.000 & 90.7 & 1.000 & 89.4 & 1.000 \\
PA  & 94.0 & 1.000 & 91.2 & 1.000 & 90.5 & 1.000 \\
\bottomrule
\end{tabular}%
\end{table}

\begin{table}[!ht]
\centering
\caption{Preliminary extension to open-ended generation on \textsc{FEVER} under \textit{QCCP} at $\alpha=0.1$, with \textbf{Qwen3-8B} as the backbone. Coverage Rate (CR, \%) closer to $90\%$ is better; Set Size (SS) $\downarrow$ is better. Results are reported under the \textbf{clean}, \textbf{semi-reworded (Semi.)}, and \textbf{Fully-reworded (Full)} settings. {TRON-adapt.\ and SPUQ-adapt.\ denote the per-option probability-score controls here;} {the probability-score adaptation equals LAC for any label count, so these two rows match by construction}.}
\label{tab:fever_qccp}
\setlength{\tabcolsep}{4pt}
\begin{tabular}{c|cc|cc|cc}
\toprule
\textbf{Method} 
& \multicolumn{2}{c|}{\textbf{Clean}} 
& \multicolumn{2}{c|}{\textbf{Semi.}} 
& \multicolumn{2}{c}{\textbf{Full}} \\
\cmidrule(lr){2-3}\cmidrule(lr){4-5}\cmidrule(lr){6-7}
& \textbf{CR} & \textbf{SS} 
& \textbf{CR} & \textbf{SS} 
& \textbf{CR} & \textbf{SS} \\
\midrule
LAC & 93.7 & 1.000 & 90.5 & 1.000 & 89.9 & 1.002 \\
APS & 100.0 & 1.947 & 99.8 & 1.982 & 100.0 & 1.987 \\
{TRON-adapt.} & 93.7 & 1.000 & 90.5 & 1.000 & 89.9 & 1.002 \\
{SPUQ-adapt.} & 93.5 & 1.000 & 90.5 & 1.000 & 89.7 & 1.005 \\
PA  & 94.5 & 1.000 & 90.7 & 1.000 & 90.3 & 1.000 \\
\bottomrule
\end{tabular}%
\end{table}

Overall, we view this experiment as a preliminary but encouraging step toward broader open-ended applications. A full treatment of long-form generation, including claim decomposition and aggregation back to response-level guarantees, is an important direction for future work.

\subsection{Extension to Image Modality}
\label{appendix:extension-to-image-modality}

Although our experiments focus on text-based tasks, the proposed framework is not limited to textual inputs and can, in principle, be extended to vision–language settings such as visual question answering (VQA). The key requirement of our method is access to a stable hidden representation for each input variant, which is naturally provided by modern vision–language models (VLMs).

Concretely, given an image–question pair, we treat the visual input as part of the conditioning context and apply the same paraphrase-aware uncertainty pipeline at the representation level. In lieu of textual paraphrases, visual perturbations can be used to induce semantic-preserving variants of the input image. For example, this can be realized by adding Gaussian noise to the image, producing multiple visually indistinguishable variants that preserve the underlying semantic content while altering surface-level pixel statistics.

Each perturbed image is processed by a frozen VLM encoder, and we extract the final-layer hidden representation corresponding to the answer token, mirroring the text-only pipeline. Paraphrase-aware (variant-aware) nonconformity scores are then computed across these image variants and integrated into the same conformal prediction framework without modification.

\section{Limitations and Broader Impacts}
\label{append:limitations}

\paragraph{Limitations.}
{The strongest variant requires hidden states and labeled proxy-training data. Online paraphrase generation is a material cost: the B200 six-view measurements are 22.7--83.4 times slower than raw single-input inference (Appendix~\ref{appendix:runtime-analysis}). The blackbox extension shows applicability, but still not in parity with the white-box model (Appendix~\ref{appendix:api-only}). Arbitrary test-only rewriting is still outside the coverage guarantee. }Our framework relies on the quality of the paraphrase generator: if generated paraphrases exhibit semantic drift or limited diversity, the effectiveness of the paraphrase-aware score may degrade. Although we observe stable behavior across several paraphrase generators, domain-specific settings may require stronger or more specialized paraphrasing models. More broadly, our work focuses on semantic-preserving perturbations induced by paraphrasing and does not explicitly address other forms of distribution shift, such as topic drift, temporal changes in factual knowledge, or adversarial perturbations. In addition, our experiments are primarily conducted on English multiple-choice QA benchmarks. Extending paraphrase-aware conformal prediction to multilingual settings and free-form generation tasks remains an important direction for future work.

\paragraph{Broader impacts.}
Our method improves the stability and reliability of LLM uncertainty estimates under semantic rewording, which may be particularly helpful for users whose phrasing differs from typical training distributions, such as non-native speakers or users of regional dialects. At the same time, prediction sets produced by conformal methods should not be interpreted as guarantees of correctness in high-stakes applications, but rather as uncertainty-aware decision support.

\end{document}